\documentclass[11pt,letterpaper]{article}
\usepackage[margin=1in]{geometry}

\usepackage[utf8]{inputenc} % allow utf-8 input
\usepackage[T1]{fontenc}    % use 8-bit T1 fonts
\usepackage{hyperref}       % hyperlinks
\usepackage{url}            % simple URL typesetting
\usepackage{booktabs}       % professional-quality tables
\usepackage{amsfonts}       % blackboard math symbols
\usepackage{nicefrac}       % compact symbols for 1/2, etc.
\usepackage{microtype}      % microtypography
\usepackage{color}
\usepackage{amsmath, amsthm, amssymb}
\usepackage{algorithm, algorithmic}
\usepackage{multirow}
\usepackage{pifont}% http://ctan.org/pkg/pifont
\usepackage{enumitem}
\usepackage{graphicx}
\renewcommand{\>}{\rightarrow}

\newcommand{\E}{\mathbf{E}}
\newcommand{\X}{\mathcal{X}}

\newcommand{\cA}{\mathcal{A}}

\newcommand{\cR}{{\R_+^K}}

\newcommand{\Y}{\mathcal{Y}}

\renewcommand{\P}{\mathbf{P}}
\newcommand{\R}{\mathbb{R}}

\newcommand{\ba}{{\mathbf{a}}}
\newcommand{\bb}{{\mathbf{b}}}

\newcommand{\cC}{{\mathcal{C}}}
\newcommand{\cL}{{\mathcal{L}}}
\newcommand{\cU}{{\mathcal{U}}}
\newcommand{\pL}{{\tilde{\cL}}}
\newcommand{\cH}{{\mathcal{H}}}
\newcommand{\cO}{\mathcal{O}}
\newcommand{\tO}{\tilde{\cO}}
\newcommand{\tL}{\tilde{\cL}}

\newcommand{\0}{{\mathbf{0}}}

\newcommand{\TPR}{\textrm{\textup{TPR}}}
\newcommand{\FPR}{\textrm{\textup{FPR}}}
\newcommand{\FNR}{\textrm{\textup{FNR}}}
\newcommand{\TNR}{\textrm{\textup{TNR}}}
\newcommand{\TP}{\textrm{\textup{TP}}}
\newcommand{\TN}{\textrm{\textup{TN}}}
\newcommand{\FP}{\textrm{\textup{FP}}}
\newcommand{\FN}{\textrm{\textup{FN}}}
\newcommand{\dom}{\textrm{\textup{dom}}}

\newcommand{\sign}{\textrm{\textup{sign}}}

\newcommand{\balpha}{\boldsymbol{\alpha}}
\newcommand{\bbeta}{\boldsymbol{\beta}}

\newcommand{\blambda}{{\lambda}}
\newcommand{\bxi}{{\xi}}
\newcommand{\bR}{\mathbf{R}}
\renewcommand{\r}{\mathbf{r}}
\newcommand{\bv}{\mathbf{v}}
\newcommand{\bu}{\mathbf{u}}
\newcommand{\tR}{\tilde{R}}
\newcommand{\tbR}{\tilde{\bR}}
\newcommand{\argmin}[1]{\underset{#1}{\operatorname{argmin}}}
\newcommand{\amin}[1]{\operatorname{argmin}_{#1}}
\newcommand{\argmax}[1]{\underset{#1}{\operatorname{argmax}}}
\newcommand{\amax}[1]{\operatorname{argmax}_{#1}}
\newcommand{\cmark}{\ding{51}}%
\newcommand{\xmark}{\ding{55}}%
\definecolor{light-blue}{rgb}{0.0, 0.6, 1.0}
\definecolor{dark-blue}{rgb}{0.0, 0.0, 0.7}
\newcommand{\best}[1]{{\color{dark-blue} #1}}
\newcommand{\sbest}[1]{{\color{light-blue} #1}}
\newtheorem{thm}{Theorem}
\newtheorem{lem}{Lemma}

\newtheorem{defn}{Definition}
\newtheorem{rem}{Remark}

\allowdisplaybreaks

\title{Optimizing Generalized Rate Metrics through Game Equilibrium} % Three-Player Games}

\author{%
  Harikrishna Narasimhan, Andrew Cotter, Maya Gupta \\
    Google Research\\
  {\{hnarasimhan, acotter, mayagupta\}@google.com 
  } \\
}

\date{}

\begin{document}

\maketitle

\begin{abstract}
We present a general framework for solving a large class of learning problems with non-linear functions of classification rates. This includes problems where one wishes to optimize a non-decomposable performance metric such as the F-measure or G-mean, and constrained training problems where the classifier needs to satisfy non-linear rate constraints such as predictive parity fairness, distribution divergences or churn ratios. We extend previous two-player game approaches for constrained optimization to a game between three players to decouple the classifier rates from the non-linear objective, and seek to find an equilibrium of the game. Our approach generalizes many existing algorithms, and makes possible new algorithms with more flexibility and tighter handling of non-linear rate constraints. We provide convergence guarantees for convex functions of rates, and show how our methodology can be extended to handle sums of ratios of rates. Experiments on different fairness tasks confirm the efficacy of our approach.
\end{abstract}

\section{Introduction}
\label{sec:intro}
In many real-world machine learning problems, the performance measures used to evaluate a classification model are non-linear functions of the classifier's prediction rates. Examples include the F-measure, G-mean and H-mean used in class-imbalanced classification tasks \cite{Lewis91,Kim+13,Sun+06,WangYao12,Lawrence+98}, metrics such as predictive parity used to impose specific \textit{fairness} goals \cite{Chouldechova17}, the win-loss-ratio used to measure classifier \textit{churn} \cite{Fard+16}, and KL-divergence based metrics used in \textit{quantification} tasks \cite{Fab1,Fab2, Kar+16}. Because these goals are non-linear and are non-continuous in the model parameters, it becomes very challenging to optimize with them, especially when they are used in constraints \cite{Narasimhan18}.

Prior work on optimizing generalized rate metrics has largely focused on unconstrained learning problems. These approaches fall under two broad categories: surrogate-based methods that replace the classifier rates with convex relaxations
\cite{Kar+16,Joachims05,Goh+16,Kar+14,Narasimhan+15b}, and oracle-based methods that formulate multiple cost-sensitive learning tasks, and solve them using an oracle \cite{Parambath+14,Koyejo+14,Narasimhan+15,Narasimhan18,Yan+18,Alabi+18}. Both these approaches have notable deficiencies. The first category of methods rely crucially on the surrogates being close approximations to the rate metric, and perform poorly when this is not the case (see e.g. experiments in \cite{Narasimhan+15}). The use of surrogates becomes particularly problematic with constrained training problems, as relaxing the constraints with convex upper bounds can result in solutions that are over-constrained or infeasible  \cite{Cotter+19}. 
The second category of methods assume access to a near-optimal cost-sensitive oracle, which is usually unrealistic in practice. %\footnote{It is common to implement this oracle through a \textit{plug-in} approach \cite{Narasimhan+15, Narasimhan18, Yan+18, Donini+18}, but this method requires accurate estimates of conditional-class probabilities, which is again hard to obtain in practice.}

In this paper, we present a three-player game approach for learning problems  where both the objective and constraints can be defined by general functions of rates. The three players optimize over model parameters, Lagrange multipliers and slack variables to produce a game equilibrium. Our approach generalizes many existing algorithms (see Table \ref{tab:choices}), and makes possible new algorithms with more flexibility and tighter handling of non-linear rate constraints. Specifically, we give a new method (Algorithm \ref{algo:lagrangian-surrogate}) that can handle a wider range of performance metrics than previous surrogate methods (such as e.g.\ KL-divergence based metrics that only take inputs from a restricted range), and can be applied to constrained training problems without the risk of over-constraining the model because it needs to use surrogates less. To our knowledge, this is the first practical, surrogate-based approach that can handle constraints on generalized rate metrics. 

We show convergence of our algorithms for objectives and constraints that are \textit{convex} functions of rates. This result builds on previous work by Cotter et al.\ \cite{Cotter+19,Cotter+19b} for handling linear rate constraints, and additionally resolves an unanswered question in their work on the convergence of Lagrangian optimizers for non-zero-sum games. We also extend our framework to develop heuristics for optimizing performance measures that are a \textit{sum-of-ratios} of rates (e.g.\ constraints on predictive parity and F-measure), and demonstrate their utility on real-world tasks.

%We make the following contributions: (1) We provide a template for solving these problems by formulating a game between multiple players who optimize over model parameters, Lagrange multipliers and auxiliary slack variables, and finding an (approximate) equilibrium of the game. (2) We show that by choosing different strategies and objectives for the players, one can our framework is fairly general and includes three previous algorithms for generalized rate metrics as special cases. (3) 

\if 0
\begin{itemize}
\item We present a general recipe for  the problem described, where we introduce auxiliary slack variables to decouple the ``non-linear'' portion of the problem from the rate functions, formulate a Lagrangian max-min problem for the resulting constrained optimization problem, and solve for an equilibrium of this game. 
\item We use this framework to develop an algorithm assuming access to a cost-sensitive optimization oracle, and a more practical algorithm based on surrogate functions. 
\item We show optimally and feasibility of the oracle-based algorithm when $\psi$ is convex. For the  surrogate-based algorithm, we convergence which is optimal w.r.t.\ a surrogate comparator, but feasible w.r.t.\ the original constraints.
\item As a corollary of our convergence result, we resolve an open question in Cotter et al. \cite{xx} ...
\item Our framework unifies several previous algorithms for non-decomposable performance metrics 
into one framework
\cite{xx} as special cases, and extends ...
\end{itemize}
\fi

\subsection{Related work}  
The problem of optimizing general functions of prediction rates was first addressed by Joachims (2005) \cite{Joachims05}, who proposed a structural SVM based surrogate for common metrics. This method does not however scale to large multiclass problems or to fairness problems with a large number of protected groups, as its running time increases exponentially with the number of classes/groups. Following this seminal work, there has been a series of alternate approaches, including  plug-in methods
that tune a thresholds on estimates from a class probability oracle \cite{Ye+12,Koyejo+14,Narasimhan+14,Yan+18}, approaches that break-down the learning problem into multiple cost-sensitive learning tasks and solve them using an oracle \cite{Parambath+14,Narasimhan+15,Alabi+18}, and stochastic optimization methods that optimize surrogate approximations to the metric \cite{Kar+14,Narasimhan+15b,Kar+16}. Our framework encompasses many of these methods as special cases.

One class of evaluation measure that has received much attention are fractional-linear functions of rates \cite{Parambath+14, Koyejo+14, Narasimhan+15, Busa+15, Liu+18}, such as for example, the F-measure. Many of these works exploit the pseudo-convex structure of the F-measure, but this property is absent for the problems that we consider where we need to handle sums or differences of ratios. Pan et al. (2016) \cite{Weiwei+16} provide a heuristic approach for unconstrained sums of ratios, and recently Donini et al.\ (2018) \cite{Donini+18} handle constraints that are sums of ratios, but do so by solving a large number of linearly constrained sub-problems, that grow exponentially with the number of rates. In contrast, 
% by extending ideas from Pan et al.\  %\cite{Weiwei+16} 
we provide practical algorithms that can handle both objectives and constraints that are sums of ratios of rates.
%{\color{green}verify}

All the above methods solve unconstrained learning problems. When it comes to constrained optimization, much of the focus has been on \textit{linear} constraints on rates, e.g.\ to impose specific fairness goals %. For example, many fairness goals can be expressed as linear constraints on a model's prediction rates
\cite{Goh+16, Zafar+17}. Recent approaches handle linear rate constraints by computing an equilibrium of a game between players who optimize model parameters $\theta$ and Lagrange multipliers $\lambda$ \cite{Agarwal+18, Kearns+18, Donini+18, Cotter+19,Cotter+19b}. Of these, the closest  to us is Cotter et al. (2019) \cite{Cotter+19}, who propose having only the $\theta$-player optimize a surrogate objective, and having the $\lambda$-player use the original rates. We adapt and build on this idea to handle general functions of rates and provide theoretical guarantees for a simpler Lagrangian algorithm than the one analyzed in Cotter et al.

\section{Problem Setup and Contributions}
\label{sec:equivalence}
\textbf{Notations.} For a continuous convex function $h: \R^n \> \R$, we denote the domain of $h$ by $\dom\, h$, a \textit{sub-gradient} of $h$ at point $z \in \dom\, h$ by $\nabla h(z)$ and the set of all sub-gradients of $h$ at $z$ by $\partial h(z)$. For a strictly convex function $h$, we will use $\nabla h(z)$ to denote the \textit{gradient} of $h$ at $z$.

Let $\X \subseteq \R^d$ be an instance space and $\Y = \{\pm 1\}$ be the label space. %We will denote $p = \P(y=1)$. 
Let $f_\theta: \X \> \R$ be a prediction model, %that maps instances $x$ to a real value $f(x) \in \R$, 
parameterized by $\theta \in \Theta$. Given a model $f_\theta$ and some distribution $D$ over $\X \times \Y$, we define the classifier's rate on $D$ as:
$$
R(\theta; D) \,=\, \E_{(X,Y) \sim D}\left[\mathbb{I}\big\{Y \,=\, \sign(f_\theta(X))\big\}\right].
$$
%, and $r: \Y \times \R \> \R_+$ assigns a score $r(y, \hat{y})$ for label $y \in \Y$ and prediction $\hat{y} \in \R$.
For example, if $D_A$ denotes the distribution over a sub-population $A \subseteq X$, then $R(\theta; D_A)$  gives us the accuracy  of $f_\theta$ on this sub-population. If
$D_+$ denotes a conditional distribution over the positively-labeled examples, then
 $R(\theta; D_+)$ is the true-positive rate $\TPR(\theta)$ of $f_\theta$ and $1 - R(\theta; D_+)$ is the false-negative rate $\FNR(\theta)$ of $f_\theta$. Similarly, if $D_-$ denotes a conditional distribution over the negatively-labeled examples, then $\TNR(\theta) \,=\, R(\theta; D_-)$ is the true-negative rate of $f_\theta$, and $\FPR(\theta) \,=\, 1 - R(\theta; D_-)$ is the false-positive rate of $f_\theta$. Further, the true-positive and false-positive proportions are given by $\TP(\theta) \,=\, p\,\TPR(\theta)$ and $\FP(\theta) \,=\, (1-p)\,\FPR(\theta)$, where $p = \P(Y=1)$; we can similarly define the true negative proportion $\TN$ and false negative proportion $\FN$. 
 
 We will consider several distributions $D_1, \ldots, D_K$ over $\X \times \Y$, and use the short-hand $R_k(\theta)$ to denote the rate function $R(\theta; D_k)$. We will denote a vector of $K$ rates as: $\bR(\theta) \,=\, [R_1(\theta), \ldots, R_K(\theta)]^\top$. We assume we have access to unbiased estimates of the rates
$\hat{R}_k(\theta) \,=\, \frac{1}{n}\sum_{i=1}^n \mathbb{I}\big\{y^k_i \,=\, \sign(f_\theta(x^k_i))\big\}$,  computed, for example, from samples $S = \{(x^k_1, y^k_1),\ldots, (x^k_{n_k},y^k_{n_k})\}\,\sim\,D_k$.

We will also consider stochastic models which are defined by distributions over the model parameters $\Theta$. Let $\Delta_\Theta$ denote the set of all such distributions over $\Theta$, where for any $\mu \in \Delta_\Theta$, $\mu(\theta)$ is the probability mass on model $\theta$. We define the rate $R_k$ for a stochastic model $\mu$ to be the expected value for random draw of a model from $\mu$, so $R_k(\mu) \,=\, \E_{\theta \sim \mu}\left[R_k(\theta)\right]$.

\begin{table}[t]
\caption{Examples of generalized rate metrics. These may arise either as the objective or as constraints. % during training. 
Each $\psi$ is either convex (C), pseudo-convex (PC) or a sum or difference of ratios (SR). \textit{KLD} is the KL-divergence between the true proportion of positives $p = \P(Y=1)$ and the predicted proportion of positives $\hat{p}$. $wins$ denotes the fraction of examples where the new model makes correct predictions, among all examples where it disagrees with the old model. $losses$ refers to the fraction of examples where the new model makes incorrect predictions, among all examples where it disagrees with the old model. $A$ and $B$ refer to different protected groups or slices of the population.}
\vspace{5pt}
\label{tab:measures}
\centering
\begin{tabular}{lccccc}
	\toprule
	\textbf{Measure} & \textbf{Definition} & $\psi$ & \textbf{Type}\\ % & Ref.\\
	\toprule
	G-mean \cite{KubatMa97, Daskalaki+06} & $1 - \sqrt{\TPR \times \TNR}$ & $1 - \sqrt{z_1\,z_2}$ & C \\
	H-mean \cite{Kennedy+09} & $1 - \frac{2}{1/\TPR + 1/\TNR}$ & $1 - \frac{2}{1/z_1 + 1/z_2}$ & C \\
	Q-mean  \cite{Lawrence+98} & $1 - \sqrt{\FPR^2 + \FNR^2}$ & $1 - \sqrt{z_1^2 + z_2^2}$ & C\\
	KL-divergence \cite{Fab1,Fab2} & $p\log(\frac{p}{\hat{p}}) \,+\, p\log(\frac{1-p}{1-\hat{p}})$ & $p\log(\frac{p}{z_1}) \,+\, p\log(\frac{1-p}{z_2})$ & C \\
	F-measure \cite{Manning+08} & $1 - \frac{2\TP}{2\TP \,+\, \FP \,+\, \FN}$ & $1 - \frac{2z_1}{2z_1 + z_2 + z_3}$
	& PC\\
	\midrule
	Predictive parity \cite{Chouldechova17} & $\frac{\TP_A}{\TP_A + \FN_A} \,-\, \frac{\TP_B}{\TP_B + \FN_B}$ & $\frac{z_1}{z_1 + z_2} \,-\,  \frac{z_3}{z_3 + z_4}$ & SR\\
	F-measure parity & $\frac{2\TP_A}{2\TP_A + \FP_A + \FN_A} \,-\, \frac{2\TP_B}{2\TP_B + \FP_B + \FN_B}$ & $\frac{2z_1}{2z_1 + z_2 + z_3} \,-\,  \frac{2z_4}{2z_4 + z_5 + z_6}$ & SR \\
	Churn difference \cite{Fard+16} & $\frac{wins_A}{losses_A} \,-\, \frac{wins_B}{losses_B}$ & $\frac{z_1}{z_2} \,-\,  \frac{z_3}{z_4}$ & SR \\
	\bottomrule
\end{tabular}
\end{table}

\subsection{Generalized Rate Metrics as Objective}
In many real-world applications the performance of a model is evaluated by a function of multiple rates: $\psi\left(R_1(\theta), \ldots, R_K(\theta)\right)$ for some $\psi: \cR \> \R$. % and $\cR \subseteq \R^K$. 
For example, a common evaluation metric used in class-imbalanced classification tasks is the G-mean:
$\text{GM}(\theta) \,=\, 1 \,-\, \sqrt{\TPR(\theta)\,\TNR(\theta)},$ which is a \textit{convex} function of rates \cite{KubatMa97, Daskalaki+06}. 
Similarly, a popular evaluation metric used in text retrieval is the F-measure: $F_1(\theta) \,=\, \frac{2\TP(\theta)}{2\TP(\theta) \,+\, \FP(\theta) \,+\, \FN(\theta)}$, 
which is a {fractional-linear} or \textit{pseudo-convex} function of rates \cite{Lewis:1994}.
See Table \ref{tab:measures} for more examples of evaluation metrics. 
In the last several years, there has been much interest in directly optimizing performance measures of this form during training \cite{Kar+14,Narasimhan+15b,Parambath+14,Koyejo+14,Narasimhan+15,Kar+16,Narasimhan+14}:\footnote{The performance measures we consider are specified by functions $\psi$ that are defined over the non-negative orthant. This covers all the metrics listed in Table \ref{tab:measures}, including the KL-divergence metric, which can be seen as a function 
$\psi(z_1, z_2) = p\log\big(\frac{p}{z_1}\big) \,+\, p\log\big(\frac{1-p}{z_2}\big)$
of two rates $R_1(\theta) = \hat{p}(\theta)$ and $R_2(\theta) = 1 - \hat{p}(\theta)$.}
\begin{equation}
    \min_{\theta \in \Theta}\, \psi\left(R_1(\theta), \ldots, R_K(\theta)\right).
\tag{P1}
\label{eq:custom-opt}
%\vspace{-5pt}
\end{equation}

\subsection{Generalized Rate Metrics as Constraints}
There are also many applications, where one wishes to impose constraints defined by non-linear function of rates, for example to ensure group-specific fairness metrics. Examples include:
\begin{itemize}
\item
\textit{Predictive parity fairness}: Fair classification tasks where one wishes to match the \textit{precision} of a model across different protected groups:
$\left|\frac{\TP_A(\theta)}{\TP_A(\theta) \,+\, \FP_A(\theta)} \,-\, \frac{\TP_B(\theta)}{\TP_B(\theta) \,+\, \FP_B(\theta)}\right|  \,\leq\, \epsilon$
\cite{Chouldechova17}. One may also want to match e.g.\ the \textit{F-measure} of a model across different groups. 
\item
\textit{Distribution matching}: Fairness or quantification tasks where one wishes to match the distribution of a model's outputs across different protected groups \cite{Fab2,Kar+16}. One way to achieve this is to constrain the \textit{KL-divergence} between the overall class proportion $p$ and proportion of predicted positives for each group $\hat{p}_A$, i.e.\ $\text{KLD}(p, \hat{p}_A) \leq \epsilon$, where KLD is convex in $\hat{p}_A$ \cite{Narasimhan18}. 
\item
\textit{Churn}: Problems where one wishes to replace a legacy model with a more accurate model while limiting the changed predictions between the new and old models (possibly across different slices of the user base). This is ideally framed as constraints on the win-loss ratio's \cite{Cotter+19b}, which can be expressed as ratios of rates \cite{Fard+16}. 
\end{itemize}

These and related problems can be framed generally as:
\begin{equation}
    \min_{\theta \in \Theta}\, g(\theta)
    ~~~~\text{s.t.}~~~~
\phi^j\left(R_1(\theta), \ldots, R_K(\theta)\right)
    \,\leq\, 0,\,~~ \forall j \in [J],
\tag{P2}
\label{eq:non-linear-opt}
\end{equation}
for some objective function $g:\Theta \> \R$, and $J$ constraint functions $ \phi^j: \cR \> \R.$

% where $\cR \subseteq \R^K$. 

\subsection{Limitations with Existing Surrogate-based Methods} 
\label{sec:limitations-existing-surrogates}
Existing methods for solving \eqref{eq:custom-opt} either assume access to an idealized oracle that can optimize a linear combination of the rates $R_1, \ldots, R_K$ \cite{Parambath+14,Koyejo+14,Narasimhan+15,Yan+18,Alabi+18}, or resort to optimizing a convex relaxation to the performance metric \cite{Narasimhan+15b, Kar+16}. The latter class of methods enjoy convergence guarantees under more realistic assumptions, and proceed by replacing the rates $R_k$ with continuous and differentiable surrogate functions $\tR_k: \Theta \> \R_+$, so that the resulting objective $\psi\big(\tR_1(\theta), \ldots, \tR_K(\theta)\big)$ is a convex (or pseudo-convex) upper-bound on $\psi\left(R_1(\theta), \ldots, R_K(\theta)\right)$. However, finding such global convex relaxations to \eqref{eq:custom-opt} may not be possible for many common performance metrics. 

For example, if $\psi$ is convex but monotonically decreasing in $R_k$, then one would need to replace $R_k$ with a surrogate $\tR_k$ that is concave in $\theta$ and lower bounds $R_k$: $R_k(\theta) \geq \tR_k(\theta),~\forall \theta\in \Theta$. This can be problematic for metrics such as the G-mean $\psi(z, z') = 1-\sqrt{z_1z_2}$, where any non-trivial concave surrogates that lower bounds the TPR and TNR will necessarily evaluate to negative values for some $\theta$'s, rendering the square root undefined. A similar issue arises with the KLD and F-measure metrics (see Appendix \ref{app:kld} for details). While one can still apply existing surrogate methods to these metrics by restricting the model space $\Theta$ to only those $\theta$'s for which the surrogates are non-negative, this severely limits the class of models for which these methods are guaranteed to converge. 

Surrogate relaxations pose further complications when used to approximate constraints  \eqref{eq:non-linear-opt}, where we run the risk of over-constraining the model space, and making the problem infeasible.

\subsection{Our Contributions}
In light of the above issues, we seek to develop practical algorithms with provable convergence guarantees without having to make unrealistic assumptions (such as access to an oracle) or placing undue restrictions on the model space, while performing at least as well as existing methods in practice.  % Our framework is particularly beneficial in solving problems with constraints, where the use of surrogates upper bounds on metrics can result in over-constraining the model or may render the problem infeasible. By limiting the use of surrogates to only one set of parameter updates, we are able to work with a larger family of metrics and enable a tighter handling of constraints.
% Below, we summarize our main contributions.
We make the following contributions.
\begin{itemize}
\item We present a general template for solving \eqref{eq:custom-opt} and \eqref{eq:non-linear-opt} by introducing slack variables to decouple the rates from the ``non-linear'' portion of the problem, and formulating a max-min game between players who optimize over the model parameters, the slack variables and auxiliary Lagrange multipliers  (see Section \ref{sec:3-player-game}). By plugging in different no-regret  strategies for the three players, one can recover many existing approaches and design new oracle-based algorithms (Algorithm \ref{algo:lagrangian-ideal}) and improved surrogate-based algorithms (Algorithm \ref{algo:lagrangian-surrogate}).
\item One of the main advantages of our framework is that it allows the different players to work with different relaxations to the max-min objective, as long as the player strategies lead to an equilibrium of the game. In particular, the proposed surrogate-based algorithms only use surrogates for optimizing over the model parameters, and use the original rate values for all other updates. As a result, they enjoy provable convergence guarantees for metrics such as the G-mean and KLD even when the surrogates evaluate to non-positive values, and allow for a tighter handling of the non-linear rate constraints. 
\item  % Our algorithms output a stochastic classifier, i.e.\ a distribution over a finite set of model parameters, and are guaranteed to find an approximate equilibrium when the objective/constraints are convex functions of the rates. 
We provide convergence guarantees for convex functions of rates. Our main results are for the proposed surrogate-based algorithms: we show a $\tilde{O}\left(1/\sqrt{T}\right)$ convergence bound for \eqref{eq:custom-opt} and a  $\tilde{O}\left(1/{T^{1/4}}\right)$ bound for \eqref{eq:non-linear-opt}. For the first result, we measure the performance of the learned classifier using \textit{original} rates and that of the comparator using the \textit{surrogate} rates
(see Section \ref{sec:surrogate}). For the second result, despite the mismatch in player objectives, we show that the learned classifier is near-feasible on constraints defined by the \textit{original} rates (see Section \ref{sec:convex-constraint}). 
\item As a corollary, we resolve an unanswered question in Cotter et al.\ (2019) \cite{Cotter+19}. Cotter et al.\ propose a version of our surrogate-based algorithm for linear rate constraints, but leave its convergence properties unanswered. Our result provides a convergence bound for a generalization of their algorithm.
\item We extend our framework to handle constraints that are a sum of ratios of rates and provide two heuristic algorithms (see Section \ref{sec:sum-ratios-constraint}). The first algorithm directly optimizes ratios of slack variables (Algorithm \ref{algo:slack-ratios}), and the second is based on a biconvex formulation \cite{Weiwei+16} (Algorithm \ref{algo:biconvex}). 
\item Finally, we conduct experiments on two constrained learning tasks with several benchmark and real-world fairness datasets. In the first task, we present cases where our new surrogate-based algorithm often yields significant improvements over an existing surrogate-based approach. %perform better than \cite{Narasimhan18} (while doing so without having to make an idealistic oracle assumption).
In the second task, we show that the proposed algorithms are better at optimizing objectives/constraints with sum-of-ratios of rates compared to existing baselines.
\end{itemize}

% All the proofs are provided in the appendix.
% ~\\[-5pt]

We note that all our game formulations can be equivalently regarded as \textit{two-player} games where one player is in-charge of the parameters that need to be minimized, and the other player in-charge the parameters that need to be maximized. We will however use the three-player
viewpoint in the text as we find it to be a useful way to think about the problem algorithmically in that the three sets of optimization parameters can use
different  algorithms (see Table \ref{tab:choices}).

\if 0
\paragraph{Stochastic models}
A key challenge in solving \eqref{eq:custom-opt} and \eqref{eq:non-linear-opt} is that the rate functions $R_k$ are non-continuous functions of the model parameters. %While one can try to approximate the rates with upper-bounding continuous surrogate functions, this may over-constrain the model or may yield a solution that is sub-optimal for the original problem. 
A common approach to handling this issue is to re-formulate the learning problem over a larger space of stochastic models $\mu$ obtained from distributions over $\cH$.  
Let $\Delta_\Theta$ denote the set of all such distributions over $\Theta$, where for any $\mu \in \Delta_\Theta$, $\mu(\theta)$ is the probability mass on model $\theta$. We then define the rate $R_k$ for a stochastic model $\mu$ as the  expected value it yields for a random draw of a model from $\mu$, i.e.\ by $R_k(\mu) \,=\, \E_{\theta \sim \mu}\left[R_k(\theta)\right]$. Note that the rate functions $R_k(\mu)$ are now linear in the model probabilities $\mu(\theta)$.
\fi

\section{Generalized Rate Metric Objective}
\begin{table}
    \centering
    \caption{Algorithms for \eqref{eq:custom-opt}--\eqref{eq:sum-of-ratios-opt} as three-player games. Frank-Wolfe \cite{Narasimhan+15}, SPADE \cite{Narasimhan+15b}, NEMSIS \cite{Kar+16} are previous algorithms. Alg.\ 1--3 are the proposed methods. Each player can do Best Response (BR), Online Gradient Descent (OGD) or Follow-The-Leader (FTL), and the game is zero-sum (ZS) or not. The first five algorithms find an approximate Nash or Coarse-Correlated (C.C.) equilibrium. Alg.\ 3 seeks to handle constraints that are sum-of-ratios of rates \eqref{eq:sum-of-ratios-opt} (see Section \ref{sec:sum-ratios-constraint});  because the constraints are non-convex in the rates, it is not guaranteed to find an equilibrium.}
    \vspace{5pt}
    \label{tab:choices}
    \begin{tabular}{lccccccccc}  %|c|c|c|c|c|c|c|
        \toprule
        &  & \multicolumn{3}{c}{\textbf{Player Objective}} & 
        \multicolumn{3}{c}{\textbf{Player Strategy}} \\
         \cmidrule{3-5} \cmidrule{6-8}
         \textbf{Alg.}  & \textbf{Prob.} & $\xi$ &  $\lambda$ & $\theta$ & $\xi$ &  $\lambda$ & $\theta$ &\textbf{ZS} &\textbf{Equil.} 
         \\
         \midrule
               F-W & \ref{eq:custom-opt}  & - & $\Big(\displaystyle\min_\xi \cL_1(\xi; \cdot)\Big) + \cL_2(\theta; \cdot)$  &
         $\cL_2$ &
         - &  FTL &
         BR &
         \cmark &
         Nash 
         \\
        SPADE & \ref{eq:custom-opt} & $\cL_1$ &  $\cL_1 + \tL_2$ &
         $\tL_2$ &
         BR &  OGD & 
         OGD &
         \cmark &
         Nash 
         \\
         NEMSIS & \ref{eq:custom-opt} & - & $\Big(\displaystyle\min_\xi \cL_1(\xi; \cdot)\Big) + \tL_2(\theta; \cdot)$  &
         $\tL_2$ &
         - & FTL &
         OGD &
         \cmark &
         Nash 
         \\
         Alg. \ref{algo:lagrangian-ideal} & \ref{eq:custom-opt}, \ref{eq:non-linear-opt} & $\cL_1$ & $\cL_1 + \cL_2$ &
         $\cL_2$ &
         BR & OGD &
         BR &
         \cmark &
         Nash 
         \\
         Alg.
      \ref{algo:lagrangian-surrogate} & \ref{eq:custom-opt}, \ref{eq:non-linear-opt} & $\cL_1$ & $\cL_1 + \cL_2$ &
         $\tL_2$ &
         BR &  OGD &
         OGD &
         \xmark &
         C. C.
         \\
         Alg.
      \ref{algo:slack-ratios} & \ref{eq:sum-of-ratios-opt} & $\cL_1$ & $\cL_1 + \cL_2$ &
         $\tL_2$ &
         OGD &  OGD &
         OGD &
         \xmark &
         -
         \\
    %      Alg.
    %   \ref{algo:biconvex} & \ref{eq:sum-of-ratios-opt} & $\cL_1$ & $\cL_1 + \cL_2$ &
    %      $\tL_2$ &
    %      BR &  OGD &
    %      OGD &
    %      \xmark &
    %      -
    %      \\
        %  $\cL_1 + \cL_2$ &
        %  $\tL_2$ &
        %  (OGD, OGD) &
        %  OGD &
        %  \xmark &
        %  Coarse-correlated &
        %  Algorithm
        %  \ref{algo:lagrangian-all-ogd}
        %  \\[2pt]
         \bottomrule
    \end{tabular}
    %\vspace{-10pt}
\end{table}

In this section, we present algorithms for solving the first problem we presented above \eqref{eq:custom-opt}, which is to optimize a generalized rate metric without constraints. We assume  $\psi$  is strictly jointly convex, monotonic, and  $L$-Lispchitz w.r.t.\ the $\ell_\infty$-norm in $[0, 1]^K$. For ease of exposition, we assume that $\psi$ is monotonically increasing in all arguments, although our approach easily extends to metrics that are \textit{monotonically increasing in some arguments} and \textit{monotonically decreasing in others}.

\subsection{Three-player Game Formulation} 
\label{sec:3-player-game}
We equivalently re-write \eqref{eq:custom-opt}
to de-couple the rates $R_k$  from the non-linear function $\psi$ by introducing auxiliary variables $\xi_1, \ldots, \xi_K$:
\begin{equation}
    \min_{\theta \in \Theta,\, \xi \in \cR}\, \psi\left(\xi_1, \ldots, \xi_K\right)~~\text{s.t.}~~R_k(\theta) \,\leq\, \xi_k,~\forall k \in [K].
\label{eq:custom-opt-as-constrained}
\end{equation}
A standard approach for solving \eqref{eq:custom-opt-as-constrained} is to write the Lagrangian for the problem with Lagrange multipliers $\lambda_1, \ldots, \lambda_K \in \R_+$ for the $K$ constraints: %Separating out the penalty terms into two sums:
\begin{eqnarray*}
\lefteqn{\cL(\theta, \xi; \lambda)}\\[-5pt]
&=&
\psi\left(\xi_1, \ldots, \xi_K\right) \,+\, 
\sum_{k=1}^K\lambda_k\, (R_k(\theta)  \,-\, \xi_k)
~~=~
\underbrace{
\psi(\xi_1, \ldots, \xi_K) \,-\, \sum_{k=1}^K \lambda_k \xi_k}_{\cL_1(\xi; \lambda)}
\,+\, 
\underbrace{
\sum_{k=1}^K \lambda_k\, R_k(\theta)
}_{\cL_2(\theta; \lambda)}.
\label{eq:L-min-expand}
\end{eqnarray*}
Then one maximizes the Lagrangian over $\lambda \in \R_+^K$, and minimizes it over $\theta \in \Theta$ and $\xi \in \cR$:
\begin{equation}
    \max_{\lambda \in \R_+^K}\,\min_{\substack{\theta \in \Theta\\ \xi\in \cR}}\,\cL_1(\xi; \lambda) \,+\, \cL_2(\theta; \lambda)
\label{eq:custom-opt-as-max-min},
\end{equation}
Notice that $\cL_1$ is convex in $\xi$ (by convexity of $\psi$), while $\cL_1$ and $\cL_2$ are linear in $\lambda$. We pose this max-min problem as a zero-sum game between three players: a player who minimizes $\cL_1 + \cL_2$ over $\theta$, a player who minimizes $\cL_1 + \cL_2$ over $\xi$, and a player who  maximizes $\cL_1 + \cL_2$ over $\lambda$. Each of the three players can now use different optimization algorithms customized for their problem. If additionally the Lagrangian was convex in $\theta$, one could solve for an equilibrium of this game and obtain a solution for the primal problem \eqref{eq:custom-opt-as-constrained}. However,  since $\cL_2$ is a weighted sum of rates $R_k(\theta)$, it need not be convex (or even continuous) in $\theta$.

To overcome this difficulty, we expand the solution space from deterministic models $\Theta$ to stochastic models  $\Delta_\Theta$ \cite{Agarwal+18,Cotter+19,Cotter+19b}, and re-formulate \eqref{eq:custom-opt-as-max-min} as a problem that is linear in $\mu$, by replacing each $R_k(\theta)$ with $E_{\theta \sim \mu}[R_k(\theta)]$ in $\cL_2$:
\begin{equation}
    \max_{\lambda \in \Lambda}\,\min_{\substack{\mu \in \Delta_\Theta \\ \xi\in \cR}}\,\cL_1(\xi; \lambda) \,+\, \cL_2(\mu; \lambda).
\label{eq:custom-opt-as-max-min-stochastic}
\end{equation}
Here for technical reasons, we restrict the Lagrange multipliers to a bounded set $\Lambda = \{\lambda \in \R_+\,|\,\|\lambda\|_1 \leq \kappa\}$, we will choose the radius  $\kappa > 0$ later in our theoretical analysis.  %that as long as $\psi$ is convex in its arguments, 
By solving for an equilibrium of this expanded max-min problem, we can find a stochastic model $\mu \in \Delta_\theta$ that minimizes $\psi(R_1(\mu), \ldots, R_K(\mu))$.

There are two approaches that we can take to find an equilibrium of the expanded game. The first approach is to assume access to an oracle that can perform the the minimization of $\cL_2$ over $\mu$ for a fixed $\lambda$ and $\xi$.
Since this is a linear optimization over the simplex, this amounts to performing a minimization over deterministic models in $\Theta$. The second and more realistic approach is to work with surrogates for the rates that are continuous and differentiable in $\theta$. Let $\tR_1, \ldots, \tR_K: \Theta \> \R_+$ be differentiable convex surrogate functions that are upper bounds on the rates:  $R_k(\theta) \leq \tR_k(\theta), ~\forall \theta \in \Theta$. We assume access to (unbiased) stochastic sub-gradients for the surrogates, i.e. for any $\theta \in \Theta$, we have access to a noisy sub-gradient $\nabla_\theta \tR_k(\theta)$ with $\E[\nabla_\theta \tR_k(\theta)] \in \partial_\theta \tR_k(\theta)$. We then define a surrogate-based approximation for $\cL_2$:
\begin{equation}
\pL_2(\theta; \lambda) \,=\,
\textstyle \sum_{k=1}^K\lambda_k\, \tR_k(\theta).
\label{eq:proxy-lagrangian}
\end{equation}
%~\\[-25pt]

All we need to do now is to choose the objective that each player seeks to optimize (true or surrogate \cite{Cotter+19}) and the strategy that they use to optimize their objective, so that the players converge to an equilibrium of the game. Each of these choices lead to a different algorithm for (approximately) solving \eqref{eq:custom-opt}. Table \ref{tab:choices}  summarizes different choices of strategies and objectives for the players, and the type of equilibrium and the algorithm that results from these choices. As we shall see shortly (and also elaborate in Appendix \ref{app:connections}), many existing algorithms can be seen as instances of this template. %These connections that we make to previous methods form one of the contributions of this paper.

\begin{figure}
\begin{minipage}[H]{0.485\textwidth}
\begin{algorithm}[H]
\caption{Oracle-based Optimizer}
\label{algo:lagrangian-ideal}
\begin{algorithmic}
\STATE Initialize: $\blambda^0$
\FOR{$t = 0 $ to $T-1$}
\IF{\eqref{eq:custom-opt}}
\STATE $\xi^{t} \,=\,
\nabla\psi^*(\lambda^t)$
\ELSIF{\eqref{eq:non-linear-opt}}
\STATE Construct $\Phi: \xi \mapsto \sum_{j=1}^J \lambda^t_j\, \xi_k$
\STATE $\xi^{t} \,=\,
  \nabla \Phi^*(
\lambda^t_{J+1}, \ldots, \lambda^t_{J+K})$
\ENDIF
\STATE $\theta^{t} \in \amin{\theta \in \Theta}\, \cL_2(\theta; \lambda^t)$~~~~~[CSO]
%\STATE $\xi^{t+1} \in \amin{\xi \in [0,1]^K}\, \cL(\theta^t, \xi; \lambda^t)$~~~[PGD]\hspace{-10pt}
\STATE $\blambda^{t+1} =\Pi_{\Lambda}\big( \blambda^t\,+\, \eta\,\nabla_{\blambda}\,\cL(\theta^{t}, \bxi^{t};\, \blambda^{t})\big)$
\ENDFOR
\RETURN $\theta^1, \ldots, \theta^T$
\end{algorithmic}
\end{algorithm}
\end{minipage}
\begin{minipage}[H]{0.49\textwidth}
\begin{algorithm}[H]
\caption{Surrogate-based Optimizer}
\label{algo:lagrangian-surrogate}
\begin{algorithmic}
\STATE Initialize: $\theta^0$, $\blambda^0$
\FOR{$t = 0 $ to $T-1$}
\IF{\eqref{eq:custom-opt}}
\STATE $\xi^{t} \,=\,
\nabla \psi^*(\lambda^t)$
\ELSIF{\eqref{eq:non-linear-opt}}
\STATE Construct $\Phi: \xi \mapsto \sum_{j=1}^J \lambda^t_j\, \xi_k$
\STATE $\xi^{t} \,=\,
  \nabla \Phi^*(
\lambda^t_{J+1}, \ldots, \lambda^t_{J+K})$
% ~\\[-0.9cm]
% \begin{eqnarray*}
% %\forall j \in [J],~
% \xi^{t+1}_{j,:} \,=\,
% \begin{cases}
%   \nabla\phi^*\left(
%   \frac{\lambda_{j,1}^t}{\lambda^t_{j,0}},\, \ldots,\, \frac{\lambda_{j,K}^t}{\lambda^t_{j,0}}\right) & \text{if}~\lambda^t_{j,0}>0\\
%     \0^K & \text{o/w}
% \end{cases}, \forall j
% \vspace{-10pt}
% \end{eqnarray*}
\ENDIF
\STATE $\theta^{t+1} \leftarrow \Pi_{\Theta}(\theta^t \,-\, \eta_\theta\,\nabla_\theta\,\tL_2(\theta^t;\, \blambda^t))$
%\STATE $\bxi^{t+1} \leftarrow \bxi^t \,-\, \eta_{\bxi}\,\Pi_{\bxi}(\nabla_\theta\,\tL(\theta^t, \bxi^t;\, \blambda^t))$
\STATE $\blambda^{t+1} \leftarrow \Pi_{\Lambda}(\blambda^t + \eta_{\blambda}\,\nabla_{\blambda}\,\cL(\theta^t, \bxi^{t};\, \blambda^t))$
\ENDFOR
\RETURN $\theta^1, \ldots, \theta^T$
\end{algorithmic}
\end{algorithm}
\end{minipage}
\caption{Optimizers for the unconstrained  problem \eqref{eq:custom-opt} and constrained problem \eqref{eq:non-linear-opt}. Here $\Pi_\Lambda$ denotes the $\ell_1$-projection onto $\Lambda$ and 
$\Pi_\Theta$ denotes the $\ell_2$-projection onto $\Theta$. We denote a (stochastic) gradient of $\cL$ by $\nabla_\lambda \cL( \theta^t ,\xi^t; \blambda^t) \,=\, 
\sum_{k=1}^K (\hat{R}_k(\theta^t) \,-\, \xi^t_k)$,  where $\hat{R}_k(\theta^t)$ is an unbiased estimate of $R_k(\theta^t)$. We denote a (stochastic) sub-gradient of $\tL_2$ by $\nabla_\theta \tL_2( \theta^t; \blambda^t) \,=\,
\sum_{k=1}^K\lambda_k\,\nabla_\theta \tR_k(\theta^t))$, where $\nabla_\theta \tR_k(\theta^t)$ is an unbiased gradient estimate for $\tR_k(\theta^t)$, i.e.\ $\E[\nabla_\theta \tR_k(\theta^t)] \in \partial_\theta \tR_k(\theta^t), ~\forall k \in [K]$.
}
% and $\bar{\lambda}^t \,=\, \sum_{j=1}^J \lambda^t_j$.} %\todohari{$\ell_2$ for $\theta$?}}
%\vspace{-10pt}
\end{figure}

\subsection{Oracle-based Lagrangian Optimizer}
\label{sec:oracle}
As a warm-up illustration of our approach, we first describe an idealized algorithm assuming access to an  oracle that can approximately optimize $\cL_2$ over $\Theta$ (this oracle essentially optimizes a weighted-sum of rates, i.e.\ a cost-sensitive error objective \cite{Elkan01}):
\begin{defn}
\label{defn:cso}
A $\rho$-approximate cost-sensitive optimization (CSO) oracle takes 
% a $w \in \R_+^K$ as input and returns a model $\theta^* \in \Theta$ such that 
$\lambda$ as input and outputs  a model $\theta^* \in \Theta$ such that 
% $\sum_{k=1}^K w_k\, R_k(\theta^*) \,\leq\, \min_{\theta \in \Theta}\,\sum_{k=1}^K w_k\, R_k(\theta) \,+\, \rho.$
% w.p. $\geq 1 - \delta$,~
$\cL_2(\theta^*; \lambda) \,\leq\, \min_{\theta \in \Theta}\,\cL_2(\theta; \lambda) \,+\, \rho.$\footnote{An example of a cost-sensitive oracle is a procedure that minimizes the empirical cost-sensitive error on a training sample $S = \{(x_1, y_1),\ldots, (x_n,y_n)\}$ over the  space of models $\Theta$. Using standard results, one can show that for any $\lambda$, with probability $\geq 1 - \delta$ (over draw of $S$), the empirical error minimizer $\theta^*$ satisfies $\cL_2(\theta^*; \lambda) \,\leq\, \min_{\theta \in \Theta}\,\cL_2(\theta; \lambda) \,+\, \cO\big(\sqrt{\frac{\log(\delta)}{n}}\big).$ While in Definition \ref{defn:cso}, we assume a deterministic guarantee on the cost-sensitive oracle, all our results easily extend to the case where the guarantee on the oracle holds with high probability.}
\end{defn}

We have all three players optimize the true Lagrangians, with the $\theta$-player and $\xi$-player playing best responses to the opponents strategies, i.e.\ they perform full optimization over their parameter space, and the $\lambda$-player running  online gradient descent (OGD), an algorithm with no-regret guarantees
\cite{Zinkevich03}. The $\theta$-player performs best response by using the above oracle to approximately minimize  $\cL_2$ over $\theta$. For the $\xi$-player, the best response optimization can be computed in closed-form: % available for the minimization of $\cL_1$ over $\xi$:
\begin{lem}
Let  $\psi^*: \R_+^K  \> \R$ denote the Fenchel conjugate of $\psi$. Then for any $\lambda$ s.t.\ $\|\lambda\|_1 \leq L$:
$$
\textstyle
\nabla \psi^*(\lambda) \,\in\,
\argmin{\xi \in \cR}\,\cL_1(\xi; \lambda).
% \amin{\xi \in [0,1]^K}\,\left\{\psi(\xi_1, \ldots,(\xi; \lambda) \xi_K) \,-\, \sum_{k=1}^K \lambda_k \xi_k\right\}.
$$.
\vspace{-15pt}
\label{lem:xi-opt-closed-form}
\end{lem}
% Next note that minimization of $\cL_2$ over $\mu$ is a linear optimization  over the simplex. It is therefore sufficient for us to perform this minimization over deterministic models in $\Theta$. Suppose we have access to an oracle that can perform this optimization:
% \begin{defn}
% A $\rho$-approximate cost-sensitive optimization (CSO) oracle takes a $w \in \R_+^K$ as input and returns a model $\theta^* \in \Theta$ such that 
% $\sum_{k=1}^K w_k\, R_k(\theta^*) \,\leq\, \min_{\theta \in \Theta}\,\sum_{k=1}^K w_k\, R_k(\theta) \,+\, \rho.$
% % $\cL_2(\theta^*; \lambda) \,\leq\, \min_{\theta \in \Theta}\,\cL_2(\theta; \lambda) \,+\, \rho.$
% \end{defn}
% For the $\lambda$-player play, we run the online gradient descent (OGD) of \cite{Zinkevich03} that has no-regret guarantees. 
The resulting algorithm outlined in Algorithm \ref{algo:lagrangian-ideal} is 
%, where in each iteration, we perform a gradient-based update on $\theta$, a full minimization over $\xi$ using the closed-form solution in Lemma \ref{lem:xi-opt-closed-form} and a full minimization over $\mu$ using the CSO oracle. 
guaranteed to find an approximate Nash equilibrium of the max-min game in \eqref{eq:custom-opt-as-max-min-stochastic}, and yields an approximate solution to \eqref{eq:custom-opt}:
%\todohari{Lambda player does not play OGD, but grad ascent, also mention gradients can be stochastic}
\begin{thm}
    Let $\theta^1, \ldots, \theta^T$ be the iterates generated by Algorithm \ref{algo:lagrangian-ideal} for \eqref{eq:custom-opt}, and let $\bar{\mu}$ be a stochastic model with a probability mass of $\frac{1}{T}$ on $\theta^t$. %Suppose $\psi$ is strictly convex, monotonically non-decreasing in each argument and $L$-Lispchitz w.r.t.\ the $\ell_\infty$ norm. 
    % Let $\tilde{B}_\theta = \max_{\theta \in \Theta}\,\psi(\tilde{\bR}(\theta))$.
    Define %$B_\theta \geq \max_{\theta, \bxi}\,\cL(\theta, \bxi; \blambda)$ and
    $B_\lambda \,\geq\, \max_{t}\|\nabla_{\blambda}\cL(\bxi^t,\theta^t; \blambda^t)\|_2$. Then setting $\kappa \,=\, L$ and $\eta = \frac{L}{B_{\blambda}\sqrt{2T}}$, we have w.p.\ $\geq 1 - \delta$ over draws of stochastic gradients:
    $$
    \psi\left(\bR(\bar{\mu})\right)
    \,\leq\, \min_{\mu \in \Delta_\Theta}\psi\left(\bR(\mu)\right) %\,+\,4LB_{\lambda}\sqrt{\frac{2}{T}} \,+\, 2\rho.
    \,+\, \cO\left(\sqrt{\frac{\log(1/\delta)}{T}}\right) \,+\, \rho.
    $$
    \label{thm:convergence-ideal}
    \vspace{-5pt}
\end{thm}
See Appendix \ref{app:thm-convergence-ideal} for the proof.

\begin{rem}[\textbf{Frank-Wolfe: a special case}]
A previously proposed oracle-based algorithm for optimizing convex functions of rates is the Frank-Wolfe based approach of Narasimhan et al. (2015) \cite{Narasimhan+15b}. Interestingly, this method can be recovered from our framework by 
reformulating $\lambda$-player's objective to include the minimization over $\xi$: $\cL_{\text{FW}}(\lambda, \theta) \,=\, \min_{\xi} \cL_{1}(\xi; \lambda) + \cL_2(\theta; \lambda)$
and having the $\lambda$-player play the Follow-The-Leader (FTL) algorithm \cite{AbernethJ17} to maximize this objective. % over $\lambda$. 
As before, the $\theta$-player plays best response on $\cL_2$ 
%over $\theta$ 
using the CSO oracle. See Table \ref{tab:guarantees-unconstrained} and Appendix \ref{app:connections} for details.
\end{rem}

\begin{table}
    \centering
    \caption{Convergence guarantees for algorithms for \eqref{eq:custom-opt}. Frank-Wolfe \cite{Narasimhan+15}, SPADE \cite{Narasimhan+15b}, NEMSIS \cite{Kar+16} are previous algorithms.
    %Alg.\ 1--3 are the proposed methods. Each player can do Best Response (BR), Online Gradient Descent (OGD) or Follow-The-Leader (FTL), and the game is zero-sum (ZS) or not. The first five algorithms find an approximate Nash or Coarse-Correlated (C.C.) equilibrium. Alg.\ 3 seeks to optimize a non-convex function of rates and is not guranteed to find an equilibrium.
    Algorithms require $\psi$ to be \textit{smooth} or not; apply either to a general convex model class $\Theta$ or to only a \textit{restricted} subset of models for which the surrogates evaluate to values within the domain of $\psi$; may need access to a cost-sensitive \textit{oracle} or not; output either a {stochastic} classifier $\bar{\mu}$ or a deterministic classifier $\bar{\theta}$. The bounds are on the optimality gap.
    }
    \vspace{5pt}
    \label{tab:guarantees-unconstrained}
    \begin{tabular}{lccccccccc}  %|c|c|c|c|c|c|c|
        \toprule
         \textbf{Alg.} & 
         \textbf{Smooth?} &
         \textbf{Restr.\ $\Theta$?} &
         \textbf{Oracle?} & 
        \textbf{Opt.\ Gap} &
         \textbf{Opt.\ Bound} &
         \\
         \midrule
               F-W & 
               \cmark &
               \xmark &
               \cmark &
               $
               \displaystyle
                \psi\left(\bR(\bar{\mu})\right)
                \,-\, \min_{\mu \in \Delta_\Theta}\psi\left(\bR(\mu)\right)
                $
                &
                $\tO\left(\frac{1}{T} \,+\, \rho\right)$
         \\
               SPADE & 
               \xmark &
               \cmark &
               \xmark &
               $
               \displaystyle
                \psi\big(\tbR(\bar{\theta})\big)
                \,-\, \min_{\theta \in \Theta}\psi\big(\tbR(\theta)\big)
                $
                &
                $
                \tO\Big(\sqrt{\frac{1}{T}}\Big)
                $
        \\
               NEMSIS & 
               \cmark &
               \cmark &
               \xmark &
               $
               \displaystyle
                \psi\big(\tbR(\bar{\theta})\big)
                \,-\, \min_{\theta \in \Theta}\psi\big(\tbR(\theta)\big)
                $
                &
                $
                \tO\Big(\sqrt{\frac{1}{T}}\Big)
                $
        \\
        \midrule
              Alg.\ \ref{algo:lagrangian-ideal} & 
               \xmark &
               \xmark &
               \cmark &
               $
               \displaystyle
                \psi\left(\bR(\bar{\mu})\right)
                \,-\, \min_{\mu \in \Delta_\Theta}\psi\left(\bR(\mu)\right)
                $
                &
                $\tO\Big(\sqrt{\frac{1}{T}}\Big) \,+\, \rho$
        \\
               Alg.\ \ref{algo:lagrangian-surrogate} & 
               \xmark &
               \xmark &
               \xmark &
               $
               \displaystyle
                \psi\big(\bR(\bar{\mu})\big)
                \,-\, \min_{\theta \in \tilde{\Theta}}\psi\big(\tbR(\theta)\big)
                $
                &
                $
                \tO\Big(\sqrt{\frac{1}{T}}\Big)
                $
        \\
         \bottomrule
    \end{tabular}
    %\vspace{-10pt}
\end{table}

\subsection{Surrogate-based Lagrangian Optimizer}
\label{sec:surrogate}
While the CSO oracle may be available in some special cases (e.g.\ when $\Theta$ is finite or when the underlying conditional-class probabilities can be estimated accurately), in many practical scenarios, it is not realistic to assume access to an oracle that can optimize non-continuous rates. We now provide a more practical algorithm for solving \eqref{eq:custom-opt} where the $\theta$-player optimizes the surrogate Lagrangian function $\tL_2$ in \eqref{eq:proxy-lagrangian} instead of $\cL_2$ and uses stochastic gradients $\nabla_\theta\,\tL_2(\theta; \lambda)$. The $\xi$- and $\lambda$-players, however, continue to operate on the true Lagrangian functions $\cL_1$ and $\cL_2$, which are continuous in the parameters $\xi$ and $\lambda$ that these players optimize. 

In our proposed approach, outlined in Algorithm \ref{algo:lagrangian-surrogate}, both the $\theta$-player and $\lambda$-player now run online gradient descent algorithms, while the $\xi$-player plays its best response at each iteration. Since it is the $\theta$-player alone who optimizes a surrogate, the resulting game between the three players is no longer zero-sum. Yet, we are able to show that the player strategies converge to an approximate \textit{coarse-correlated} (C. C.) equilibrium of the game, and yields an approximate solution to \eqref{eq:custom-opt}.

% that builds on the approach of Cotter et al.\ \cite{xx}, and replaces the rate functions with smooth surrogate functions. 

% Let $\tR_1, \ldots, \tR_K: \Theta \> \R_+$ be differentiable surrogate functions that are upper bounds on the rates:  $R_k(\theta) \leq \tR_k(\theta), ~\forall \theta \in \Theta$. 
% Observe that it is only for the minimization over $\theta$ that we need the surrogates, and the maximization of the $\lambda$ can continue to use the original rates. Consequently, we define a proxy Lagrangian function:
% \[
% \pL_2(\theta; \lambda) \,=\,
% \sum_{k=1}^K\lambda_k\, \tR_k(\theta).
% \]
% %where $\tR_k(\mu) \,=\, \E_{\theta \sim \mu}\left[R_k(\theta)\right]$. 
%
% The $\theta$-player now seeks to minimize the proxy Lagrangian $\pL_2$ over $\theta \in \Theta$. Like before, the $\xi$-player seeks to minimize $\cL_1$ and the $\lambda$ player seeks to maximize $\cL_1 + \cL_2$. This results in a non-zero-sum game between the three players and the equilibrium analysis used for the previous algorithm does not apply. Our proposed approach is outlined in Algorithm \ref{algo:lagrangian-surrogate}, where both $\theta$-player and $\lambda$-player run no-regret algorithms, while the $\xi$-player plays its best response at each iteration. 
\begin{thm}
    Let $\theta^1, \ldots, \theta^T$ be the iterates of Algorithm \ref{algo:lagrangian-surrogate} for \eqref{eq:custom-opt}, and let $\bar{\mu}$ be a stochastic model with probability $\frac{1}{T}$ on $\theta^t$. % Suppose $\psi$ is strictly convex, monotonically non-decreasing in each argument and $L$-Lispchitz w.r.t.\ $\ell_\infty$ norm.  
    Let $\tilde{\Theta} \,=\, \big\{\theta \in \Theta \,|\, 
    \tilde{\bR}(\theta) \,\in\, \dom\, \psi
    \big\}$. 
    Let $B_\Theta \geq \max_{\theta \in \Theta}\,\|\theta\|_2$, $B_\theta \geq \max_{t}\|\nabla_{\theta}\tL_2( \theta^t; \blambda^t)\|_2$ and
    $B_\lambda \,\geq\, \max_{t}\|\nabla_{\blambda}\cL(\bxi^t, \theta^t; \blambda^t)\|_2$. Then setting $\kappa \,=\, L$, $\eta_\theta = \frac{B_\Theta}{B_{\theta}\sqrt{T}}$ and $\eta_\lambda = \frac{L}{B_{\blambda}\sqrt{2T}}$, we have w.p. $\geq 1 - \delta$ over draws of stochastic gradients:\vspace{-5pt}
    $$
    \psi\big({\bR}(\bar{\mu})\big)
    \,\leq\, \min_{\theta \in \tilde{\Theta}}\psi\big(\tilde{\bR}(\theta)\big) 
    \,+\, \cO\left(\sqrt{\frac{\log(1/\delta)}{T}}\right).
    % \frac{\tilde{B}_\theta}{T^{1/4}}
    % \,+\,
    % \left(\frac{1}{T^{1/4}}+1\right)\left(B_\Theta\,B_\theta\sqrt{\frac{2}{T}} + L\,B_\lambda\frac{\sqrt{2}}{T^{1/4}}\right).
    $$
    \label{thm:convergence-surrogate}
    \vspace{-5pt}
\end{thm}
See Appendix \ref{app:thm-convergence-surrogate} for the proof.
%
% The proof shows that Algorithm \ref{algo:lagrangian-surrogate} finds a coarse-correlated equilibrium for the formulated game, and proceed by extending Agarwal et al.'s analysis for zero-sum games \cite{xx}. 
%This result shows convergence of Algorithm \ref{algo:lagrangian-surrogate} to the optimal solution for a proxy objective $\psi(\tilde{\bR}(\cdot))$ defined on surrogate functions. 
Note the right-hand side contains the optimal value for the surrogate objective $\psi(\tilde{\bR}(\cdot))$ and not for the original performance metric. This is unsurprising given the $\theta$-player's inability to work with the true rates. Also,
while Algorithm \ref{algo:lagrangian-surrogate} can be applied to optimize over a general (bounded) convex model class $\Theta$, the comparator for our guarantee is a subset of models $\tilde{\Theta} \subseteq \Theta$ for which $\psi(\tilde{\bR}(\theta))$ is defined. This is needed as the surrogate $\tilde{\bR}(\theta)$ may output values outside the domain of $\psi$.  % We expect to improve on this int out that may be able to can obtain a better convergence rate for that the poorer convergence rate compared to Algorithm \ref{algo:lagrangian-ideal} is a consequence of the mismatch in Lagrangian objectives between the $\theta$-player and the other players (we provide a more detailed discussion in the appendix).

\begin{rem}[\textbf{SPADE, NEMSIS: special cases of our approach}]
Our approach includes two previous surrogate-based algorithms %for optimizing convex rate metrics 
as special cases: SPADE \cite{Narasimhan+15} and NEMSIS \cite{Kar+16}. SPADE can be recovered from our framework by  having the same player strategies as Algorithm \ref{algo:lagrangian-surrogate} but with both  the $\theta$- and $\lambda$-players optimizing surrogate objectives, i.e.\ with the $\theta$-player minimizing $\tL_2$ and the $\lambda$-player maximizing $\cL_1 + \tL_2$. NEMSIS also uses surrogates for both the $\theta$ and $\lambda$ updates. It can be recovered by having the $\theta$-player run OGD on $\tL_2$, and having the $\lambda$-player play FTL over $\lambda$ on the combined objective $\tL_{\text{NEM}}(\theta, \lambda) = \min_\xi \cL_1(\xi; \lambda) + \tL_2(\theta; \lambda)$. See Table \ref{tab:guarantees-unconstrained} and Appendix \ref{app:connections} for details.
\end{rem}
\begin{rem}[\textbf{Application to wider range of metrics}]
As noted in Section \ref{sec:limitations-existing-surrogates},  because of their strong reliance on surrogates, SPADE and NEMSIS cannot be applied directly to  functions $\psi$  that take inputs from a restricted range (e.g.\ KL-divergence), unless the surrogates are also bounded in the same range. 
In Appendix \ref{app:kld}, we point out scenarios where the NEMSIS method \cite{Kar+16} fails to optimize the KL-divergence metric, unless the model is sufficiently regularized to not output large negative values.
%\footnote{E.g.\ consider a surrogate $-\log(\tR(\theta))$ for upper-bounding the negative logarithm; for this to be a convex in $\theta$,  $\tR$ needs to either be a constant or be allowed to take negative values (see Appendix \ref{app:kld} for details).}
% E.g.\  consider the log prediction rate term that appears in KLD: $-\log(\E_{X}(\mathbb{I}\{f_\theta(X) \geq 0\}))$. Here $-\log$ is convex. A hinge-based convex surrogate upper bounds on this metric would look like $-\log(\E_{X}(\min\{1, 1 + f_\theta(X)\}))$, which is not defined when the term within the log is negative and can cause potential numerical issues when we apply GD to it.  
% Indeed in practice, one could strongly regularize the model or objective to avoid values outside the allowed range (which how we believe Kar et al. (2016) \cite{Kar+16} handle this issue), but this may restrict us from finding good solutions. 
Algorithm \ref{algo:lagrangian-surrogate} has no such restriction and can be applied even if the outputs of the surrogates are not within the domain of  $\psi$. This is because it
it uses the original rates for updates on $\lambda$. As a result, the game play between $\xi$ and $\lambda$ never produces values  that are outside the domain of $\psi$.  % decouples the function $\psi$ from the rates using  slacks $\xi$ (that be constrained to be in the right range), and uses surrogates only for the updates on $\theta$. %It is thus more widely applicable. 
%On the other hand, Algorithm \ref{algo:lagrangian-surrogate} does not have this problem as in our formulation, the function $\psi$ takes the slack variables $\xi$ as input, which are constrained to be in $[0,1]$.
% E.g.\  consider the negative log-rate   $-\log(\E_{X}(\mathbb{I}\{f_\theta(X) \geq 0\}))$; a hinge-based convex surrogate that upper bounds this metric would look like $-\log(\E_{X}(\min\{1, 1 + f_\theta(X)\}))$, which is clearly  not defined when the term within the log is negative. 
%Because Algorithm \ref{algo:lagrangian-surrogate} uses original rates for updates on $\lambda$ and not surrogates, it does not require the surrogates to be bounded in a specific range.
\end{rem}

\if 0
\begin{rem}[\textbf{Extending to black-box metrics}]
Both Algorithms \ref{algo:lagrangian-ideal}
and \ref{algo:lagrangian-surrogate} need to compute the Fenchel conjugate of $\psi$ to implement the best response strategy for the $\xi$-player. This is usually possible for common convex functions such as geometric or harmonic mean. For functions $\psi$ for which the Fenchel conjugate is difficult to compute analytically, one can have the $\xi$-player execute OGD on $\cL_1$ instead of playing best response (see Algorithm \ref{algo:lagrangian-all-ogd} in the appendix). Even in scenarios where the form of $\psi$ is not known, and it can be accessed only as a black-box, as long as $K$ is small, we can still apply our algorithms by partitioning the range $[0,1]^K$ into a $K$-dimensional grid and performing the minimization over $\xi$ through a brute-force search over this grid.
\end{rem}
\fi

\section{Generalized Rate Metric Constraints}
In this section, we describe how to apply our approach to the constrained optimization problem in \eqref{eq:non-linear-opt}.
%For ease of exposition, we will assume there is a single constraint $\phi(R(\theta_1),\ldots, R(\theta_K)) \,\leq\, 0$, and 
We consider two cases: the constraint functions $\phi^j$'s are (1) convex in the rates, and (2) summation of ratios in the rates. 

\subsection{Case of Convex $\phi$}
\label{sec:convex-constraint}
We start with the case where the $\phi^j$'s are strictly jointly convex, monotonic in each argument and $L$-Lipschitz w.r.t.\ the $\ell_\infty$-norm in $[0,1]^K$, and $g$ is a bounded convex function. For convenience, we assume that the $\phi^j$'s are monotonically increasing in all arguments.  Constraints on the KL-divergence and G-mean metrics are examples of this setting. 

We introduce 
 a set auxiliary variables $\xi_{1}, \ldots, \xi_{K}$ for the $K$ rate functions and re-write \eqref{eq:non-linear-opt} as:
%  This allows us to decouple the rates $R_k$ from the $\psi$ and $\phi$, while also allowing us to optimize $\psi$ and $\phi$ using two separate players. The following is then equivalent to \eqref{eq:non-linear-opt} as:
\begin{equation*}
    \min_{\theta \in \Theta,~ \xi \in \cR}\, g(\theta)
%\label{eq:non-linear-opt-slack}
\end{equation*}
\vspace{-10pt}
\begin{eqnarray}
~&\text{s.t.}&
\phi^j\left(\xi_{1}, \ldots, \xi_{K}\right)
    \,\leq\, 0,~~ \forall j \in [J]
    \label{eq:cons-phi}
\\
% ~&&
% R_k(\theta) \,\leq\, \xi_{K+k},~~ \forall k \in [K]
%     \label{eq:cons-xi-prime}
% \\
~&&
R_k(\theta) \,\leq\, \xi_{k},~~ \forall k \in [K].
    \label{eq:cons-xi}
\end{eqnarray}
The Lagrangian for the re-written problem is given below, where $\lambda_{1}, \ldots, \lambda_{J} \in \R_+$ and
$\lambda_{J+1}, \ldots, \lambda_{J+K} \in \R_+$  are the Lagrange multipliers for the constraints in \eqref{eq:cons-phi} and \eqref{eq:cons-xi} respectively:
\begin{eqnarray}
\vspace{-5pt}
\cL(\theta, \xi; \lambda) %&=& \textstyle g(\theta) \,+\, \sum_{j=1}^J \lambda_{j}\, 
%\phi^j\left(\xi_{1}, \ldots, \xi_{K}\right)
%\,+\,
%\sum_{k=1}^K\lambda_{J+k}\, (R_k(\theta)  \,-\, \xi_{k})
%\\
&=&
\textstyle
\underbrace{
\sum_{j=1}^J
\lambda_{j}\,
\phi^j(\xi_{1}, \ldots, \xi_K) \,-\, \sum_{k=1}^K \lambda_{J+k}\, \xi_{k}
}_{\cL_1(\xi; \lambda)}
\,+\,
\underbrace{
g(\theta)
\,+\,
\sum_{k=1}^K \lambda_{J+k}\, R_k(\theta)
}_{\cL_2(\theta; \lambda)}.
% \,+\,
% \underbrace{
% \psi(\xi) \,-\, \sum_{k=1}^K \lambda_k \xi_k}_{\cL_2(\xi_{:K}; \lambda)}
\label{eq:L-min-expand-constrained}
\vspace{-5pt}
\end{eqnarray}

As before, we expand the search space to include stochastic models in $\Delta_\Theta$, restrict the Lagrange multipliers to a bounded set $\Lambda \,=\, \big\{\lambda \in \R_+^{J+K}\,\big|\, 
%\|\lambda_{:J}\|_1 \leq \kappa_1, \, \|\lambda_{J+1:}\|_1 \leq \kappa_2\big\}$, 
\|\lambda\|_1 \leq \kappa\big\}$, 
and formulate a max-min problem:
\begin{equation}
    \max_{\lambda \in \Lambda}\,\min_{\mu \in \Delta_\Theta,\, \xi\in \cR}\,\cL_1(\xi; \lambda) \,+\, \cL_2(\mu; \lambda).
\label{eq:custom-opt-as-max-min-1}
\end{equation}
%
% Once again this can be seen as a zero-sum game between a $\theta$-player that minimizes $\cL_1$, a $\xi$-player that minimizes $\cL_2 + \cL_3$ and a $\lambda$-player that maximizes $\cL$. 
%
One can now apply Algorithm \ref{algo:lagrangian-ideal} and \ref{algo:lagrangian-surrogate} to this problem. For Algorithm \ref{algo:lagrangian-ideal}, the $\theta$-player uses the CSO oracle to optimize $\cL_2$, and for Algorithm \ref{algo:lagrangian-surrogate}, the $\theta$-player uses OGD to optimize a surrogate Lagrangian $\tL_2(\mu; \lambda) \,=\, \sum_{k=1}^K \lambda_{J+k}\, \tR_k(\mu)$. The best response strategy for the $\xi$-player is given by:
% outlines our first approach for finding an equilibrium of this game, where we assume access to the CSO oracle, and have the $\xi$- and $\theta$-players playing their best response, and the $\lambda$-player run a no-regret algorithm. The $\xi$-player's best response is computed as follows: based on the following result.
%%%%%
% \begin{lem}
% Let  $\phi_j^*: \R^K  \> \R$ denote the convex Fenchel conjugate of $\phi_j$. Define:
% $$
% H^*_{j}(\lambda) \,=\,
% \begin{cases}
%   \phi_j^*\left(
%   {\lambda_{j,1}^t}/{\lambda^t_{j,0}},\, \ldots,\, {\lambda_{j,K}^t}/{\lambda^t_{j,0}}\right) & \text{if}~\lambda^t_{j,0}>0\\
%     \0^K & \text{o/w}
% \end{cases}.
% $$
% Then
% $
% \textstyle
% \nabla H^*_{j}(\lambda) \,\in\,
% \amin{\xi \in [0,1]^K}\,\cL_1(\xi; \lambda).
% % \amin{\xi \in [0,1]^K}\,\left\{\psi(\xi_1, \ldots,(\xi; \lambda) \xi_K) \,-\, \sum_{k=1}^K \lambda_k \xi_k\right\}.
% $
% \todohari{Check if this holds when $\xi \in [0,1]$}
% \label{lem:xi-opt-closed-form-constrained}
% \end{lem}
%%%%%
\begin{lem}
Let  $\Phi(\xi) \,=\, \sum_{j=1}^J \lambda_j\, \phi^j(\xi)$ with $\lambda_j > 0$ for some $j \in [J]$. Let 
$\Phi^*: \R_+^K  \> \R_+$ denote the Fenchel conjugate of $\Phi$.
Then: %for any $\lambda$ s.t.\ $\sum_{k=1}^K |\lambda_{J+k}| \leq L$,
$$
%\textstyle
\nabla \Phi^*(
\lambda_{J+1}, 
\ldots,
\lambda_{J+K}
) \,\in\,
\argmin{\xi \in \cR}\,\cL_1(\xi; \lambda).
% \amin{\xi \in [0,1]^K}\,\left\{\psi(\xi_1, \ldots,(\xi; \lambda) \xi_K) \,-\, \sum_{k=1}^K \lambda_k \xi_k\right\}.
$$
%\todohari{Check if this holds when $\xi \in [0,1]$}
\label{lem:xi-opt-closed-form-constrained}
\vspace{-10pt}
\end{lem}

\begin{thm}
    \label{thm:convergence-ideal-constrained}
    % \todohari{Algorithm 1 for P2}
    Let $\theta^1, \ldots, \theta^T$ be the iterates generated by Algorithm \ref{algo:lagrangian-ideal} for \eqref{eq:non-linear-opt}, and let $\bar{\mu}$ be a stochastic model with a probability mass of $\frac{1}{T}$ on $\theta^t$. 
    % Suppose each $\phi^j$ is monotonically non-decreasing in each argument and $L$-Lipschitz w.r.t.\ the $\ell_\infty$-norm. 
    Suppose there exists a $\mu' \in \Delta_\Theta$ such that $\phi^j(\bR(\mu')) \leq -\gamma, \, \forall j \in [J]$, for some $\gamma > 0$. Let $B_g \,=\, \max_{\theta \in \Theta}\,g(\theta)$.
    Let $\mu^* \in \Delta_\Theta$ be such that $\mu^*$ is feasible, i.e.\ $\phi^j({\bR}(\mu^*)) \leq 0,\,\forall j \in [J]$, and $\E_{\theta \sim \mu^*}\left[g(\theta)\right] \leq \E_{\theta \sim \mu}\left[g(\theta)\right]$ for every $\mu \in \Delta_{\Theta}$ that is feasible. %satisfy $\phi^j(\tilde{\bR}(\mu)) \leq 0,\,\forall j \in [J]$.
     Let %$B_\theta \geq \max_{\theta, \bxi}\,\cL(\theta, \bxi; \blambda)$ and
    $B_\lambda \,\geq\, \max_{t}\|\nabla_{\blambda}\cL(\bxi^t,\theta^t; \blambda^t)\|_2$. Then setting $\kappa \,=\, \frac{2(L+1)B_g}{\gamma}$ and $\eta = \frac{\kappa}{ B_{\blambda}\sqrt{2T}}$, we have w.p.\ $\geq 1 - \delta$ over draws of stochastic gradients:
    \begin{equation*}
        \E_{\theta \sim \bar{\mu}}\left[g(\theta)\right] \,\leq\,
        \E_{\theta \sim \mu^*}\left[g(\theta)\right]
         \,+\, \cO\bigg(\sqrt{\frac{\log(1/\delta)}{T}}\,+\, \rho\bigg)
        %\end{equation*}
        ~~~~~\text{and}~~~~~
        %\begin{equation*}
        \phi^j(\bR(\bar{\mu}))
        \,\leq\,
        \cO\bigg(\sqrt{\frac{\log(1/\delta)}{T}}\,+\, \rho\bigg),~~~\forall j \in [J].
    \end{equation*}
    %\vspace{-5pt}
\end{thm}
%For Algorithm \ref{algo:lagrangian-surrogate}, we have the $\theta$-player optimize the proxy-Lagrangian function $\tL_2(\mu; \lambda) \,=\, \sum_{k=1}^K \lambda_{J+k}\, \tR_k(\theta)$ and obtain the following convergence result. %In this case, both the $\theta$ and $\lambda$ players run a no-regret algorithm, while the $\xi$-player plays best response.

See Appendix \ref{app:thm-convergence-ideal-constrained} for the proof. 
We have thus shown that Algorithm \ref{algo:lagrangian-ideal} outputs a stochastic model that has an objective close to the optimal feasible solution for \eqref{eq:non-linear-opt}, while also closely satisfying the constraints. 
We next present a convergence result for Algorithm \ref{algo:lagrangian-surrogate}.
\begin{thm}
    \label{thm:convergence-surrogate-constrained}
    Let $\theta^1, \ldots, \theta^T$ be the iterates of Algorithm \ref{algo:lagrangian-ideal} for \eqref{eq:non-linear-opt}, and let $\bar{\mu}$ be a stochastic model with probability $\frac{1}{T}$ on $\theta^t$.  %Suppose each $\phi^j$ is strictly convex, monotonically non-decreasing in each argument and $L$-Lispchitz w.r.t.\ $\ell_\infty$ norm, and $g$ is convex. 
    % Let $\lambda^*$ be a maximizer of $\min_{\xi, \mu}\cL(\xi, \mu; \lambda)$ over $\lambda \in \R^K$ and set $\kappa \geq 2\|\lambda^*\|_1$. 
    Let $\tilde{\Theta} \,=\, \big\{\theta \in \Theta \,|\, 
    \forall j, \,\tilde{\bR}(\theta) \in \dom\,\phi^j,~\forall j\big\}$. Let $\tilde{\theta}^* \in \tilde{\Theta}$ be such that it satisfies the surrogate-relaxed constraints
    $\phi^j(\tilde{\bR}(\tilde{\theta}^*)) \leq 0,\,\forall j \in [J]$,
    and $g(\tilde{\theta}^*) \leq g({\theta})$ for every $\theta \in \tilde{\Theta}$ that satisfies the same constraints. %satisfy $\phi^j(\tilde{\bR}(\theta)) \leq 0,\,\forall j \in [J]$. 
    % Let $B_g = \max_{\theta \in \Theta}\,g(\theta)$.
    Let $B_\Theta \geq \max_{\theta \in \Theta}\,\|\theta\|_2$, $B_\theta \geq \max_{t}\|\nabla_{\theta}\tL_2( \theta^t; \blambda^t)\|_2$ and
    $B_\lambda \,\geq\, \max_{t}\|\nabla_{\blambda}\cL(\bxi^t, \theta^t; \blambda^t)\|_2$. Then setting $\kappa \,=\, (L+1)T^{\omega}$ for $\omega \in (0, 0.5)$, $\eta_\theta = \frac{B_\Theta}{B_{\theta}\sqrt{2T}}$ and $\eta_\lambda = \frac{\kappa}{B_{\blambda}\sqrt{2T}}$, we have w.p. $\geq 1 - \delta$ over draws of stochastic gradients:
    \begin{equation*}
        \E_{\theta \sim \bar{\mu}}\left[g(\theta)\right] \,\leq\,
        g(\tilde{\theta}^*)
         \,+\, \cO\bigg(\frac{\sqrt{\log(1/\delta)}}{T^{1/2-\omega}}\bigg)
        %\end{equation*}
        ~~~~~\text{and}~~~~~
        %\begin{equation*}
        \phi^j(\bR(\bar{\mu}))
        \,\leq\,
        \cO\bigg(\frac{\sqrt{\log(1/\delta)}}{T^{\omega}}\bigg),~~~\forall j \in [J].
    \end{equation*}    
    \vspace{-10pt}
\end{thm}

See Appendix \ref{app:thm-convergence-surrogate-constrained} for the proof.
The proof is an adaptation of the equilibrium analysis in Agarwal et al.\ (2018) \cite{Agarwal+18} to non-zero-sum games. We point out that despite the $\theta$-player optimizing surrogate functions, %thanks to the upper-bounding property of the surrogates, 
the final stochastic classifier is near-feasible for the original rate metrics. We also note that this result holds even if the surrogates output values outside the domain of the constraint functions $\phi^j$'s  (e.g.\ with  KL-divergence based constraints). %, and for these cases, a convergence guarantee would not be possible had the $\lambda$-player performed updates with surrogate functions instead of the original rates. 

While the above convergence rate  is not as good as the standard $\cO(1/\sqrt{T})$ rate achievable for OGD, this is similar to the guarantees shown by e.g.\ Agarwal et al.\ \cite{Agarwal+18} for a linear rate-constrained optimization problem.\footnote{See e.g.\ Theorem 3 in their paper, where, for $T = \cO(n^{4\alpha})$ iterations, the error bound is $\tilde{\cO}(n^{-\alpha}) = \tilde{\cO}(T^{-1/4})$.} The reason for the poorer convergence rate is that we are unable to fix the radius $\kappa$ of the space of Lagrange multipliers $\Lambda$ to a constant, and instead set it to a function of $T$. The larger the radius $\kappa$, the tighter is the bound on the constraint violation, but worse is the bound on the objective $g$. The choice of $\omega$ strikes a trade-off between the tightness of the two bounds. 

In the case of the previous result for Algorithm \ref{algo:lagrangian-ideal} (see Theorem \ref{thm:convergence-ideal-constrained}), we were able to obtain a $\cO(1/\sqrt{T})$ rate by making an assumption that the constraints can be satisfied with a margin $\gamma$, and using this to set $\kappa$ to a constant that depends on $\gamma$. This does not work for Algorithm \ref{algo:lagrangian-surrogate} because of the mismatch in objectives optimized by the $\theta$- and $\lambda$-players.\footnote{Theorem \ref{thm:convergence-ideal-constrained} and \ref{thm:convergence-surrogate-constrained}  assume that the sequence of Lagrangian objectives  have bounded gradients. This may not hold in one particular scenario: if at an iteration $t$,  $\lambda^t_j = 0, \forall j \in [J]$ and $\lambda^t_k > 0$ for some $k \in [J+K]$, then the best response $\xi^t$ is unbounded, and this would result in the gradients of $\cL$ being unbounded. An easy fix to this problem is to define the set of allowable Lagrange multipliers $\Lambda$ to contain vectors for which the first $J$ coordinates are greater than or equal to a small constant. Our convergence results would still apply under an additional assumption that no unconstrained minimizer of $g$ satisfies all the constraints in \eqref{eq:non-linear-opt}.}

\begin{table}
    \centering
    \caption{Convergence guarantees for algorithms for \eqref{eq:non-linear-opt}. COCO \cite{Narasimhan18} is a previous algorithm and SPADE+ is an adaptation of SPADE \cite{Narasimhan+15b} to constrained problems. 
    %Alg.\ 1--3 are the proposed methods. Each player can do Best Response (BR), Online Gradient Descent (OGD) or Follow-The-Leader (FTL), and the game is zero-sum (ZS) or not. The first five algorithms find an approximate Nash or Coarse-Correlated (C.C.) equilibrium. Alg.\ 3 seeks to optimize a non-convex function of rates and is not guranteed to find an equilibrium.
    Algorithms apply either to a general convex model class $\Theta$ or to a \textit{restricted} subset for which the surrogates evaluate to values within the domain of $\phi^j$'s; require access to a cost-sensitive \textit{oracle} or not; output a \textit{stochastic} classifier $\bar{\mu}$ or a deterministic classifier $\bar{\theta}$. We denote $g(\bar{\mu}) = \E_{\theta \sim \bar{\mu}}\left[g(\theta)\right]$ and $\omega \in (0,1/2)$. COCO requires $\phi^j$'s to be smooth. The bounds are  on the optimality gap and on the maximum constraint violation.
    }
    \vspace{5pt}
    \label{tab:guarantees-constrained}
    \begin{tabular}{lccccccccc}  %|c|c|c|c|c|c|c|
        \toprule
         \textbf{Alg.} & 
%         \textbf{Smooth?} &
         \textbf{Restr.\ $\Theta$?} &
         \textbf{Oracle?} & 
         \textbf{Opt.\ Gap} &
         \textbf{Opt.\ Bound} &
         \textbf{Max.\ Viol.}
         \\
         \midrule
               COCO & 
 %              \cmark &
               \xmark &
               \cmark &
               $
               \displaystyle
                g(\bar{\mu})
                \,-\, \min_{\substack{\theta \in \Theta\\ \phi^j({\bR}(\theta)) \leq 0,\,\forall j}} g(\theta)
                $
                &
                $\tO\left(\sqrt{\frac{1}{T}}\right) \,+\, \rho$
                &
                $\tO\left(\sqrt{\frac{1}{T}}\right) \,+\, \rho$
         \\
               SPADE+ & 
  %             \xmark &
               \cmark &
               \xmark &
               $
               \displaystyle
                g(\bar{\theta})
                \,-\, \min_{\substack{\theta \in \Theta\\ \phi^j(\tilde{\bR}(\theta)) \leq 0,\,\forall j}} g(\theta)
                $
                &
                $
                \tO\Big(\sqrt{\frac{1}{T}}\Big)
                $
                &
                $
                \tO\Big(\sqrt{\frac{1}{T}}\Big)
                $
        \\
        \midrule
              Alg.\ \ref{algo:lagrangian-ideal} & 
   %            \xmark &
               \xmark &
               \cmark &
               $
               \displaystyle
                g(\bar{\mu})
                \,-\, \min_{\substack{\mu \in \Delta_\Theta\\ \phi^j({\bR}(\mu)) \leq 0,\,\forall j}} g(\mu)
                $
                &
                $\tO\Big(\sqrt{\frac{1}{T}}\,+\, \rho\Big) $
                &
                $\tO\Big(\sqrt{\frac{1}{T}}\,+\, \rho\Big) $
        \\
               Alg.\ \ref{algo:lagrangian-surrogate} & 
    %           \xmark &
               \xmark &
               \xmark &
                $
               \displaystyle
                g(\bar{\mu})
                \,-\, \min_{\substack{\theta \in \tilde{\Theta}\\ \phi^j(\tilde{\bR}(\theta)) \leq 0,\,\forall j}} g(\theta)
                $
                &
                $
                \tO\Big(\frac{1}{T^{1/2-\omega}}\Big)
                $
                &
                $
                \tO\Big(\frac{1}{T^{\omega}}\Big)
                $
        \\
         \bottomrule
    \end{tabular}
    %\vspace{-10pt}
\end{table}

\begin{rem}[\textbf{Unanswered Question in Cotter et al.}]
Theorem \ref{thm:convergence-surrogate} resolves an unanswered question in Cotter et al. (2019) \cite{Cotter+19,Cotter+19b} on whether having the $\lambda$- and $\theta$-players perform OGD updates on the true and surrogate Lagrangian respectively, would lead to convergence. 
Cotter et al.\ consider optimization problems with linear rate constraints, formulate a non-zero-game like us, and consider two algorithms, one where both the $\theta$- and
$\lambda$-player seek to minimize external regret (through OGD updates), and the other where the $\theta$-player optimizes alone minimizes external regret, while the $\lambda$-player minimizes swap regret. They are however able to show convergence guarantees only for the
the swap regret algorithm, which arguably has a more complicated set of updates (requiring computing eigen vectors), and leave the analysis of the external regret algorithm unanswered. 
Theorem \ref{thm:convergence-surrogate-constrained} provides convergence guarantees for a generalization of their external regret algorithm.\footnote{Note that our approach can be  easily applied to linear rate constraints by  formulating the Lagrangian in \eqref{eq:L-min-expand-constrained} without the additional slack variables.} 

In their paper, Cotter et al. obtain a better $\cO(1/\sqrt{T})$ convergence rate for their swap regret algorithm. It is easy to show a similar convergence rate for an adaptation of this algorithm to our setting (see Appendix \ref{app:cotteretal}), but we stick to our present algorithm because of its simplicity.
%\vspace{-8pt}
\end{rem}

\begin{rem}[\textbf{Relationship to previous algorithms for \eqref{eq:non-linear-opt}}]
A previous oracle-based approach for solving \eqref{eq:non-linear-opt} is COCO \cite{Narasimhan18}, which extends the Frank-Wolfe approach for the unconstrained rate metrics \cite{Narasimhan+15} with an outer gradient ascent solver, but does  not directly fit into our three-player framework. As seen in Table \ref{tab:guarantees-constrained}, this method enjoys similar convergence guarantees as Algorithm \ref{algo:lagrangian-ideal}, but requires an additional smoothness assumption on $\phi^j$'s and a more complicated set of updates. 

Another previous approach that can be adapted to solve \eqref{eq:non-linear-opt} is the SPADE algorithm, originally proposed for unconstrained problems  \cite{Narasimhan+15b} (see Table \ref{tab:choices} and \ref{tab:guarantees-unconstrained}), where the $\theta$- and $\lambda$-players perform OGD updates on surrogate objectives, and the $\xi$-player plays best response. We call this adaptation of SPADE as SPADE+ and present its convergence guarantee in Table \ref{tab:guarantees-constrained} (with further details in Appendix \ref{app:spade+}). As with existing surrogate-based methods, SPADE+ requires that the surrogate rates do not evaluate to values outside the domain of $\phi^j$'s. Algorithm \ref{algo:lagrangian-surrogate} does not have this restriction. %, and outperforms SPADE+ in all our experiments.
%  While this method does not directly fit into the proposed two-player framework, 
\end{rem}

\textbf{Shrinking.} 
The final stochastic classifier in all our theoretical results is defined by a uniform distribution over $T$ deterministic classifiers. The large support size may make this stochastic classifier undesirable in practice.
In our experiments, we post-process the iterates of our algorithms to construct a sparse stochastic classifier over only $J+1$ iterates (where recall $J$ is the number of constraints in \eqref{eq:non-linear-opt}). Specifically, we adopt the shrinking procedure of Cotter et al. (2019) \cite{Cotter+19b} and solve the following linear program over the $T$-dimensional simplex:
\begin{equation}
\min_{\mu \,\in\, \R_+^T,~ \sum_{t=1}^T \mu_t = 1}\, \sum_{t=1}^T \mu_t\,g(\theta^t)
~~\text{s.t.}~~
\sum_{t=1}^T \mu_t\,\phi^j(\theta^t) \,\leq\, 0,\,~\forall j \in [J].
\label{eq:shrinking}
\end{equation}
We then use the optimal weighting $\mu^*$ for this problem to construct the final stochastic classifier. Since \eqref{eq:shrinking} is a linear optimization over the simplex with $J$ linear constraints, the solution $\mu^*$ can be shown to have at most $J+1$ non-zero entries \cite{Cotter+19b}. 
Further, by convexity of $\phi^j$'s, this solution  is also feasible  for our original constrained problem \eqref{eq:non-linear-opt}, i.e. $\sum_{t=1}^T \mu^*_t\,\phi^j\left(\bR(\theta^t)\right) \,\leq\, 0 ~\Rightarrow~ \phi^j\left(\E_{\theta \sim \mu^*}[\bR(\theta)]\right) \,\leq\, 0, \forall j \in [J]$. Thus the final stochastic classifier constructed from \eqref{eq:shrinking} is both sparse and feasible, and as we shall see in our experiments, is often a better solution than a uniform distribution over all $T$ iterates.

%%%%%%%%%%%%%%%%%%%
% \begin{table}[]
%     \centering
%     \begin{tabular}{|c|c|c|c|c|c|c|c|}
%         \hline
%         \multicolumn{3}{|c|}{\textbf{Player Objective}} & \multicolumn{3}{|c|}{\textbf{Player Strategy}} &
%         \multirow{2}{*}{\textbf{Equilibrium}} &
%         \multirow{2}{*}{\textbf{Algorithm}}
%         \\[3pt]
%         \cline{1-6}
%         % &
%         % \multirow{2}{*}{\textbf{Convergence Rate}} \\
%          $\xi$ & $\theta$ & $\lambda$ &
%          $\xi$ & $\theta$ & $(\lambda_0, \lambda_{1:})$ &
%          &
%          \\[3pt]
%          \hline
%          &&&&&&&\\[-6pt]
%          $\cL_1$ & 
%          $\cL_2$ &
%          $-(\cL_1 + \cL_2)$ &
%          BR &
%          BR &
%          (OGD, FTL) &
%          Nash &
%          COCO \cite{xx}
%          \\[2pt]
%          $\cL_1$ & 
%          $\cL_2$ &
%          $-(\cL_1 + \cL_2)$ &
%          BR &
%          OGD &
%          OGD &
%          Nash &
%          Algorithm \ref{algo:lagrangian-ideal-1}
%          \\[2pt]
%          $\cL_1$ & 
%          $\tL_2$ &
%          $-(\cL_1 + \cL_2)$ &
%          BR &
%          OGD &
%          OGD &
%          Coarse-correlated &
%          Algorithm \ref{algo:lagrangian-surrogate-1}
%          \\[2pt]
%          \hline
%     \end{tabular}
%     \vspace{5pt}
%     \caption{Choice of player objectives and strategies for \eqref{eq:non-linear-opt}.% We consider three player strategies: Best Response (BR), Online Gradient Descent (OGD) and Follow-The-Leader (FTL).
%     }
%     \label{tab:my_label}
% \end{table}
%%%%%%%%%%%%%%

\subsection{Case of Sum-of-ratios $\phi$}
\label{sec:sum-ratios-constraint}
Moving beyond convex functions of rates, we present heuristic surrogate-based algorithms for handling constraints that involve summation of fractional-linear functions of rates. % that we show works well in practice. 
For simplicity, we explain our approach with a single constraint and with a general objective:
\begin{equation}
    \min_{\theta \in \Theta}\, g(\theta)
~~~\text{s.t.}~~~
\sum_{m=1}^M\frac{\balpha_m^\top \bR(\theta)}{\bbeta_m^\top \bR(\theta)} \,\leq\, \gamma,
\tag{P3}
\label{eq:sum-of-ratios-opt}
\end{equation}
where %$\bR(\theta) \,=\, [R_1(\theta), \ldots, R_K(\theta)]^\top$, 
$\balpha_m \in \R_+^K$ and $\bbeta_m \in \R_+^K$. We will assume that the numerators and denominators in  each ratio term is bounded, i.e.\ $\underline{a} \leq \alpha_m^\top \bR(\theta) \leq \bbeta_m^\top \bR(\theta) \leq \overline{b}, ~\forall \theta \in \Theta$ for some $\underline{a}, \overline{b} > 0$. We use our framework to provide two algorithms.

\begin{figure}
\begin{minipage}[H]{0.51\textwidth}
\begin{algorithm}[H]
\caption{Slack-ratios Optimizer}
\label{algo:slack-ratios}
\begin{algorithmic}
\STATE Initialize: $\ba^0$, $\bb^0$, $\theta^0$, $\blambda^0$
\FOR{$t = 0 $ to $T-1$}
\STATE $\ba^{t+1} := \Pi_{\cC}\left(\ba^t - \eta_\ba\nabla_\ba\,\cL_{sr}(\theta^t, \ba^t, \bb^t; \blambda^t)\right)$
\STATE $\bb^{t+1} := \Pi_{\cC}\left(\bb^t - \eta_\bb\nabla_\bb\,\cL_{sr}(\theta^t, \ba^t, \bb^t; \blambda^t)\right)$
\STATE $\theta^{t+1} := \Pi_{\Theta}(\theta^t - \eta_\theta\,\nabla_\theta\,\tL_{sr}(\theta^t; \ba^t, \bb^t\, \blambda^t))$
%\STATE $\bxi^{t+1} \leftarrow \bxi^t \,-\, \eta_{\bxi}\,\Pi_{\bxi}(\nabla_\theta\,\tL(\theta^t, \bxi^t;\, \blambda^t))$
\STATE $\blambda^{t+1} := \Pi_{\Lambda}\left(\blambda^t + \eta_{\blambda}\,\nabla_{\blambda}\,\cL_{sr}(\theta^t, \ba^t, \bb^t;\, \blambda^t)\right)$
\ENDFOR
\RETURN $\theta^1, \ldots, \theta^T$
\end{algorithmic}
\end{algorithm}
\end{minipage}
\begin{minipage}[H]{0.5\textwidth}
\begin{algorithm}[H]
\caption{Biconvex Optimizer}
\label{algo:biconvex}
\begin{algorithmic}
%\STATE Given: $0 < a \leq \alpha_\ell^\top \bR(\theta) \leq \bbeta_\ell^\top \bR(\theta) \leq b < 1$
\STATE Initialize: $\theta^0, \bu^0, \blambda^0$
\FOR{$t = 0 $ to $T-1$}
\STATE $\xi^{t}_{m} :=
\left(u^t_m\,\lambda_0^t\,/\,\lambda_m^t\right)^2,~~\forall m \in [M]
$
\STATE $\theta^{t+1} := \Pi_{\Theta}(\theta^t \,-\, \eta_\theta\,
\nabla_\theta\,\tilde{\cL}_{bc}(\theta^t, \xi^{t}, \bu^t; \lambda^t))$
% where $\tL_{sr}$ is defined with surrogates $\tilde{R}_k$.
% where $\tL_{1}(\theta;\, \blambda; u)) = g(\theta) + \lambda_0\sum_{m=1}^M u_m^2\balpha_m^\top\tilde{\bR}(\theta) - \sum_{m=1}^M\lambda_m \bbeta_m^\top\tilde{\bR}(\theta)$
\STATE $u^{t+1} := \Pi_{\cU}(\bu^t - \eta_{\bu}\,\nabla_{\bu}\,\cL_{bc}(\theta^t, \xi^{t}, \bu^t; \lambda^t))$ % 
% where $\cU = [0,\sqrt{a}/b]^M$.
%where $\cL_{2}(\theta, u, \xi;\, \blambda)) = \lambda_0\left(\sum_{m=1}^M u_m\sqrt{\xi_m} \,+\, \sum_{m=1}^M \lambda_m \xi_m\right)$
\STATE $\blambda^{t+1} := \Pi_{\Lambda}\big(\blambda^t\,+\, \eta_\lambda\,\nabla_{\blambda}\,\cL_{bc}(\theta^t, \xi^{t}, \bu^t; \lambda^t)$
% where $\Lambda \subseteq \R_+^{M+1}$ is bounded.
%where $\cL_{3}(\theta, u;\, \blambda) \,=\, $
\ENDFOR
\RETURN $\theta^1, \ldots, \theta^T$
\end{algorithmic}
\end{algorithm}
\end{minipage}
\caption{Optimizers for a sum-of-ratios constraint \eqref{eq:sum-of-ratios-opt}. % with $0 < \underline{a} \leq \alpha_m^\top \bR(\theta) \leq \bbeta_m^\top \bR(\theta) \leq \overline{b}, ~\forall \theta \in \Theta$. 
In Algorithm \ref{algo:slack-ratios}, $\tL_{sr}(\theta, \ba, \bb; \lambda)$ is the Lagrangian function in \eqref{eq:lagrangian-sr} with the rates $R_k(\theta)$ replaced with surrogates $\tilde{R}_k(\theta)$, $\cC = [\underline{a}, \overline{b}]^M$ and $\Lambda \subseteq \R_+^{2M+1}$ is a bounded set. In Algorithm \ref{algo:biconvex}, $\tL_{bc}(\theta, \xi, \bu; \lambda)$ is the Lagrangian function in \eqref{eq:lagrangian-bc} with the rates $R_k(\theta)$ replaced with surrogates $\tilde{R}_k(\theta)$, $\cU = [0,\sqrt{\underline{a}}/\overline{b}]^M$, and $\Lambda \subseteq \R_+^{M+1}$ is a bounded set.}
%\vspace{-5pt}
\end{figure}

\paragraph{Slack-ratios optimizer.}
Our first approach introduces slack variables $a_1, \ldots, a_m$ and $b_1, \ldots, b_m$ for the numerators and denominators respectively to decouple the rates from the ratio terms:
\begin{equation}
    \min_{\theta \in \Theta,~ \ba,\bb \in [\underline{a}, \overline{b}]^M}\, g(\theta)
~~~\text{s.t.}~~~
\sum_{m=1}^M\frac{a_m}{b_m} \,\leq\, \gamma,~~
a_m \,\geq\, \balpha_m^\top \bR(\theta), ~~
b_m \,\leq\, \bbeta_m^\top \bR(\theta).
\end{equation}
We then formulate the Lagrangian for the above problem with multipliers $\lambda \in \R_+^{2M+1}$:
\begin{eqnarray}
\lefteqn{
\cL_{sr}(\theta, \ba, \bb; \lambda)
}
\label{eq:lagrangian-sr}
\\
&=&
g(\theta) \,+\,
\lambda_0
\left(
\sum_{m=1}^M\frac{a_m}{b_m} \,-\, \gamma
\right)
\,+\,
\sum_{m=1}^M \lambda_m (\balpha_m^\top \bR(\theta) \,-\, a_m)
\,+\,
\sum_{m=M+1}^{2M} \lambda_{m}(b_m \,-\, \bbeta_m^\top \bR(\theta)),
\nonumber
\end{eqnarray}
and seek to maximize the Lagrangian over $\lambda$ and minimize it over $\theta, \ba, \bb$:
\begin{equation}
\max_{\lambda \in \R_+^{2M+1}}\,\min_{\theta \in \Theta,~ \ba,\bb \in [\underline{a}, \overline{b}]^M}\,
\cL_{sr}(\theta, \ba, \bb; \lambda).
\label{eq:lagrangian-slack-ratios}
\end{equation}
Because $\cL_{sr}$ is non-convex in the slack variables $\ba, \bb$, strong duality may not hold, and an optimal solution for the dual problem \eqref{eq:lagrangian-slack-ratios} may not be optimal for \eqref{eq:sum-of-ratios-opt}. Yet, by performing OGD updates for the $\lambda$, $\theta$ and the slack variables, with the $\theta$-player alone optimizing the surrogate rates $\tilde{R}_k$, we obtain a heuristic algorithm for \eqref{eq:sum-of-ratios-opt}. The details are given in Algorithm \ref{algo:slack-ratios}.

\paragraph{Biconvex optimizer.}
Our second approach poses \eqref{eq:sum-of-ratios-opt} as a biconvex optimization problem. 
We adapt a trick from Pan et al.\ (2016) \cite{Weiwei+16} to replace the constraint with equivalent  expressions that do not contain ratios. It can be verified that for any $0 < a \leq b < 1$, the ratio $\frac{a}{b} \,=\, \min_{u \in \R}\,\{u^2b - 2u\sqrt{b - a} + 1\}$.  Define $\varphi(z, z'; u) = u^2z - 2u\sqrt{z'} + 1$,  which for a fixed $u$, is convex in $z$ and $z'$, and vice versa. %, but not jointly convex in all variables.
Using auxiliary variables $\xi,\xi' \,\in\, [0,1]^M$, we can now re-write \eqref{eq:sum-of-ratios-opt} as:
\begin{equation*}
    \min_{\theta \in \Theta, ~\xi \,\in\, \R_+^M}\, g(\theta)
~~~\text{s.t.}~~~
\min_{u \in \R^M}\,
\sum_{m=1}^M
\varphi(\balpha_m^\top \bR(\theta), \, \xi_m;\, u_m)
\,\leq\, 
\gamma,
~~~~
\xi_m
\,\leq\, 
\bbeta_m^\top \bR(\theta) \,-\, \balpha_m^\top \bR(\theta),
%\label{eq:lagrangian-sum-of-ratios-opt}
\end{equation*}
% While the constraint is non-convex,
%Note that each term $\frac{\xi_m}{\xi'_m}$ in the constraints is individually linear in $\xi_m$ and convex in $\xi'_m$, but not jointly convex in both parameters.
%Following Narasimhan et al. (2016), we introduce auxiliary variables $u$ and re-write this constraint so that it is separately convex in the rates and separately convex in $u$. 
% the biconvex structure of $\varphi$ allows us to adapt Algorithm \ref{algo:lagrangian-surrogate} 
% (but without its theoretical guarantees). 
and formulate the Lagrangian for this problem with multipliers $\lambda \in \R_+^{M+1}$:
%Introducing Lagrange multipliers for the constraints, we define the Lagrangian for this problem as:
\begin{eqnarray}
\lefteqn{
\cL(\theta, \xi; \lambda)
}
\label{eq:lagrangian-bc}
\\
&=&
g(\theta) \,+\,
\lambda_0
\left(
\min_{\bu \in \R^M}\,
\sum_{m=1}^M
\varphi(\balpha_m^\top \bR(\theta), \, \xi_m;\, u_m)
\,-\, \gamma
\right)
% \left(
% u_\ell^2\,\bbeta_\ell^\top \bR(\theta)\,-\, 2u_\ell\,\sqrt{\xi_\ell} \,+\, 1 \,-\, \gamma
% \right)
\,+\,
\sum_{m=1}^M \lambda_m\,
\big(
\xi_m \,-\, {\bbeta_m^\top\bR(\theta) \,+\,
\balpha_m^\top\bR(\theta)}
\big)
\nonumber
\\
&=&
\min_{\bu \in \R^M}
\underbrace{
g(\theta) \,+\,
\lambda_0
\left(
\sum_{m=1}^M
\varphi(\balpha_m^\top \bR(\theta), \, \xi_m;\, u_m)
\,-\, \gamma
\right)
% \left(
% u_\ell^2\,\bbeta_\ell^\top \bR(\theta)\,-\, 2u_\ell\,\sqrt{\xi_\ell} \,+\, 1 \,-\, \gamma
% \right)
\,+\,
\sum_{m=1}^M \lambda_m\,
\big(
\xi_m \,-\, {\bbeta_m^\top\bR(\theta) \,+\,
\balpha_m^\top\bR(\theta)}
\big)
}_{\cL_{bc}(\theta, \xi, \bu; \lambda)}.
\nonumber
\end{eqnarray} 
We seek to maximize $\cL_{bc}$ over  $\lambda$ and minimize it over $\theta, \xi, \bu$:
\[
\max_{\lambda \in \R_+^K}\, \min_{\substack{\theta \in \Theta\\\xi, \bu \,\in\, \R^M}}\,\cL_{bc}(\theta, \xi, \bu; \lambda).
\]
The biconvex structure of $\varphi$ allows us to adapt Algorithm \ref{algo:lagrangian-surrogate} to this problem, giving us Algorithm \ref{algo:biconvex}, where we perform best response on slack variables $\xi$, and OGD updates on $\lambda$, $\theta$ and $\bu$, with the $\theta$-player alone optimizing surrogate rates $\tilde{R}_k$.
% The resulting algorithm can be seen an generalization of the approach of \cite{Weiwei+16} for unconstrained optimization of the multiclass F-measure.

%\todohari{Mention somewhere that the monotonically increasing assumption on $\psi$ and $\phi$ is very reasonable as one can always negate rate to make these functions increasing.}

\section{Experiments}
\label{sec:expts}
We conduct two types of experiments. In the first, we evaluate Algorithm \ref{algo:lagrangian-surrogate} on the task of optimizing a convex rate objective subject to linear rate constraints. We show that our approach outperforms an existing surrogate-based approach \cite{Narasimhan+15b}, and performs often as well as an oracle-based approach for this problem \cite{Narasimhan18} (and does so without having to make an idealized oracle assumption). In the second set of experiments, we evaluate Algorithms \ref{algo:slack-ratios} and \ref{algo:biconvex} on the task of optimizing a sum-of-ratios objective subject to sum-of-ratios constraints, and show that they outperform known baselines.

\begin{table}
    \centering
    \caption{Datasets used in our experiments.}
    \label{tab:datasets}
    \vspace{5pt}
    \begin{tabular}{cccc}
        \toprule
        Dataset & No. of instances & No. of features & Protected Attribute\\
        \toprule
        COMPAS & 4073 & 31 & Gender
        \\
        Communities \& Crime & 1495 & 135 & Race Percentage
        \\
        Law School & 15388 & 36 & Race
        \\
        Adult & 32561 & 122 & Gender
        \\
        WikiToxicity & 95692 & 100 & Term `Gay'
        \\
        Business & 11560 & 36 & Chain/Non-chain\\
        \bottomrule
    \end{tabular}
\end{table}

\subsection{Datasets and Setup}
\textbf{Datsets.} We use six datasets: (1) \textit{COMPAS}, where the goal is to predict recidivism with \textit{gender} as the protected attribute \cite{Angwin+16}; (2) \textit{Communities \& Crime}, where the goal is to predict if a community in the US has a crime rate above the $70$th percentile \cite{uci}, and as in \cite{Kearns+18}, we consider communities having a black population above the $50$th percentile as protected; 
(3) \textit{Law School}, where the task is to predict whether a law school student will pass the bar exam, with \textit{race} (black or other) as the protected attribute \cite{Wightman:1998}; (4) \textit{Adult}, where the task is to predict if a person's income exceeds 50K/year, with \textit{gender} as protected attribute \cite{uci}; (5) \textit{Wiki Toxicity}, where the goal is to predict if a comment posted on a Wikipedia talk page contains non-toxic/acceptable content, with the comments containing the term `\textit{gay}' considered as a protected group \cite{Dixon:2018}; (6) \textit{Business Entity Resolution}, a proprietary dataset from a large internet services company, where the task is to predict whether a pair of business descriptions refer to the same real business, with \textit{non-chain} businesses treated as protected. A summary of the datasets is provided in Table \ref{tab:datasets}. Wiki Toxicity is a text dataset, and we use an embedding \cite{pennington2014glove} to convert the text to numerical features. 

% See Appendix \ref{app:expts} for additional details on the setup and datasets. 
\textbf{Implementation Details.} All experiments use a linear model. 
We implemented Algorithms \ref{algo:lagrangian-surrogate}--\ref{algo:biconvex} using the open-source Tensorflow Constrained Optimization (TFCO) library\footnote{\url{https://github.com/google-research/tensorflow_constrained_optimization/}} of Cotter et al.\ (2019) \cite{Cotter+19,Cotter+19b}. We use hinge-based surrogates for the rates. We use Adam to perform exact gradient updates on $\theta$ and $\lambda$ and run our algorithms for a total of 5000 iterations. 
The datasets are split randomly into train-validation-test sets in the ratio 4/9:2/9:1/3, except WikiToxicity where we use the splits made available by the authors \cite{Dixon:2018}. 
%The  step-sizes for the $\theta$ and $\lambda$ updates are chosen from the range $\{0.001, 0.01, 0.1, 1.0\}$ using the validation set. 
Our code will be made available online.

\textbf{Hyperparameter Choices.} % The proposed algorithms output a stochastic classifier specified by a uniform distribution over all iterates. In our experiments,
The step-sizes for the $\lambda$- and $\theta$-updates in the proposed algorithms were chosen from the range $\{0.001, 0.01, 0.1, 1.0\}$ using the validation set. We use a heuristic provided in Cotter et al.\ (2019) \cite{Cotter+19b} to pick the hyperparameters that best trade-off between the fairness objective and constraint violations. 
 We record snapshots of the iterates of our algorithms every 10 iterations and construct both a stochastic classifier and a deterministic classifier from the iterates. For the stochastic classifier, we apply the shrinking procedure described in Section \ref{sec:convex-constraint} to the recorded iterates (see \eqref{eq:shrinking}). %For a problem with $J$ constraints, the shrinking procedure finds a distribution over only $J+1$ iterates that minimizes the objective subject to satisfying the constraints. In practice, the post-processed sparse solution is often better than a uniform distribution over all $T$  iterates. 
For the deterministic classifier, we pick the \textit{single best} iterate using the trade-off heuristic provided by Cotter et al.

\begin{table}[t]
\caption{Optimizing KL-divergence fairness metric s.t.\ error rate constraints. For each method, we report two metrics: A\,(B), where A is the test fairness metric  (\textit{lower} is better) and B is the ratio of the test error rate of the method and that of a classifier that optimizes unconstrained error rate (\textit{lower} is better).
During training, we constrain B to be $\leq 1.1$.
Among the last 4 columns, the lowest fairness metric is highlighted in blue, and the second-lowest is shown in light blue.}
\vspace{5pt}
\label{tab:kld}
\centering
% \begin{tabular}{lcccccc}
% 	\toprule
% 	& & & & & \multicolumn{2}{c}{Algorithm \ref{algo:lagrangian-surrogate}} \\\cmidrule(r){6-7}
% & UncError & PostShift & SPADE+ & COCO & Stochastic & Determ.\ \\\cmidrule(r){2-7}
% COMPAS	 & 0.115 (1.00) 	 & 0.000 (1.01) 	 & \textit{0.009} (1.03) 	 & {0.010} (1.03) 	 & \textbf{0.000} (1.03) 	 & \textbf{0.000} (1.03) \\
% Crime	 & 0.224 (1.00) 	 & 0.005 (1.40) 	 & 0.158 (1.11) 	 & 0.202 (0.86) 	 & {\textit{0.094}} (1.11) 	 & \textbf{0.085} (1.16) \\
% Law 	 & 0.199 (1.00) 	 & 0.001 (1.45) 	 & \textit{0.040} (1.09)	 & \textbf{0.032} (1.04) 	 & {{0.054}} (1.12) 	 & 0.056 (1.08) \\
% Adult	 & 0.114 (1.00) 	 & 0.000 (1.22) 	 &  0.071 (1.03) 	 & \textbf{0.011} (1.10) 	 & {\textit{0.014}} (1.10) 	 & {\textit{0.014}} (1.10) \\
% Wiki	 & 0.175 (1.00) 	 & 0.001 (1.21) 	 & \textbf{0.083} (1.18)  	 & 0.134 (1.17) 	 & 0.133 (1.09) 	 & \textit{0.127} (1.18) \\
% Business	 & 0.014 (1.00) 	 & 0.003 (1.05) 	 & \textit{0.007} (1.03) 	 & \textbf{0.002} (1.03) 	 & \textit{0.007} (1.07) 	 & \textit{0.007} (1.07) \\
% 	\bottomrule
% \end{tabular}
%\vspace{-5pt}
\begin{tabular}{lcccccc}
	\toprule
	& & & & & \multicolumn{2}{c}{Algorithm \ref{algo:lagrangian-surrogate}} \\\cmidrule(r){6-7}
& UncError & PostShift & SPADE+ & COCO & Stochastic & Determ.\ \\\cmidrule(r){2-7}
COMPAS	 & 0.115 (1.00) 	 & 0.000 (1.01) 	 & \sbest{0.009} (1.03) 	 & 0.043 (1.01) 	 & \best{0.000} (1.03) 	 & \best{0.000} (1.03) \\
Crime	 & 0.224 (1.00) 	 & 0.005 (1.40) 	 & 0.158 (1.11) 	 & 0.166 (0.90)	 & {\sbest{0.094}} (1.11) 	 & \best{0.085} (1.16) \\
Law 	 & 0.199 (1.00) 	 & 0.001 (1.45) 	 & \best{0.040} (1.09)	 & \sbest{0.043} (1.05) 	 & {{0.054}} (1.12) 	 & 0.056 (1.08) \\
Adult	 & 0.114 (1.00) 	 & 0.000 (1.22) 	 &  0.071 (1.03) 	 & \best{0.011} (1.10) 	 & {\sbest{0.014}} (1.10) 	 & {\sbest{0.014}} (1.10) \\
Wiki	 & 0.175 (1.00) 	 & 0.001 (1.21) 	 & \best{0.083} (1.18)  	 & 0.134 (1.17) 	 & 0.133 (1.09) 	 & \sbest{0.127} (1.18) \\
Business	 & 0.014 (1.00) 	 & 0.003 (1.05) 	 & \sbest{0.007} (1.03) 	 & \best{0.004} (1.04)  	 & \sbest{0.007} (1.07) 	 & \sbest{0.007} (1.07) \\
	\bottomrule
\end{tabular}
\end{table}
%%%
\begin{table}[t]
\caption{ %Optimizing KL-divergence fairness metric s.t.\ error rate constraints. 
Same  as Table \ref{tab:kld}, except we compare SPADE+, COCO and Algorithm \ref{algo:lagrangian-surrogate} \textit{without} the post-processing shrinking procedure  in \eqref{eq:shrinking} to construct the final stochastic classifier. SPADE+ outputs $\frac{1}{T}
\sum_{t=1}^T \theta^t$. COCO uses the weighting scheme in \cite{Narasimhan18}. Algorithm \ref{algo:lagrangian-surrogate} outputs $Unif(\theta^1, \ldots, \theta^T)$.
% For each method, we report the test fairness metric  (\textit{lower} is better) and the ratio of the test error rate of the method and that of the  UncError classifier is shown within parenthesis (this has to be $\leq 1.1$). Among the last 4 columns, the lowest fairness metric is highlighted in bold, and the second-lowest is shown in italic.
}
\vspace{5pt}
\label{tab:kld-spade+}
\centering
\begin{tabular}{lccccccc}
	\toprule
& SPADE+ & COCO & Algorithm \ref{algo:lagrangian-surrogate} \\\cmidrule(r){2-4}
	COMPAS	 & 0.000 (1.02) 	 & 0.000 (1.07) 	 & 0.016 (1.02) \\
Crime	 & 0.157 (1.00) 	 & 0.043 (1.25) 	 & 0.116 (1.09) \\
Law 	 & 0.296 (1.01) 	 & 0.027 (1.13) 	 & 0.074 (1.11) \\
Adult	 & 0.098 (1.01) 	 & 0.003 (1.31) 	 & 0.068 (1.08) \\
Wiki	 & 0.106 (1.23) 	 & 0.010 (1.39) 	 & 0.098 (1.30) \\
Business	 & 0.006 (1.03) 	 & 0.001 (1.14) 	 & 0.005 (1.08) \\
	\bottomrule
\end{tabular}
%\vspace{-5pt}
\end{table}
%%%

\subsection{KL-divergence Based Fairness Objective} 
We consider a demographic-parity fairness objective that seeks to match the proportion of positives predicted in each group $\hat{p}_G$ with the true proportion of positives in the data $p$, measured using a KL-divergence metric: $\sum_{G \in \{0,1\}} \text{KLD}(p, \hat{p}_G)$. Note that this is convex in $\hat{p}_G$. We additionally enforce a constraint that the error rate of the model is no more than 10\% higher than an unconstrained model that optimizes error rate (UncError), i.e. $\hat{err}(f_\theta) \leq 1.1\, \hat{err}(f_{unc})$. %We set $\tilde{f}$ to the unconstrained classifier that optimizes the error rate (UncError). % Note that the objective here is a convex function of rates and the constraint is linear in the error rate.  

We apply Algorithm \ref{algo:lagrangian-surrogate} (adapted to include a convex rate metric as the objective) to solve this problem. We compare it with the
surrogate-based method SPADE+, an adaptation of the SPADE method \cite{Narasimhan+15b}  to constrained problems, and the oracle-based method COCO \cite{Narasimhan18}. COCO uses a logistic regression based \textit{plug-in} method to implement an approximate cost-sensitive oracle. For a fair comparison, we apply the post-processing shrinking step to these baselines as well, i.e.\ we apply \eqref{eq:shrinking} to the iterates of SPADE+ and COCO, and construct stochastic classifiers. We also include the post-shift method of Hardt et al. (2016) \cite{Hardt+16} that seeks to minimize the error rate subject to zero fairness violations. Specifically, it takes a pre-trained class probability estimation model (in our case, logistic regression) and assigns thresholds to  the groups to correct for fairness disparity \cite{Hardt+16}. 

We implement UncError by optimizing the hinge loss. Both UncError and the logistic regression for COCO and PostShift are trained with 2500 iterations of Adam. COCO is run for 500 outer iterations and 10 inner iterations. The step-size parameter for these three methods is chosen from $\{0.005, 0.01, 0.05, \ldots, 10.0\}$.
SPADE+ uses Adam for both the $\theta$ and $\lambda$ updates, is run for 5000 iterations, with the two step-sizes  chosen from $\{0.001, 0.01, 0.1, 1.0, 10.0\}$. 

We evaluate all methods based on their (a) KLD fairness metric on the test set, and (b) their constraint violation on the test set, measured by the ratio of error rate of the learned model and that of the unconstrained model: $\hat{err}(f_\theta)  /  \hat{err}(f_{unc})$. The results are shown in Table \ref{tab:kld}.

% For our algorithm, we report the the trained stochastic classifier and the heuristically chosen best \textit{deterministic} classifier.
% 

 PostShift performs the best on the fairness metric but often fairs poorly on the constraints.
 SPADE+ yields significantly poorer fairness values on the Crime and Adult datasets and suffers high constraint violation on the Wiki dataset. In contrast, the stochastic classifier trained by Algorithm \ref{algo:lagrangian-surrogate} closely satisfies the error rate constraint on almost all datasets. Also on five datasets, the proposed algorithm achieves the best or second-best fairness metric, doing significantly better than SPADE+ and COCO  on the Crime dataset. COCO yields the best fairness metric on two datasets. 
 
 It is also worth noting that the best deterministic classifier from Algorithm \ref{algo:lagrangian-surrogate} often has a similar performance to the stochastic classifier that it provides. Only on two of the six datasets, Crime and Wiki, the deterministic classifier suffers higher constraint violations than the stochastic classifier.
 
 All three constrained optimization algorithms benefit from using the shrinking procedure to post-process their iterates and construct the final stochastic classifier. To better analyze their convergence behavior, we also report in Table \ref{tab:kld-spade+}, their performance without the shrinking procedure. Specifically, for each method, we construct the final classifier as prescribed by its convergence gurantee: for SPADE+, this a deterministic classifier given by the average model parameters: $\sum_{t=1}^T \theta^t$ (see Theorem \ref{thm:convergence-spade+}); for COCO, this is a stochastic classifier specified by the weighting scheme provided in Narasimhan (2019) \cite{Narasimhan18}; for Algorithm \ref{algo:lagrangian-surrogate}, this is a stochastic classifier specified by a uniform distribution over all iterates (see Theorem \ref{thm:convergence-surrogate-constrained}). 
 
 Without the post-processing step, SPADE+ tends to overconstrain the model, and yields very poor fairness objective on three of the six datasets. This is because of its heavy dependence on surrogates. In contrast, even without the shrinking step, Algorithm \ref{algo:lagrangian-surrogate} closely satisfies the constraint on almost all datasets, while suffering a modest increase in fairness metric on four datasets. COCO suffers a higher constraint violation on all datasets in the absence of the shrinking step.
 
\begin{table}[t]
\centering
\caption{Optimizing F-measure s.t.\ F-measure constraints. For each method, we report two metrics: A\,(B), where A is the overall test F-measure  (\textit{higher} is better) and B is the test constraint violation: $\text{Fmeasure}_{ptr} - \text{Fmeasure}_{other} - 0.01$ (\textit{lower} is better). Among the last 4 columns, the highest F-measure objective is highlighted in blue and the lowest constraint violation is shown in bold.
}
\vspace{5pt}
\label{tab:fmeasure}
\begin{tabular}{lcccccc}
	\toprule
	& & &
	\multicolumn{2}{c}{Algorithm \ref{algo:slack-ratios}} 
	&
	\multicolumn{2}{c}{Algorithm \ref{algo:biconvex}}
	\\\cmidrule(r){4-7}
	& UncError & UncF1 & Stochastic & Determ. & Stochastic & Determ. \\\cmidrule(r){2-7}
	COMPAS	 & 0.652 (0.14) 	 & 0.671 (0.11) 	 & 0.618 (0.04) 	 & 0.615 (0.07) 	 & \best{0.629} (0.06) 	 & 0.622 (\textbf{0.03}) \\
Crime	 & 0.742 (0.07) 	 & 0.761 (0.07) 	 & \best{0.735} (\textbf{0.06}) 	 & 0.720 (0.09) 	 & 0.692 (0.12) 	 & 0.671 (0.13) \\
Law 	 & 0.974 (0.08) 	 & 0.975 (0.09) 	 & 0.927 (0.04) 	 & 0.973 (0.09) 	 & 0.911 (\textbf{0.01}) 	 & \best{0.974} (0.08) \\
Adult	 & 0.671 (0.05) 	 & 0.687 (0.05) 	 & 0.655 (0.04) 	 & \best{0.678} (0.07) 	 & 0.592 (\textbf{0.02}) 	 & 0.591 (\textbf{0.02}) \\
Wiki	 & 0.968 (0.19) 	 & 0.969 (0.17) 	 & 0.838 (0.05) 	 & \best{0.892} (0.10) 	 & 0.829 (\textbf{0.02}) 	 & 0.805 (0.04) \\
Business	 & 0.778 (0.04) 	 & 0.783 (0.03) 	 & \best{0.764} (-0.04) 	 & \best{0.764} (-0.04) 	 & 0.762 (-0.04) 	 & 0.754 (\textbf{-0.06}) \\
	\bottomrule
\end{tabular}
\end{table}

\subsection{F-measure Based Parity Constraints} 
We consider the fairness goal of training a classifier that yields at least as high a F-measure for the protected group as it does for the rest of the population, and impose this as a constraint. Specifically, we seek to maximize the overall F-measure subject to the constraint: $\text{Fmeasure}_{ptr} \,\geq\, \text{Fmeasure}_{other} - 0.01$. We apply Algorithms \ref{algo:slack-ratios} and \ref{algo:biconvex} (modified to include the F-measure as the objective) and compare it with an unconstrained classifier that optimizes error rate (UncError).\footnote{To implement Algorithm \ref{algo:biconvex} using the helper functions in the TF constrained optimization library, we perform a simple change of variables: we replace $\lambda_m = \lambda_0\, \lambda'_m,$ and optimize over $\lambda'_m$'s instead of $\lambda_m$'s, $m = 1, \ldots, M$.}
We also include a plug-in classifier, which is a popular approach for optimizing the F-measures without constraints (UncF1) \cite{Koyejo+14, Narasimhan+14, Yan+18}. We implement the plug-in approach by tuning the threshold on the UncError classifier to yield maximum F-measure on the training set.  We are not aware of any other approach that can handle objective and constraints that are both sums of ratios of rates.

The results are shown in Table \ref{tab:fmeasure}, where we present both the F-measure objective on the test set achieved by different methods, and their constraint violations measured by: $\text{Fmeasure}_{ptr} - \text{Fmeasure}_{other} - 0.01$. None of the baselines are able to satisfy the constraints.
While UncF1 yields better objective than UncError, it suffers significant constraint violations. 
The proposed algorithms significantly reduce the constraint violations, by trading-off for accuracy, with Algorithm \ref{algo:biconvex} yielding the least violation in most cases, and Algorithm \ref{algo:slack-ratios} doing better on the F-measure objective.

\section{Conclusion}
We have proposed a three-player game framework for solving learning problems where the objective and constraints are defined by generalized rate metrics. Our template recovered many previous algorithms and provides a modular approach for developing new improved algorithms. Specifically, we developed a new surrogate-based algorithm that can handle objectives and constraints that are convex functions of rates. The proposed algorithm can be applied to a wider range of metrics than existing surrogate methods (without additional restrictions on the model space), enables a tighter handling of constraints, and enjoys provable convergence guarantees (without assuming access to an idealized oracle). The benefit of our method over previous surrogate-based approaches was evident from their improved performance on many real-world fairness tasks. We also used our framework to provide new heuristic methods for optimizing with sums of ratios of rates, and showed that they are better at solving constrained training problem compared to existing baselines.
\paragraph{Acknowledgements.}
HN thanks Harish Ramaswamy, Purushottam Kar, Prateek Jain and Shivani Agarwal for several past discussions and collaborations on topics related to generalized rate metrics. The authors thank Heinrich Jiang and Qijia Jiang for helpful feedback on the writing.

\bibliographystyle{abbrv}
\bibliography{references1,references2}

\appendix
% \section{Preliminaries}
% \label{app:prelim}
\nocite{pennington2014glove}
\nocite{Cesa+14}
\section{Preliminaries}
\label{app:prelims}
We cover some preliminary material that will be useful in subsequent sections.
\paragraph{Fenchel conjugate}
For a convex function $\psi: \R^K \> \R$, the Fenchel conjugate $\psi^*: \R^K \> \R$ is defined by:
\begin{equation}
\psi^*(v) \,=\, 
\max_{c \,\in\, \dom\, \psi}\left\{
c^\top v \,-\, \psi(c)\right\}
\,=\,
-\min_{c \,\in\, \dom\, \psi}\left\{\psi(c) \,-\, c^\top v\right\},
\label{eq:fenchel-conjugate}
\end{equation}
where $\dom\, \psi$ denotes the domain of $\psi$. Since $\psi^*(v)$ is a point-wise maximum of linear functions in $v$, it is easy to see that $\psi^*$ is convex. We denote the second Fenchel conjugate of $\psi$ by $\psi^{**}$, which for any $c \in \dom \psi$ is given by:
\begin{equation}
\psi^{**}(c) \,=\, 
\max_{v \,\in\, \dom\,\psi^*}\left\{c^\top v
\,-\, \psi^*(v)
\right\}
\,=\, 
-\min_{v \,\in\, \dom\,\psi^*}\left\{\psi^*(v) \,-\, c^\top v\right\},
\label{eq:second-fenchel-conjugate}
\end{equation}
where $\dom\, \psi^*$ is the subset of $\R^K$ for which $\psi^*$ is defined. For example, if $\psi$ is monotonically non-decreasing in an argument, $\dom\,\psi^*$ contains vectors that are non-negative in that coordinate. For convex functions $\psi$, $\psi^{**}(c) = \psi(c),\, \forall c \in \dom\,\psi$.

The subdifferential $\partial\psi(c)$ of a convex  function $\psi$ at a point $c\in \R^K$ is the set of all vectors $v \in \R^K$ such that $f(c') \geq f(c) + v^\top (c' - c), \,  \forall c'  \in \R^K$. It  follows from the definition of the Fenchel conjugate % that follows directly from the Fenchel conjugate definition: 
that for any $c \in \dom\, \psi$,
\begin{equation}
 v \in \partial\psi(c)
~~\iff~~
\psi(c) + \psi^*(v) \,=\, v^\top c.
\label{eq:fenchel-young}
\end{equation}
This is often referred to as the \textit{Fenchel-Young equality}. 
A straight-forward consequence of \eqref{eq:fenchel-young} is that for any $c \in \dom\, \psi$,
\begin{equation}
 v^* \in \partial\psi(c)~~\iff~~
v^* \in \argmax{v \,\in\, \dom\,\psi^*}\left\{
c^{\top} v \,-\, \psi^*(v)\right\}.
\label{eq:psi-grad}
\end{equation}
% and similarly, for any $v \in \dom\, \psi^*$,
% \begin{equation}
% c^* \in \argmax{c \in \R^K}\left\{
% c^\top v \,-\, \psi(c)\right\}
% ~~\iff~~ c^* \in \nabla\psi^*(v).
% \label{eq:psi-star-grad}
% \end{equation}

\paragraph{Online convex optimization}
Consider an online convex optimization problem with a bounded convex set $\cC$, and a sequence of convex functions $h_1, h_2, \ldots, h_T: \cC \> \R$. In each round $t$, the learner needs to choose a point $c_t \in \cC$ and incurs a loss $h_t(c_t)$. The goal is to minimize the learner's average regret: 
$\mathcal{R}_T = \frac{1}{T}\sum_{t=1}^T h_t(c_t) \,-\, \min_{c\in \cC} \frac{1}{T}\sum_{t=1}^T h_t(c)$. 

Of course, the learner could implement a \textit{best response} (BR) strategy, i.e. simply output the minimizer of $h_t$ at each step: $c_t \in \amin{c \in \cC}\,h_t(c).$ This trivially gives a no-regret guarantee, but would be expensive in practice.

A more practical and simple approach to the above problem is the classical \textit{follow-the-leader} (FTL) algorithm, which simply outputs a point  with the least total loss for the previous rounds:
$$
c_{t+1} = \text{FTL}_\cC(h_1, \ldots, h_t) \in \amin{c \in \cC} \sum_{\tau=1}^t h_\tau(c).
$$
% We will denote this by $\text{FTL}_\cC(\ell_1, \ldots, \ell_t)$.
If each $h_t$ is Lipchitz and strongly convex, then FTL enjoys the following no-regret guarantee: $\mathcal{R}_T = \cO\big(\frac{\log(T)}{T}\big)$.

Another popular approach is \textit{online gradient descent} (OGD), which performs gradient updates based on the previously chosen point: 
$
c_{t+1} = \Pi_{\cC}\left(c_{t} \,-\, \nabla h_t(c_{t})\right),
$
where $\nabla h_t(c_{t}) \in \partial\, h_t(c_{t})$ and $\Pi_\cC$ denotes the $\ell_2$-projection onto $\cC$. If each $h_t$ is Lipchitz, then OGD has the following guarantee: $\mathcal{R}_T = \cO\Big(\sqrt{\frac{1}{T}}\Big)$.

\section{Connections to Existing Algorithms for \eqref{eq:custom-opt}}
\label{app:connections}
We elaborate on the three previous algorithm listed in Table \ref{tab:choices} and show how their updates can be seen as special cases of our approach. We present these methods using the notations used in this paper and re-written to handle a  convex function $\psi: \cR \> \R$ that is monotonic in each argument. We point out these approaches were not previously derived using the three-player game view-point that we present in this paper. 

\paragraph{SPADE:} This method \cite{Narasimhan+15b} seeks to optimize a convex function of two rates: $\psi(R_1(\theta), R_2(\theta))$, and does so by replacing the rates with convex surrogates $\tR_1$ and $\tR_2$, and performing  gradient updates on $\theta$ and additional dual variables $\alpha, \beta \in \R_+$:
\begin{equation}
\theta^{t+1} \leftarrow \Pi_{\Theta}\left(\theta^t \,-\, \eta\,\alpha^t\,\nabla\,\tilde{R}_1(\theta)
\,-\, \eta \beta^t\,\nabla\,\tilde{R}_2(\theta)\right)
\label{eq:spade-theta}
\end{equation}
\vspace{-7pt}
\begin{equation}
    (\alpha^{t+1},
    \beta^{t+1})
    \leftarrow
    \Pi_{\cA}\left(
    (\alpha^{t},
    \beta^{t})
    \,-\, \eta'(\tR_1(\theta), \tR_2(\theta)) \,+\,
    \eta' \nabla \psi^*(\alpha^t, \beta^t)\right),
\label{eq:spade-dual}
\end{equation}
where $\cA \subseteq \R_+^2$ is a bounded set that contains all gradients of $\psi$ and $\eta,\eta' > 0$. It is straight-forward to show that the above updates can be recovered from our  proposed framework by  having both the $\theta$- and $\lambda$-players play OGD on the surrogate Lagrangian, and the $\xi$-player play best response.
\begin{thm}
   With slack variables $\xi_1, \xi_2 \in  \R_+$ and Lagrange multipliers $\alpha, \beta \in \R_+$,  define $\cL(\theta, \xi; \alpha, \beta) \,=\, \psi(\xi_1, \xi_2) \,+\, \alpha(\tR_1(\theta) - \xi_1) \,+\, \beta(\tR_2(\theta) \,-\, \xi_2)$.
    Starting with the same $\theta^0, \alpha^0, \beta^0$, the following updates generate the same iterates $\theta^t$ as the updates in \eqref{eq:spade-theta} and \eqref{eq:spade-dual}:
    $$\xi^{t} \leftarrow \argmin{\xi \in \R_+^2}\, \cL(\theta^t, \xi; \alpha^t, \beta^t)
    $$
    $$\theta^{t+1} \leftarrow \Pi_{\Theta}\left(\theta^t \,-\, \eta\,\nabla_{\theta}\,\cL(\theta^t, \xi^{t}; \alpha^t, \beta^t)\right)
    $$
    $$(\alpha^{t+1},
    \beta^{t+1}) \leftarrow \Pi_{\cA}\left((\alpha^{t},
    \beta^{t}) \,-\, \eta'\,\nabla_{\alpha, \beta}\,\cL(\theta^t, \xi^{t}; \alpha^t, \beta^t)\right).
    $$
\end{thm}
The proof is straight-forward and follows from Lemma \ref{lem:xi-opt-closed-form}, which shows that a solution to the $\xi$-player's update is $\xi^{t} = \nabla\psi^*(\alpha^t, \beta^t)$. Substituting this into the expression for the $\lambda$-player's update gives us the same updates as in \eqref{eq:spade-theta} and \eqref{eq:spade-dual}.

\paragraph{NEMSIS:} This method \cite{Kar+16} was designed to work with nested convex functions $\psi$, i.e. functions $\psi$ that are a convex function of multiple convex functions. While these nested metrics can be easily handled within our framework by introducing separate slack variables for each of the inner convex functions, here we consider an application of this method to a simpler convex metric of the form $\psi(R_1(\theta), R_2(\theta))$. 
The NEMSIS updates for this metric, assuming $\psi$ is smooth, is given by:
\begin{equation}
\theta^{t+1} \leftarrow \Pi_{\Theta}\left(\theta^t \,-\, \eta\,\alpha^t\,\nabla\,\tilde{R}_1(\theta)
\,-\, \eta\beta^t\,\nabla\,\tilde{R}_2(\theta)\right)
\label{eq:nemsis-theta}
\end{equation}
$$
    (\alpha^{t+1},
    \beta^{t+1})
    \leftarrow
    \argmax{(\alpha, \beta) \in \cA}\,\sum_{\tau=1}^t g_\tau(\alpha, \beta)
$$
\begin{equation}
    ~~\text{where}~~
    g_\tau(\alpha, \beta) \,:=\,
\psi^*(\alpha, \beta) \,-\, \alpha\tR_1(\theta^\tau) \,-\, \beta\tR_2(\theta^\tau).
\label{eq:nemsis-dual}
\end{equation}

These updates can be recovered from our framework by having the $\lambda$-player play FTL on the combined objective that includes the min over $\xi$, and having the $\theta$-player play OGD, with both players operating on the surrogate Lagrangian.
\begin{thm}
    With slack variables $\xi_1, \xi_2 \in  \R_+$ and Lagrange multipliers $\alpha, \beta \in \R_+$, define $\cL_2(\theta; \alpha, \beta) = \alpha\,\tR_1(\theta) \,+\, \beta\,\tR_2(\theta)$ and $\cL(\theta, \xi; \alpha, \beta) \,=\, \psi(\xi_1, \xi_2) \,-\, \alpha\,\xi_1 \,-\, \beta\,\xi_2 \,+\, \cL_2(\theta; \alpha, \beta)$. 
    Starting with the same $\theta^0, \alpha^0, \beta^0$, the following updates generate the same iterates $\theta^t$ as the updates in \eqref{eq:nemsis-theta} and \eqref{eq:nemsis-dual}:
    $$\theta^{t+1} \leftarrow \Pi_{\Theta}\left(\theta^t \,-\, \eta\,\nabla_{\theta}\,\cL_2(\theta^t; \alpha^t, \beta^t)\right)
    $$
    \begin{equation*}
        (\alpha^{t+1},
        \beta^{t+1})
        \leftarrow
        \textrm{\textup{FTL}}_{\cA}\left(
        \left\{
        \textstyle
        \min_{\xi \in \R_+^2}
        \cL(\theta^\tau, \xi; \cdot, \cdot)
        \right\}_{\tau=1}^t
        \right)
    %\label{eq:nemsis-dual}
    \end{equation*}
\end{thm}
The proof follows in a straight-forward manner from the definition of the Fenchel conjugate of $\psi$: $\psi^*(\alpha, \beta) = \min_{\xi_1,\xi_2 \in \R_+}\{
\psi(\xi_1, \xi_2) \,-\, \alpha\xi_1 \,-\, \beta\xi_2
\}
$.

\paragraph{F-W:} The Frank-Wolfe based method of \cite{Narasimhan+15} was designed to optimize performance metrics $\psi(R_1(\theta), \ldots, R_K(\theta))$ with smooth, convex $\psi$. 
This employs the classical Frank-Wolfe algorithm along with a cost-sensitive oracle to perform the inner linear optimization.
Specifically, the algorithm maintains iterates $\theta^t \in \Theta$ and $\r^t \in \R_+^K$. Starting with $\r^0 = \bR(\0)$, it performs the following updates:
\begin{equation}
\theta^{t+1} \leftarrow \argmin{\theta \in \Theta}\,
\sum_{k=1}^K v^t_k R_k(\theta),
~~\text{where}~\bv^t \leftarrow \nabla\psi(\r^t)
\label{eq:fw-1}
\end{equation}
\begin{equation}
\r^{t+1} \leftarrow \left(1  - \frac{1}{t}\right) \r^t + \frac{1}{t}\, \bR(\theta^{t+1}),
\label{eq:fw-2}
\end{equation}
where the first step uses a cost-sensitive oracle. The output of this algorithm after $T$ iterations is a stochastic model with probability $\frac{1}{t}\prod_{\tau=t+1}^T
\left(1-\frac{1}{\tau}\right)$ on $\theta^t$. 

These updates can be recovered from our game formulation by having the $\theta$-player play OGD and the $\lambda$-player play FTL on a combined objective that contains a min over $\xi$, with both players operating on the true Lagrangians.
\begin{thm}
     With slack variables $\xi \in  \R^K_+$ and Lagrange multipliers $\lambda \in \R^K_+$,  define $\cL(\theta, \xi; \lambda) \,=\, \psi(\xi) \,+\, \sum_{k=1}^K \lambda_k (R_k(\theta) - \xi_k)$ and $\cL_2(\theta; \lambda) \,=\,\sum_{k=1}^K \lambda_k R_k(\theta)$. Let $\Lambda = \R_+^K$. 
     Then starting with $\lambda^0 = \nabla\psi(\bR(\0))$, 
    the following updates generate the same iterates $\theta^t$ as \eqref{eq:fw-1} and \eqref{eq:fw-2}:
    \begin{equation}
    \theta^{t+1} \leftarrow \argmin{\theta \in \Theta}\,\cL_2(\theta; \lambda^t)
    \label{eq:fw-players-theta}
    \end{equation}
    \begin{equation}
        \lambda^{t+1}
        \leftarrow
        \textrm{\textup{FTL}}_{\Lambda}\left(
        \left\{
        \textstyle
        \min_{\xi \in \R_+^K}
        \cL(\theta^\tau, \xi; \cdot)
        \right\}_{\tau=1}^t
        \right)
    \label{eq:fw-players-lambda}
    \end{equation}
\end{thm}
\begin{proof}
The proof follows directly from a similar result in \cite{AbernethJ17}. We first note that by definition of the Fenchel conjugate of $\psi$:
\[
\min_{\xi \in \R_+^K} \cL(\theta, \xi; \lambda) ~=~
\psi^*(\lambda) \,+\, \sum_{k=1}^K \lambda_k R_k(\theta).
\]
We expand the FTL update and show that $\lambda^t$ here plays the same role as $\bv^t$ in the F-W method.
\begin{eqnarray*}
\lambda^{t+1} &=& \argmax{\lambda \in \R_+^K}\,\left\{ t\psi^*(\lambda) \,+\, \sum_{\tau=1}^t \sum_{k=1}^K \lambda_k R_k(\theta^\tau)\right\}\\
&=& \argmax{\lambda \in \R_+^K}\,\left\{ \psi^*(\lambda) \,+\,  \sum_{k=1}^K \lambda_k \frac{1}{t}\sum_{\tau=1}^t R_k(\theta^\tau)
\right\}\\
&=& \nabla\psi\left(\frac{1}{t}\sum_{\tau=1}^t R_k(\theta^\tau)\right)
~=~ \nabla\psi(\r^t) ~=~ \bv^t,
\end{eqnarray*}
where the third equality follows from \eqref{eq:psi-grad}.
With the above observation, it is clear that \eqref{eq:fw-1} and \eqref{eq:fw-players-theta} generate the same iterates $\theta^t$.
\end{proof}

We note that Abernethy and Wang (2017) \cite{AbernethJ17} previously pointed out that the classical Frank-Wolfe algorithm for convex optimization can be viewed as computing the equilibrium of a zero-sum game. We directly apply their result to show a connection between the method of Narasimhan et al.\ (2015)\cite{Narasimhan+15} and our framework.

\section{SPADE+: An Adaptation of SPADE for \eqref{eq:non-linear-opt}}
\label{app:spade+}
\begin{algorithm}[t]
\caption{SPADE+ for \eqref{eq:non-linear-opt}}
\label{algo:spade+}
\begin{algorithmic}
\STATE Initialize: $\theta^0$, $\blambda^0$
\FOR{$t = 0 $ to $T-1$}
\STATE $\xi^{t} \,=\,
  \nabla \Phi^*(
\lambda^t_{J+1}, \ldots, \lambda^t_{J+K})$, ~where $\Phi: \xi \mapsto \sum_{j=1}^J \lambda^t_j\, \xi_k$
\STATE $\theta^{t+1} \leftarrow \Pi_{\Theta}(\theta^t \,-\, \eta_\theta\,\nabla_\theta\,\tL_2(\theta^t;\, \blambda^t))$
%\STATE $\bxi^{t+1} \leftarrow \bxi^t \,-\, \eta_{\bxi}\,\Pi_{\bxi}(\nabla_\theta\,\tL(\theta^t, \bxi^t;\, \blambda^t))$
\STATE $\blambda^{t+1} \leftarrow \Pi_{\Lambda}(\blambda^t + \eta_{\blambda}\,\nabla_{\blambda}\,\tL(\bxi^{t},\theta^t;\, \blambda^t))$
\ENDFOR
\RETURN $\bar{\theta} = \frac{1}{T}\sum_{t=1}^T\theta^t$
\end{algorithmic}
\end{algorithm}

We extend the SPADE algorithm of \cite{Narasimhan+15b}, originally proposed for unconstrained optimization of generalized rate metrics, to the constrained problem in \eqref{eq:non-linear-opt}. We seek to solve following max-min problem, where we use a convex relaxation to the Lagrangian for both the 
$\theta$- and $\lambda$-player:
\begin{equation}
    \max_{\lambda \in \R_+^K}\,\min_{\substack{\theta \in \Theta,\\ \xi\in \cR}}\,
    \underbrace{
    \cL_1(\xi; \lambda) \,+\, \tL_2(\theta; \lambda)}_{\tL(\xi, \theta; \lambda)},
\label{eq:spade+-max-min}
\end{equation}
The procedure is outlined in Algorithm \ref{algo:spade+}, where the $\theta$- and $\lambda$-players perform OGD updates on surrogate objectives, and the $\xi$-player plays best response. Since the max-min objective in \eqref{eq:spade+-max-min} is convex in $\theta$,
the algorithm returns the average of the model across all iterates.
%\begin{figure}
%\caption{Here $\Pi_\Lambda$ denotes the $\ell_1$-projection onto $\Lambda$ and  $\Pi_\Theta$ denotes the $\ell_2$-projection onto $\Theta$.}
% and $\bar{\lambda}^t \,=\, \sum_{j=1}^J \lambda^t_j$.} %\todohari{$\ell_2$ for $\theta$?}}
%\vspace{-10pt}
%\end{figure}
We can then show the following convergence guarantee for Algorithm \ref{algo:spade+} under the assumption that $\Theta$ only contains models for which each $\phi^j(\tilde{\bR}(\theta))$ is defined. 
\begin{thm}
    \label{thm:convergence-spade+}
    Let $\bar{\theta}$ be the model returned by Algorithm \ref{algo:spade+}. 
    % Let $\theta^1, \ldots, \theta^T$ be the iterates of Algorithm \ref{algo:lagrangian-ideal} for \eqref{eq:non-linear-opt}, and let $\bar{\mu}$ be a stochastic model with probability $\frac{1}{T}$ on $\theta^t$.  %Suppose each $\phi^j$ is strictly convex, monotonically non-decreasing in each argument and $L$-Lispchitz w.r.t.\ $\ell_\infty$ norm, and $g$ is convex. 
    % Let $\lambda^*$ be a maximizer of $\min_{\xi, \mu}\cL(\xi, \mu; \lambda)$ over $\lambda \in \R^K$ and set $\kappa \geq 2\|\lambda^*\|_1$. 
    Let $\Theta$ be a bounded convex set such that each $\phi^j(\tilde{\bR}(\theta))$ is defined  for all $\theta \in \Theta$. Let ${\theta}^* \in {\Theta}$ be such that 
    $\theta^*$ is feasible, i.e.\ $\phi^j(\tilde{\bR}(\theta^*)) \leq 0,\,\forall j \in [J]$, and
    $g({\theta}^*) \leq g({\theta})$ for all $\theta \in \tilde{\Theta}$ that are feasible. % satisfy $\phi^j(\tilde{\bR}(\theta)) \leq 0,\,\forall j \in [J]$.  
    Suppose there exists a $\theta' \in \Theta$ such that $\phi^j(\tilde{\bR}(\theta')) \leq -\gamma, \, \forall j \in [J]$, for some $\gamma > 0$. Let $B_g \,=\, \max_{\theta \in \Theta}\,g(\theta)$.
    Let $B_\Theta \geq \max_{\theta \in \Theta}\,\|\theta\|_2$, $B_\theta \geq \max_{t}\|\nabla_{\theta}\tL_2( \theta^t; \blambda^t)\|_2$ and
    $B_\lambda \,\geq\, \max_{t}\|\nabla_{\blambda}\tL(\bxi^t, \theta^t; \blambda^t)\|_2$. 
     Then setting $\kappa \,=\, 2(L+1)B_g/\gamma$, $\eta_\theta = \frac{B_\Theta}{B_{\theta}\sqrt{2T}}$ and $\eta_\lambda = \frac{\kappa}{B_{\blambda}\sqrt{2T}}$, we have w.p. $\geq 1 - \delta$ over draws of stochastic gradients:
    \begin{equation*}
        g(\bar{\theta}) \,\leq\,
        g({\theta}^*)
         \,+\, \cO\bigg(\sqrt{\frac{\log(1/\delta)}{T}}\bigg)
        %\end{equation*}
        ~~~~~\text{and}~~~~~
        %\begin{equation*}
        \phi^j(\tbR(\bar{\theta}))
        \,\leq\,
        \cO\bigg(\sqrt{\frac{\log(1/\delta)}{T}}\bigg),~~~\forall j \in [J].
    \end{equation*}    
\end{thm}

The proof follows from an adaptation of the proofs for Theorems \ref{thm:convergence-ideal-constrained} and \ref{thm:convergence-surrogate-constrained} (see Sections \ref{app:thm-convergence-ideal-constrained} and \ref{app:thm-convergence-surrogate-constrained}). The iterates of SPADE+ form an approximate  mixed Nash equilibrium of the zero-sum game in \eqref{eq:spade+-max-min}.

% \begin{table}
% \begin{tabular}{lcccccccc}
% 	\toprule
% 	& UncError & SPADEavg & SPADEbest & SPADElast & SPADEstoch & Stochastic & Pruned & Deterministic \\\cmidrule(r){2-6}
% 	COMPAS	 & 0.115 (1.00) 	 & 0.000 (1.02) 	 & 0.003 (1.04) 	 & 0.009 (1.05) 	 & 0.002 (1.01) 	 & 0.016 (1.02) 	 & 0.000 (1.03) 	 & 0.000 (1.03) \\
% Crime	 & 0.224 (1.00) 	 & 0.157 (1.00) 	 & 0.141 (1.00) 	 & 0.152 (1.00) 	 & 0.158 (1.11) 	 & 0.116 (1.09) 	 & 0.094 (1.11) 	 & 0.085 (1.16) \\
% Law School	 & 0.199 (1.00) 	 & 0.296 (1.01) 	 & 0.042 (1.04) 	 & 0.247 (1.05) 	 & 0.036 (1.07) 	 & 0.074 (1.11) 	 & 0.054 (1.12) 	 & 0.046 (1.10) \\
% Adult	 & 0.114 (1.00) 	 & 0.098 (1.01) 	 & 0.019 (1.11) 	 & 0.050 (1.08) 	 & 0.046 (1.09) 	 & 0.068 (1.08) 	 & 0.014 (1.10) 	 & 0.014 (1.10) \\
% WikiToxicity	 & 0.175 (1.00) 	 & 0.106 (1.23) 	 & 0.068 (1.22) 	 & 0.108 (1.19) 	 & 0.123 (1.17) 	 & 0.104 (1.25) 	 & 0.133 (1.09) 	 & 0.125 (1.18) \\
% MapFacts	 & 0.014 (1.00) 	 & 0.006 (1.03) 	 & 0.003 (1.01) 	 & 0.006 (1.02) 	 & 0.003 (1.01) 	 & 0.005 (1.08) 	 & 0.007 (1.07) 	 & 0.007 (1.07) \\
% 	\bottomrule
% \end{tabular}
% \end{table}

\section{Surrogates for KL-divergence}
\label{app:kld}

We now explain why it is difficult to apply previous surrogate-based methods to performance measures such as the KL-divergence metric  in Table \ref{tab:measures} that are defined over a restricted domain, and why our framework provides a cleaner solution for optimizing these metrics.

\paragraph{Difficulty with SPADE \& NEMSIS.}
Consider the negative logarithm of the prediction rate of a model: $-\log(\E_{X}(\mathbb{I}\{f_\theta(X) \geq 0\}))$. This is one of the terms in the KL-divergence  metric in Table \ref{tab:measures}.
%\footnote{E.g.\ consider a surrogate $-\log(\tR(\theta))$ for upper-bounding the negative logarithm; for this to be a convex in $\theta$,  $\tR$ needs to either be a constant or be allowed to take negative values (see Appendix \ref{app:kld} for details).}
%
Previous surrogate methods such as SPADE \cite{Narasimhan+15b} or NEMSIS \cite{Kar+16}
optimize this metric by first 
replacing the indicator function within the logarithm with a surrogate function such as the hinge-loss to obtain a convex upper bound, and then  using this  surrogate approximation for all updates. The only issue with this approach is that the log is not defined for negative values, and as we shall see below, any reasonable choice of surrogate to replace the inner indicator would produce values that are negative. 

To see this, first note that $-\log$ is convex, but non-increasing in its input. 
Hence to construct a convex upper bound for the negative log-rate, we would have to replace the inner indicator with a concave lower-bounding surrogate. However, if we insist that the surrogate needs to be both concave and lower-bound the indicator function, then we would have to allow the surrogate to be negative for some parts of the input space (unless the surrogate is a constant function). %, rendering the outer log ill-defined. 
For example, a
hinge-based convex upper bound on the negative log metric would look like $-\log(\E_{X}(\min\{1, f_\theta(X)\}))$. But if  $f_\theta(X)$ takes large negative values for a large-enough portion of the input space $X$, the term within the log would be negative, rendering this function ill-defined. See Figure \ref{fig:concave-lower-bound} for an illustration.

\begin{figure}
    \centering
    \includegraphics[width=0.3\textwidth]{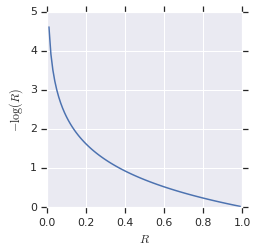}
    ~~
    \includegraphics[width=0.3\textwidth]{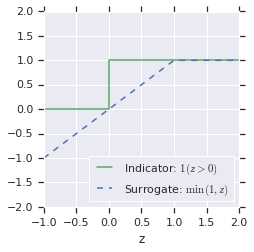}
    \caption{Left: Plot of negative logarithm as a function of rate. Right: A hinge-based concave lower bound on the indicator function.}
    \label{fig:concave-lower-bound}
\end{figure}

The above issue points to a drawback with the SPADE and NEMSIS methods that rely heavily on the use of surrogates for optimization. For instance, to optimize
$-\log(\E_{X}(\min\{1, f_\theta(X)\}))$ using
the NEMSIS method and  a hinge-based surrogate \cite{Kar+16}, we would implement the following dual update (see \eqref{eq:nemsis-dual}):
$$
    \alpha^{t+1}
    \leftarrow
    \argmax{\alpha \in \R_+}\,
    \left\{\log(\alpha+\epsilon) - \alpha\frac{1}{t}\sum_{\tau=1}^t\,\E_{X}\left[\min\{1, f_{\theta^\tau}(X)\}\right]\right\}.
$$
A small $\epsilon > 0$ is added to the log argument to avoid numerical issues. Note that if the term $\E_{X}\left[\min\{1, f_{\theta^\tau}(X)\}\right]$ is negative at any point during the course of optimization, the maximization over $\alpha$ becomes unbounded, and the iterates will never converge from that point onwards.

One way NEMSIS can be applied  to handle log-rates, is to strongly regularize the model to not output large negative values. 
We  believe that this is how % the authors of \cite{Kar+16} use 
the NEMSIS was previously used to optimize the KL-divergence metric \cite{Kar+16}. % in their paper. 
This approach may however unnecessarily restrict the space of models we are allowed to optimize over, and may prevent us from finding good solutions.

\paragraph{Applying Algorithm \ref{algo:lagrangian-surrogate} to KLD.}
On the other hand, Algorithm \ref{algo:lagrangian-surrogate} offers a cleaner solution to dealing with log-rates and performance measures based on the KL-divergence. By using original rates (instead of surrogates) for the updates on $\lambda$, the proposed algorithm ensures that the game play between $\lambda$ and $\xi$ never produces values that are outside the domain of $\psi$. 

For completeness, we derive the updates for Algorithm \ref{algo:lagrangian-surrogate} for minimizing the KL-divergence metric  between the true positive proportion $p$ and the model's prediction rate $\hat{p}(\theta)$.  The optimization problem (with constants removed) is given by:
\begin{equation*}
\min_{\theta \in \Theta}\, -p\log(\hat{p}(\theta)) \,-\, (1-p)\log(1 - \hat{p}(\theta)).
\end{equation*}
Introducing slack variables for the terms within the log, this can be re-written as: 
\begin{equation}
\min_{\theta \in \Theta,~ \xi \in \R^2_+}\, -p\log(\xi_1) \,-\, (1-p)\log(\xi_2)
~~~~\text{s.t.}~~~\xi_1 \leq \hat{p}(\theta),~~\xi_2 \leq 1 - \hat{p}(\theta).
\label{eq:log-helper}
\end{equation}
Further, formulating the Lagrangian:
\[
\cL(\theta, \xi; \lambda) \,=\,
-p\log(\xi_1) \,-\, (1-p)\log(\xi_2)
\,+\,
\lambda_1(\xi_1 - \hat{p}(\theta))
\,+\,
\lambda_2(\xi_2 - (1-\hat{p}(\theta))),
\]
we  seek to solve the following max-min optimization problem:
\[
\max_{\lambda \in \R_+^2}\,
\min_{\theta\in\Theta,~\xi \in \R_+^2}\,
\cL(\theta, \xi; \lambda).
\]
Notice that the optimum Lagrange multipliers $\lambda$ are always non-negative. 

Using a concave lower-bounding surrogate $\tilde{p}_+$ for $\hat{p}$ and a concave lower-bounding surrogate $\tilde{p}_-$ for $1-\hat{p}$, we have the following updates for Algorithm \ref{algo:lagrangian-surrogate}:
$$
\xi_1^{t+1} \,=\, \frac{p}{\lambda_1^t + \epsilon};~~
\xi_2^{t+1} \,=\,
\frac{1-p}{\lambda_2^t + \epsilon}
$$
$$
\theta^{t+1} \,=\,
\Pi_{\Theta}\Big(
\theta^t \,-\,
\eta_\theta\left(
\lambda_1^t\,
\nabla_\theta\, [-\tilde{p}_+(\theta^t)]
\,+\, 
\lambda_2^t\,
\nabla_\theta\, [-\tilde{p}_-(\theta^t)]
\right)
\Big)
$$
$$
\lambda^{t+1}
 \,=\,
\Pi_\Lambda\Big(
\lambda^t \,+\,
\eta_\lambda
\left(
\xi^{t+1} \,-\,
\begin{bmatrix}
p(\theta^t)\\
1-p(\theta^t)
\end{bmatrix}
\right)
\Big)
$$
where $\Lambda \subset \R_+$ is a bounded set and $\epsilon > 0$ is a very small value to avoid numerical issues. Because $\lambda \geq \0$, both $\xi_1$ and $\xi_2$ are always non-negative, ensuring that the $\log$ in \eqref{eq:log-helper} is always defined.

\paragraph{Other Metrics.} As noted in Section \ref{sec:limitations-existing-surrogates}, the above issue also arises with the G-mean metric, where the square root is undefined for negative values, as well as fractional-linear metrics such as the F-measure. In the case of fractional-linear metrics, one can obtain a pseudo-convex upper bound for the metric by replacing the numerator with a convex upper-bounding surrogate and the denominator with a concave lower-bounding surrogate; however the pseudo-convex property of the resulting function holds only if the surrogate for the denominator evaluates to non-negative values, which as with the KLD, poses a restriction on the model class we are allowed to use.

\section{Connections to Cotter et al. (2019)}\label{app:cotteretal}

Like us, Cotter et al.~\cite{Cotter+19} proposed using surrogates \emph{only} for one of the two players of their game, however, unlike us, they did not prove a convergence result for such an approach, when using the usual Lagrangian formulation. Instead, they proposed a different non-zero-sum game, which they called the proxy-Lagrangian. In the context of \eqref{eq:custom-opt}, we can write the proxy-Lagrangian of \eqref{eq:custom-opt-as-max-min} as:
\begin{align*}
    \mathcal{L}_{\theta,\xi}\left(\theta,\xi;\lambda\right) =& \lambda_1 \psi\left(\xi_1,\ldots,\xi_K\right) + \sum_{k=1}^K \lambda_{k+1}\left( \tilde{R}_K(\theta) - \xi_k \right) \\
    \mathcal{L}_{\lambda}\left(\theta,\xi;\lambda\right) =& \sum_{k=1}^K \lambda_{k+1}\left( R_K(\theta) - \xi_k \right)
\end{align*}
where $\lambda \in \Delta^{K-1} \subseteq \R^{K}$ is the $(K-1)$-probability-simplex---notice that, unlike for the Lagrangian formulation, a multiplier $\lambda_1$ is associated with the \emph{objective}, in addition to the $K$ multipliers associated with the constraints. A similar variant could be easily given for \eqref{eq:non-linear-opt}. The only difference from their presentation is that we have grouped $\theta$ and $\xi$ together as the parameters of the first player. Cotter et al.~\cite{Cotter+19}'s proposed approach to such a problem is to minimize $\mathcal{L}_{\theta,\xi}$ over $\theta$ and $\xi$ using an external regret minimizing algorithm (e.g. OGD), and to maximize $\mathcal{L}_{\lambda}$ over $\lambda$ using a \emph{swap regret} minimizing algorithm. They go on to prove a $\mathcal{O}(1/\sqrt{T})$ convergence rate.

This discussion shows that our proposed approach can be plugged-in to theirs straightforwardly. The resulting algorithm is more complicated (due to the use of swap regret), but the convergence guarantee has a better $T$-dependence. Unfortunately, since they do not use the Lagrangian formulation, and their $\lambda$-player minimizes swap regret, it's difficult to see how analytically optimizing over $\xi$, as we do in Algorithms \ref{algo:lagrangian-ideal} and \ref{algo:lagrangian-surrogate}, could be incorporated into their approach.

\section{Proof of Theorem \ref{thm:convergence-ideal}}
\label{app:thm-convergence-ideal}
% Our proof adapts ideas from previous results on constrained optimization and game equilibrium \cite{Narasimhan18, Cotter+19b}.
All our proofs are for the case where the domain of $\psi$ is $\R_+^K$. This covers all the metrics considered in this paper, including the log-rate.

We first state a lemma that relates the Lagrangian $\cL$ with the Fenchel conjugate of $\psi$. In particular, we will show that if $\psi$ is $L$-Lipschitz in $[0,1]^K$, it suffices to set the radius of the space of Lagrange multipliers over which we optimize $\cL$ to be at most $L$. This will later be helpful in choosing $\kappa$ in Algorithm \ref{algo:lagrangian-ideal}.
\begin{lem}
\label{lem:app-fenchel-conjugate}
Suppose $\psi$ is strictly convex, monotonically non-decreasing in each argument and $L$-Lipschitz w.r.t.\ the $\ell_\infty$-norm in $[0,1]^K$.
Setting the radius of the space of Lagrange multipliers $\Lambda$ to be at most $L$, i.e.\ $\Lambda \,=\, \{\lambda \in \R_+^K: \,|\, \|\lambda\|_1 \leq L\}$, we have for any $\mu \in \Delta_\Theta$:
\begin{enumerate}
    \item $
    \displaystyle
    \lambda^* \in \argmax{\lambda \in \R_+^K}\,\Big\{
    \min_{\xi \in \cR}\,\cL(\xi, \mu;\, \lambda)\Big\}
    ~\Rightarrow~ \lambda^* \in \Lambda
    $
    \item $
    \displaystyle
    \psi(\bR(\mu))\,=\, \max_{\lambda \in \Lambda}
    \Big\{
    \min_{\xi \in \cR}\cL(\xi, \mu;\, \lambda)
    \Big\}.
    $
    % \item $
    % \displaystyle
    % \nabla\psi(\bR(\mu))\,=\, \argmax{\lambda \in \R_+^K}\Big\{
    % \min_{\xi \in \cR}\,\cL(\xi, \mu;\, \lambda)\Big\}
    % $
    \item $
    \displaystyle
    \nabla\psi^*(\lambda^*)\,=\, \argmin{\xi \in \cR}\,\cL_1(\xi;\, \lambda^*),~\forall \lambda^*\in \Lambda.
    $
\end{enumerate}
\end{lem}
\begin{proof}
%Since there exists $\lambda^* \in \Lambda$ that maximizes $\min_{\xi} \cL(\xi, \mu; \lambda)$ over all $\lambda \in \R^K$, we only need to perform this maximization over $\Lambda$. 
1. By strong duality, we first have:
\begin{equation}
\psi(\bR(\mu))~=~\max_{\lambda \in \R_+^K}\,
    \min_{\xi \in \cR}\,\cL(\xi, \mu;\, \lambda)
\label{eq:strong-duality}
\end{equation}
Since $\psi$ is strictly convex, it's differentiable.
% The convexity of $\psi$ further gives us:
% \begin{eqnarray*}
% \lambda^* = \nabla\psi(\bR(\mu))
% &\iff&
% \psi(\xi') \,\geq\,
% \psi(\bR(\mu)) \,+\, 
% \langle \lambda^*,\, \xi' \,-\, \bR(\mu) \rangle,~\forall \xi' \in \cR
% \\
% &\iff&
% \psi(\bR(\mu)) 
% \,\leq\,
% \psi(\xi') 
% \,-\, 
% \langle \lambda^*,\, \xi' \rangle
% \,+\, 
% \langle \lambda^*,\, \bR(\mu) \rangle
% ,~\forall \xi' \in \cR
% \\
% &\iff&
% \psi(\bR(\mu)) 
% \,=\,
% \min_{\xi' \in \cR}
% \left\{\psi(\xi') 
% \,-\, 
% \langle \lambda^*,\, \xi' \rangle
% \,+\, 
% \langle \lambda^*,\, \bR(\mu) \rangle
% \right\}
% \\
% % &\iff&
% % \psi(\bR(\mu))\,=\
% % -\psi^*(\lambda^*)
% % \,+\, 
% % \langle \lambda^*,\, \bR(\mu) \rangle\\
% &\iff&
% \psi(\bR(\mu))\,=\
% \min_{\xi \in \cR}\,
% \cL(\xi, \mu;\, \lambda^*)\\
% &\iff&
% \lambda^* \in
% \argmax{\lambda \in \R_+^K}\,
%     \min_{\xi \in \cR}\,\cL(\xi, \mu;\, \lambda).
% \end{eqnarray*}
From the Fenchel-Young’s equality in \eqref{eq:fenchel-young}, we have:
\begin{eqnarray*}
\lambda^* = \nabla\psi(\bR(\mu))
&\iff&
\psi(\bR(\mu))\,=\,
-\psi^*(\lambda^*)
\,+\, 
\sum_{k=1}^K \lambda_k^*\, R_k(\mu)\\
&\iff &
\psi(\bR(\mu))\,=\,
\min_{\xi \in \cR}
\left\{
\psi(\xi)
\,-\,
\sum_{k=1}^K \lambda^*_k\, \xi_k
\right\}
\,+\, 
\sum_{k=1}^K \lambda^*_k\, R_k(\mu) 
\\
&\iff&
\psi(\bR(\mu))\,=\,
\min_{\xi \in \cR}\,\cL(\xi, \mu;\, \lambda^*)
\\
&\iff&
\max_{\lambda \in \R_+^K}\,
    \min_{\xi \in \cR}\,\cL(\xi, \mu;\, \lambda)
    \,=\,
\min_{\xi \in \cR}\,\cL(\xi, \mu;\, \lambda^*)
\\
&\iff &
\lambda^* \,\in\, \argmax{\lambda \in \R_+^K}\,
    \min_{\xi \in \cR}\,\cL(\xi, \mu;\, \lambda),
\end{eqnarray*}
where where the second step uses the definition of Fenchel dual $\psi^*$ (see \eqref{eq:fenchel-conjugate}), the fourth step follows from \eqref{eq:strong-duality} and the last step
uses monotonicity of $\psi$ and the fact the gradients of $\psi$ all have non-zero entries. This shows that  maximizer of $\min_{\xi} \cL(\xi, \mu; \lambda)$ over $\lambda \in \R_+^K$ is unique and equal to the gradient  $\nabla\psi(\bR(\mu))$. Because $\psi$ is $L$-Lipschitz w.r.t.\ the $\ell_\infty$-norm  in $[0,1]^K$, the gradient norm $\|\nabla\psi(\bR(\mu))\|_1\,\leq\, L$. Hence, the set $\Lambda$ always contains a maximizer of $\min_{\xi} \cL(\xi, \mu; \lambda)$ over $\lambda \in \R_+^K$.

% The statement then follows from the Lipchiztness and monotonicity of $\psi$.

2. For any $\mu$:
\begin{eqnarray*}
\max_{\lambda \in \Lambda}\min_{\xi \in \cR}\,
\cL(\xi, \mu;\,\lambda)&=&
\max_{\lambda \in \Lambda}\,
\left\{
\min_{\xi \in \cR}
\left\{\psi(\xi) \,-\, \sum_{k=1}^K \lambda_k \xi_k\right\}
\,+\, 
\sum_{k=1}^K \lambda_k\, R_k(\mu)
\right\}
\\
&=&
\max_{\lambda \in \Lambda}\left\{
-\psi^*(\lambda)
\,+\, 
\sum_{k=1}^K \lambda_k\, R_k(\mu)
\right\}\\
&=&
\max_{\lambda \in \R^K}\left\{
-\psi^*(\lambda)
\,+\, 
\sum_{k=1}^K \lambda_k\, R_k(\mu)
\right\}\\
&=& \psi^{**}(\bR(\mu))\\
&=& \psi(\bR(\mu)),
\end{eqnarray*}
where in the fourth step, we have used statement 1 to replace the max of $\Lambda$ with a max over $\R_+^K$, the fifth step uses the definition of second Fenchel conjugate (see \eqref{eq:second-fenchel-conjugate}); the last step follows from convexity of $\psi$.

3. By strict convexity and monotonicity of $\psi$, $\psi^*$ is defined for any $\lambda^* \in \Lambda$ and is differentiable in $\Lambda$. We have from the Fenchel-Young's equality in \eqref{eq:fenchel-young}:
\begin{eqnarray*}
\xi^* = \nabla\psi^*(\lambda^*)
&\iff&
\psi^*(\lambda^*)
\,=\, \psi(\xi^*) \,-\, \sum_{k=1}^K \lambda^*_k\xi^*_k\\
&\iff& \psi^*(\lambda^*)\,=\,
\cL_1(\xi^*; \lambda^*)\\
&\iff& \xi^* \,\in\,
\argmin{\xi \in \cR}\,\cL_1(\xi;\, \lambda^*).
\end{eqnarray*}
This completes the proof of parts 1-3.
%or in other words $\xi^* \in \amin{\xi \in \cR}\, \cL_1(\xi; \lambda^*)$.
\end{proof}
\subsection{Proof of Lemma \ref{lem:xi-opt-closed-form}}
The proof follows directly from statement 3 of Lemma \ref{lem:app-fenchel-conjugate}.

\subsection{General Convergence Result}
We present a convergence result for a general no-regret strategy for the $\lambda$-player, and then apply it to the case where the player runs OGD with specific step-sizes. The iterates  generated by Algorithm \ref{algo:lagrangian-ideal} for \eqref{eq:custom-opt} yield an approximate Nash equilibrium, i.e. the $\lambda$-player choosing the fixed strategy $\bar{\lambda} = \frac{1}{T}\sum_{t=1}^T \lambda^{t}$, 
the $\xi$-player choosing the fixed strategy $\bar{\xi} = \frac{1}{T}\sum_{t=1}^T \xi^t$, and the $\theta$-player choosing a uniform distribution $\bar{\mu}$ over $\theta^1, \ldots, \theta^T$, together comprise an approximate mixed-strategy Nash equilibrium of the zero-sum game in \eqref{eq:custom-opt-as-max-min-stochastic}.
%\todohari{Existence of maximizer -- mention assumption somewhere}
\begin{thm}
    Let $\theta^1, \ldots, \theta^T, \xi^1, \ldots, \xi^T, \lambda^1, \ldots, \lambda^T$ be the iterates generated by Algorithm \ref{algo:lagrangian-ideal} for \eqref{eq:custom-opt}. Suppose the $\lambda$ iterates satisfy the following:
    \[
    \frac{1}{T}\sum_{t=1}^T\,\cL(\xi^t, \theta^t; \lambda^t) 
    \,\geq\, 
    \max_{\lambda \in \Lambda}\,\frac{1}{T}\sum_{t=1}^T\,\cL(\xi^t, \theta^t; \lambda) \,-\,
    \epsilon_\lambda,
    \]
    for some $\epsilon_\lambda>0$.  Suppose $\psi$ is strictly convex, monotonically non-decreasing in each argument and $L$-Lispchitz w.r.t.\ $\ell_\infty$ norm in $[0,1]^K$. 
    Let $\bar{\mu}$ be a stochastic model with a probability mass of $\frac{1}{T}$ on $\theta^t$. 
    % Let $\lambda^*$ be a maximizer of $\min_{\xi, \mu}\cL(\xi, \mu; \lambda)$ over $\lambda \in \R_+^K$. 
    Then setting $\kappa = L$:
    $$
    \psi\big(\bR(\bar{\mu})\big)
    \,\leq\, \min_{\mu \in \Delta_\Theta}\psi\big(\bR(\mu)\big) \,+\,
    \epsilon_\lambda + \rho.
    %\left(2L/\kappa+1\right)(\epsilon_\lambda \,+\, \rho).
    $$
\label{thm:convergence-ideal-general}
\end{thm}
\begin{proof}%[Proof of Theorem \ref{thm:convergence-ideal-general}]
% Let $\bar{\xi} = \frac{1}{T}\sum_{t=1}^T \xi^t$ and $\bar{\lambda} = \frac{1}{T}\sum_{t=1}^T \lambda^{t}$. 
%Let $(\xi^*, \mu^*, \lambda^*)$ be such that $\lambda^* \in \amax{\lambda \in \Lambda}\,\cL(\xi^*, \mu^*; \lambda)$ and $(\xi^*, \mu^*) \in \amin{\xi, \mu}\,\cL(\xi, \mu; \lambda^*).$
% The value that we set for the Lagrange multipliers radius $\kappa$ ensures that there exists $\lambda^* \in \Lambda$ that maximizes  $\min_{\xi, \mu}\cL(\xi, \mu; \lambda)$ over all $\lambda \in \R_+^K$. 
% Given that we set $\kappa$ to the Lipschitz constant of $\psi$, we can apply statement 3 of Lemma \ref{lem:app-fenchel-conjugate} with $\mu = \mu^*$ and get  $\cL(\xi^*, \mu^*; \lambda^*) \,=\, \psi(\bR(\mu^*))$.
%
The best-response strategies of the $\xi$-player and $\theta$-players give us: 
\begin{equation}
\cL_1(\xi^{t}; \lambda^t) \,=\, \min_{\xi \in \cR}\cL_1(\xi; \lambda^t)
\label{eq:app-xi}
\end{equation}
\begin{equation}
% \frac{1}{T}\sum_{t=1}^T\,\cL_2(\theta^t; \lambda^t) \,\leq\,
% \min_{\theta \in \Theta}\,\frac{1}{T}\sum_{t=1}^T\,\cL_2(\theta; \lambda^t)
% \,+\, \epsilon_\theta\\
\cL_2(\theta^{t}; \lambda^t) \,\leq\, \min_{\theta \in \Theta}\cL_2(\theta; \lambda^t) \,+\, \rho
\label{eq:app-theta}
\end{equation}
From \eqref{eq:app-xi} and \eqref{eq:app-theta}, we further get:
\begin{eqnarray}
\frac{1}{T}\sum_{t=1}^{T}\,
\cL(\xi^{t}, \theta^{t}; \lambda^{t}) 
&=&
\frac{1}{T}\sum_{t=1}^{T}\,
\min_{\xi \in \cR}
\cL(\xi, \theta^{t}; \lambda^{t}) 
\nonumber
\\
&\leq&
\frac{1}{T}\sum_{t=1}^{T}\, 
\min_{\xi \in \cR, \theta \in \Theta}
\cL(\xi, \theta; \lambda^{t}) 
\,+\, \rho
\nonumber
\\
&=&
\frac{1}{T}\sum_{t=1}^{T}\, 
\min_{\xi \in \cR, \mu \in \Delta_\Theta}
\cL(\xi, \mu; \lambda^{t}) 
\,+\, \rho
~~~~~(\text{by linearity of $\cL$ in $\mu$})
\nonumber
\\
&\leq&
\min_{\xi \in \cR, \mu \in \Delta_\Theta}
\frac{1}{T}\sum_{t=1}^{T}\, 
\cL(\xi, \mu; \lambda^{t}) 
\,+\, \rho
\nonumber
\\
&=&
\min_{\xi \in \cR, \mu \in \Delta_\Theta}
\cL\left(\xi, \mu;\, \bar{\lambda}\right)
\,+\, \rho
~~~~~(\text{by linearity of $\cL$ in $\lambda^t$})
\nonumber\\
&\leq&
\max_{\lambda\in \Lambda}
\min_{\xi \in \cR, \mu \in \Delta_\Theta}
\cL(\xi, \mu; \lambda) 
\,+\, \rho
\nonumber\\
&=&
\min_{\mu \in \Delta_\Theta}\,
\psi(\bR(\mu))
\,+\, \rho,
\label{eq:app-ideal-theta}
\end{eqnarray}
where the last step follows from statement 2 of Lemma \ref{lem:app-fenchel-conjugate} (given that the radius of the space of Lagrange multipliers $\kappa $ is set to $L$).

Next, from the theorem statement, the OGD updates on $\lambda$ satisfy:
\begin{eqnarray}
\frac{1}{T}\sum_{t=1}^T\,
\cL(\xi^{t}, \theta^{t}; \lambda^t) &\geq&
\max_{\lambda \in \Lambda}\,
\frac{1}{T}\sum_{t=1}^T\,
\cL(\xi^{t}, \theta^{t}; \lambda)
\,-\, \epsilon_\lambda
\nonumber
\\
&\geq&
\min_{\xi \in \cR}\,
\max_{\lambda \in \Lambda}\,
\frac{1}{T}\sum_{t=1}^T\,
\cL(\xi, \theta^{t}; \lambda)
\,-\, \epsilon_\lambda
\nonumber
\\
&\geq&
\max_{\lambda \in \Lambda}\,
\min_{\xi \in \cR}\,\frac{1}{T}\sum_{t=1}^T\,
\cL(\xi, \theta^{t}; \lambda)
\,-\, \epsilon_\lambda
\nonumber
\\
&=&
\max_{\lambda \in \Lambda}\,
\min_{\xi \in \cR}\,
\cL(\xi, \bar{\mu}; \lambda)
\,-\, \epsilon_\lambda
\nonumber
\\
&=& \psi(\bR(\bar{\mu})) \,-\, \epsilon_\lambda,
\label{eq:app-ideal-lambda}
\end{eqnarray}
where the last step follows from statement 2 of Lemma \ref{lem:app-fenchel-conjugate}.

Combining \eqref{eq:app-ideal-theta} and \eqref{eq:app-ideal-lambda}, we have the desired result.
\end{proof}

\subsection{Corollary for OGD on $\lambda$}
% \begin{lem}
% Suppose $\psi$ be monotonically non-decreasing in each argument and $L$-Lipschitz w.r.t.\ the $\ell_\infty$ norm. Fix $\mu \in \Delta_\Theta$. Then for any
% $$\lambda^* \,\in\, \argmax{\lambda\in \R_+^K}
% \left\{
% \min_{\xi \in  [0,1]^K}\,
% \cL(\xi, \mu;\,\lambda)
% \right\},
% $$
% the following holds:
% $\|\lambda^*\|_1 \,\leq\, L$.
% \label{lem:app-Lagrange-radius}
% \end{lem}
% \begin{proof}
% %For any $z  \in \nabla\psi(\bR(\mu))$, by monotonicity of $\psi$, $z \geq \0$, and 
% Since $\psi$ is $L$-Lipschitz w.r.t.\ the $\ell_\infty$-norm, $\|z\|_1 \leq L$. From Lemma \ref{lem:app-fenchel-conjugate}, we have $\lambda^* \in \nabla\psi(\bR(\mu))$ and hence $\|\lambda^*\|_1 \,\leq\, L$.
% \end{proof}

\begin{proof}[Proof of Theorem \ref{thm:convergence-ideal}]
We apply standard OGD convergence analysis \cite{Zinkevich03} to the $\lambda$-player's gradient updates on $\lambda$ and standard online-to-batch conversion arguments \cite{Cesa+14} (see e.g.\ Theorem 7 in Cotter et al. (2019) \cite{Cotter+19}). For the sequence of losses $-\cL(\xi^1, \theta^1;\, \cdot), \ldots, -\cL(\xi^T, \theta^T;\, \cdot)$, with $\eta = \frac{\kappa}{B_{\blambda}\sqrt{2T}}$ in Algorithm \ref{algo:lagrangian-ideal}, we get the following regret bound. With probability at least $1 - \delta$ over draws of stochastic gradients of $\cL$:
\[
    \frac{1}{T}\sum_{t=1}^T\,\cL(\xi^t, \theta^t; \lambda^t) 
    \,\geq\, \max_{\lambda \in \Lambda}\,\frac{1}{T}\sum_{t=1}^T\,\cL(\xi^t, \theta^t; \lambda) \,-\, 2\kappa\,B_\lambda\sqrt{\frac{1 + 16\log(1/\delta)}{T}}.
\]
For the above, we use the fact that the gradients $\nabla_\lambda \cL( \theta^t, \xi^t; \blambda^t)$ are unbiased, i.e.\ 
$\E\left[\nabla_\lambda \cL( \theta^t, \xi^t; \blambda^t)\right] \,\in\, \partial_\lambda\,\cL( \theta^t, \xi^t; \blambda^t)$. 
Following statement 1 in Lemma \ref{lem:app-fenchel-conjugate}, we  set the radius of the space of Lagrange multipliers $\Lambda$ to $\kappa = L$ and apply Theorem \ref{thm:convergence-ideal-general} to complete the proof:
$$
    \psi\left(\bR(\bar{\mu})\right)
    \,\leq\, \min_{\mu \in \Delta_\Theta}\psi\left(\bR(\mu)\right) \,+\,2L\,B_\lambda\sqrt{\frac{1 + 16\log(1/\delta)}{T}} \,+\, \rho.
$$
\end{proof}

%%%%%%%%%%%%%%%
%%%%%%%%%%%%%%%

\section{Proof of Theorem \ref{thm:convergence-surrogate}}
\label{app:thm-convergence-surrogate}
%We provide a proof assuming access to exact gradients $\nabla_\theta\tR_k(\theta)$. The proof extends in a straight-forward manner to stochastic gradients using standard online-to-batch conversion arguments \cite{Cesa+14}. Our proof applies to  $\mathcal{C} = \R_+^K$. 
We first state a lemma that relates the surrogate Lagrangian with the Fenchel conjugate of $\psi$:
\begin{lem}
\label{lem:app-fenchel-conjugate-proxy}
Let $\tL(\xi, \mu; \lambda) \,=\, \cL_1(\xi; \lambda) + \tL_2(\mu; \lambda)$. % Suppose $\psi$ is strictly convex. 
For any $\theta \in \Theta$, for which $\tilde{\bR}(\theta) \in \dom\, \psi$:
$$\psi(\tilde{\bR}(\theta))\,=\, \max_{\lambda \in \R_+^K}\min_{\xi \in \cR}\tL(\xi, \theta;\, \lambda).$$
% Similarly $\kappa = L$, for any $\theta \in \Theta$:
% $$\psi(\bR(\theta))\,=\, \max_{\lambda \in \Lambda}\min_{\xi \in \cR}\cL(\xi, \theta;\, \lambda).$$
\end{lem}
\begin{proof}
The proof follows the same steps as statement 2 of Lemma \ref{lem:app-fenchel-conjugate}.
%Since there exists $\lambda^* \in \Lambda$ that maximizes $\min_{\xi} \cL(\xi, \mu; \lambda)$ over all $\lambda \in \R^K$, we only need to perform this maximization over $\Lambda$. 
% We first observe that for any $\theta \in \Theta$ for which $\psi(\tilde{\bR}(\theta))$ is defined:
% \begin{eqnarray*}
% \max_{\lambda \in \R_+^K}\min_{\xi \in  [0,1]^K}\,
% \tL(\xi, \theta;\,\lambda)&=&
% \max_{\lambda \in \R_+^K}\,
% \left\{
% \min_{\xi \in \cR}
% \left\{\psi(\xi) \,-\, \sum_{k=1}^K \lambda_k \xi_k\right\}
% \,+\, 
% \sum_{k=1}^K \lambda_k\, \tilde{R}_k(\theta)
% \right\}
% \\
% &=&
% \max_{\lambda \in \R_+^K}\left\{
% -\psi^*(\lambda)
% \,+\, 
% \sum_{k=1}^K \lambda_k\, \tilde{R}_k(\theta)
% \right\}\\
% &=& \psi^{**}(\tilde{\bR}(\theta)) \,=\, \psi(\tilde{\bR}(\theta)),
% \end{eqnarray*}
% where $\psi^{**}$ denotes the Fenchel conjugate of $\psi^*$.% The second part follows similarly.
\end{proof}

\subsection{General Convergence Result}
We present a convergence result for general no-regret strategies for the $\theta$- and $\lambda$-player that reach an approximate coarse-correlated equilibrium, and then apply it to the case where the players run OGD with specific step-sizes.

\begin{thm}
    Let $\theta^1, \ldots, \theta^T, \xi^1, \ldots, \xi^T, \lambda^1, \ldots, \lambda^T$ be the iterates generated by Algorithm \ref{algo:lagrangian-surrogate} for \eqref{eq:custom-opt}. Let $\tilde{\Theta} \,=\, \big\{\theta \in {\Theta} \,|\, 
    \tilde{\bR}(\theta)\,\in\, \dom\, \psi
    \big\}$.
    Suppose the iterates comprise the following approximate coarse-correlated equilibrium:
    \begin{equation}
    \frac{1}{T}\sum_{t=1}^T\,\cL_1(\xi^{t}; \lambda^t) \,\leq\, \min_{\xi \in \cR}\frac{1}{T}\sum_{t=1}^T\,\cL_1(\xi; \lambda^t);
    \label{eq:xi-best-response}
    \end{equation}
    \begin{equation}
    \frac{1}{T}\sum_{t=1}^T\,\tL_2(\theta^t; \lambda^t) 
    \,\leq\, \min_{\theta \in {\Theta}}\,\frac{1}{T}\sum_{t=1}^T\,\tL_2(\theta; \lambda^t) \,+\, \epsilon_\theta;
    \label{eq:ogd-theta}
    \end{equation}
    \begin{equation}
    \frac{1}{T}\sum_{t=1}^T\,\cL(\xi^t, \theta^t; \lambda^t) 
    \,\geq\, \max_{\lambda \in \Lambda}\,\frac{1}{T}\sum_{t=1}^T\,\cL(\xi^t, \theta^t; \lambda) \,-\, \epsilon_\lambda,
    \label{eq:ogd-lambda}
    \end{equation}
    for some $\epsilon_\theta>0$ and $\epsilon_\lambda>0$. 
     Suppose $\psi$ is strictly convex, monotonically non-decreasing in each argument and $L$-Lispchitz w.r.t.\ $\ell_\infty$ norm. 
    % Let $\lambda^*$ be a maximizer of $\min_{\xi, \mu}\cL(\xi, \mu; \lambda)$ over $\lambda \in \R^K$ and set $\kappa \geq 2\|\lambda^*\|_1$. 
    %Let $\tilde{B}_\theta = \max_{\theta \in \Theta}\,\psi(\tilde{\bR}(\theta))$.
    Let $\bar{\mu}$ be a stochastic model with a probability mass of $\frac{1}{T}$ on $\theta^t$. Then setting $\kappa = L$:
    %Then setting $\kappa = \frac{L}{\sqrt{\epsilon_\lambda}}$:
    $$
    \psi\big({\bR}(\bar{\mu})\big)
    \,\leq\, \min_{\theta \in \tilde{\Theta}}\psi\big(\tilde{\bR}(\theta)\big) 
%   \,+\,
%     \tilde{B}_\theta\sqrt{\epsilon_\lambda}
%     \,+\,
%   (\sqrt{\epsilon_\lambda} + 1 )(\epsilon_\theta + \epsilon_\lambda).
    \,+\,
    % \frac{L\tilde{B}_\theta}{\kappa}
    % \,+\,
    % \left(\frac{L}{\kappa}+1\right)(\epsilon_\theta + \epsilon_\lambda)
    \epsilon_\theta + \epsilon_\lambda
    $$
\label{thm:convergence-surrogate-general}
\end{thm}

\begin{proof}[Proof of Theorem \ref{thm:convergence-surrogate-general}]
% Let  $\bar{\xi} = \frac{1}{T}\sum_{t=1}^T \xi^t$ and $\bar{\lambda} = \frac{1}{T}\sum_{t=1}^T \lambda^{t}$. 
% Let $(\xi^*, \mu^*, \lambda^*)$ be such that $\lambda^* \in \amax{\lambda \in \Lambda}\,\cL(\xi^*, \mu^*; \lambda)$
% and $(\xi^*, \mu^*)
% \in \amin{\xi, \mu}\,\cL(\xi, \mu; \lambda^*).
% $
%The value that we set for the Lagrange multipliers $\kappa$ ensures that there exists $\lambda^* \in \Lambda$ that maximizes  $\min_{\xi, \mu}\cL(\xi, \mu; \lambda)$ over all $\lambda \in \R^K$. Given this, we can apply Lemma \ref{lem:app-fenchel-conjugate} with $\mu = \mu^*$ and get  $\cL(\xi^*, \mu^*; \lambda^*) \,=\, \psi(\bR(\mu^*))$.
%
% The best-response strategy of the $\xi$-player gives us: 
% We have for the $\xi$-player:
% \begin{equation*}
% \cL_1(\xi^{t}; \lambda^t) \,=\, \min_{\xi \in \cR}\cL_1(\xi; \lambda^t)
% \label{eq:app-xi-surrogate}
% \end{equation*}
% \begin{equation}
% % \frac{1}{T}\sum_{t=1}^T\,\cL_2(\theta^t; \lambda^t) \,\leq\,
% % \min_{\theta \in \Theta}\,\frac{1}{T}\sum_{t=1}^T\,\cL_2(\theta; \lambda^t)
% % \,+\, \epsilon_\theta\\
% \cL_2(\theta^{t+1}; \lambda^t) \,\leq\, \min_{\theta \in \Theta}\cL_2(\theta; \lambda^t) \,+\, \rho
% \label{eq:app-theta}
% \end{equation}
% From \eqref{eq:app-xi} and \eqref{eq:app-theta}, we further get:
% and we further get:
We have:
\begin{eqnarray}
\frac{1}{T}\sum_{t=1}^{T}\,
\tL(\xi^{t}, \theta^{t}; \lambda^{t})
% &=&
% \frac{1}{T}\sum_{t=1}^{T}\,
% \min_{\xi \in \cR}
% \cL_1(\xi; \lambda^{t}) 
% \,+\,
% \tL_2(\theta^{t}; \lambda^{t}) 
% \nonumber
% \\
&\leq&
\min_{\xi \in \cR}
\frac{1}{T}\sum_{t=1}^{T}\, 
\cL_1(\xi; \lambda^{t})
\,+\,
\frac{1}{T}\sum_{t=1}^{T}\, 
\tL_2(\theta^{t}; \lambda^{t}) 
~~~~~(\text{from \eqref{eq:xi-best-response}})
\nonumber
\\
&\leq&
\min_{\xi \in \cR}
\frac{1}{T}\sum_{t=1}^{T}\, 
\cL_1(\xi; \lambda^{t})
\,+\,
\min_{\theta \in \Theta}\,
\frac{1}{T}\sum_{t=1}^{T}\, 
\tL_2(\theta; \lambda^{t}) 
\,+\, \epsilon_\theta
~~~~~(\text{from \eqref{eq:ogd-theta}})
\nonumber
\\
&\leq&
\min_{\xi \in \cR, \theta \in \Theta}
\frac{1}{T}\sum_{t=1}^{T}\, 
\tL(\xi, \theta; \lambda^{t})
\,+\, \epsilon_\theta
\nonumber
\\
&=&
\min_{\xi \in \cR,\, \theta \in \Theta}\, 
\tL(\xi, \theta; \bar{\lambda})
\,+\, \epsilon_\theta
~~~~~(\text{by linearity of $\tL$ in $\lambda^t$})
\nonumber
\\
&\leq&
\max_{\lambda \in \R_+^K}
\min_{\xi \in \cR,\, \theta \in \Theta}\, 
\tL(\xi, \theta; {\lambda})
\,+\, \epsilon_\theta
\nonumber
\\
&\leq&
\max_{\lambda \in \R_+^K}
\min_{\xi \in \cR,\, \theta \in \tilde{\Theta}}\, 
\tL(\xi, \theta; {\lambda})
\,+\, \epsilon_\theta
~~~~~(\text{as $\tilde{\Theta} \subseteq \Theta$})
\nonumber
\\
&=&
\min_{\theta\in \tilde{\Theta}}\psi(\tilde{\bR}(\theta))
\,+\,
\epsilon_\theta,
\label{eq:app-surrogate-theta}
\end{eqnarray}
where the last step follows from Lemma \ref{lem:app-fenchel-conjugate-proxy}.

Next, the OGD updates on $\lambda$ satisfy:
\begin{eqnarray}
\frac{1}{T}\sum_{t=1}^T\,
\cL(\xi^{t}, \theta^{t}; \lambda^t) &\geq&
\max_{\lambda \in \Lambda}\,
\frac{1}{T}\sum_{t=1}^T\,
\cL(\xi^{t}, \theta^{t}; \lambda)
\,-\, \epsilon_\lambda
~~~~~(\text{from \eqref{eq:ogd-lambda}})
\nonumber
\\
&\geq&
\min_{\xi \in \cR}\,
\max_{\lambda \in \Lambda}\,
\frac{1}{T}\sum_{t=1}^T\,
\cL(\xi, \theta^{t}; \lambda)
\,-\, \epsilon_\lambda
\nonumber
\\
&\geq&
\max_{\lambda \in \Lambda}\,
\min_{\xi \in \cR}\,\frac{1}{T}\sum_{t=1}^T\,
\cL(\xi, \theta^{t}; \lambda)
\,-\, \epsilon_\lambda
\nonumber
\\
&=&
\max_{\lambda \in \Lambda}\,
\min_{\xi \in \cR}\,
\cL(\xi, \bar{\mu}; \lambda)
\,-\, \epsilon_\lambda
\nonumber
\\
&=& \psi(\bR(\bar{\mu})) \,-\, \epsilon_\lambda,
\label{eq:app-surrogate-lambda}
\end{eqnarray}
%
%  Because $\xi^t$ is a best-response to $\lambda^t$:
%  \begin{equation}
% \frac{1}{T}\sum_{t=1}^T\,
% \min_{\xi \in \cR}\,\cL(\xi, \theta^{t}; \lambda^t) \,\geq\,
% \frac{1}{T}\sum_{t=1}^T\,
% \cL(\xi^{t}, \theta^{t}; \lambda')
% \,-\, \epsilon_\lambda.
% \label{eq:app-surrogate-lambda}
% \end{equation}
where the last step follows from statement 2 in Lemma \ref{lem:app-fenchel-conjugate}.

Combining \eqref{eq:app-surrogate-theta} and \eqref{eq:app-surrogate-lambda} using the surrogate upper-bounding property, i.e.\ using $\tL(\xi^{t}, \theta^{t}; \lambda^{t}) \geq \cL(\xi^{t}, \theta^{t}; \lambda^{t})$, gives us the desired result.
\end{proof}

\subsection{Corollary for OGD on $\theta$ and $\lambda$}
\begin{proof}[Proof of Theorem \ref{thm:convergence-surrogate}]
We show that  the iterates of Algorithm \ref{algo:lagrangian-surrogate} form an approximate coarse-correlated equilibrium and apply Theorem \ref{thm:convergence-surrogate-general}. 

The best-response strategy of the $\xi$-player gives us:
\[
\frac{1}{T}\sum_{t=1}^{T}\,
\cL_1(\xi^{t}; \lambda^{t})
~=~
\frac{1}{T}\sum_{t=1}^{T}
\min_{\xi \in \cR}
\, 
\cL_1(\xi; \lambda^{t})
~\leq~
\min_{\xi \in \cR}
\frac{1}{T}\sum_{t=1}^{T}\, 
\cL_1(\xi; \lambda^{t}).
\]
We then derive no-regret guarantees for the updates of the $\theta$- and $\lambda$-player using standard convergence analysis for OGD \cite{Zinkevich03} and standard online-to-batch conversion arguments \cite{Cesa+14} (see e.g.\ Theorem 7 in Cotter et al. (2019) \cite{Cotter+19}). 
For the sequence of losses $-\cL(\xi^1, \cdot;\, \lambda^1), \ldots, -\cL(\xi^T, \cdot;\, \lambda^T)$ optimized by the $\theta$-player, setting $\eta = \frac{B_\Theta}{B_{\theta}\sqrt{2T}}$ in Algorithm \ref{algo:lagrangian-surrogate}, we get the following regret bound for the $\theta$-player (see Corollary 3 in \cite{Cotter+19} for the complete derivation). With probability $\geq 1 - \delta/2$ over draws of  stochastic gradients of $\tL_2$:
\begin{eqnarray*}
    \frac{1}{T}\sum_{t=1}^T\,\tL(\xi^t, \theta^t; \lambda^t) 
    &\leq& \min_{\theta \in \Theta}\,\frac{1}{T}\sum_{t=1}^T\,\tL(\xi^t, \theta; \lambda^t) \,+\, 2B_\Theta\,B_\theta\sqrt{\frac{1 + 16\log(2/\delta)}{T}}.
    % \\
    % &\leq& \min_{\theta \in \tilde{\Theta}}\,\frac{1}{T}\sum_{t=1}^T\,\tL(\xi^t, \theta; \lambda^t) \,+\, 2B_\Theta\,B_\theta\sqrt{\frac{1 + 16\log(2/\delta)}{T}}~~~~(\text{as $\tilde{\Theta} \subseteq \Theta$}).
\end{eqnarray*}
The regret guarantee uses the fact that the gradients $\nabla_\theta \tL_2( \theta^t; \blambda^t)$ are unbiased, i.e.\ 
$\E\left[\nabla_\theta \tL_2( \theta^t; \blambda^t)\right] \,\in\, \partial_\theta\,\tL_2( \theta^t; \blambda^t)$. 

Similarly, for the sequence of losses $-\cL(\xi^1, \theta^1;\, \cdot), \ldots, -\cL(\xi^T, \theta^T;\, \cdot)$ optimized by the $\lambda$-player, setting $\eta = \frac{\kappa}{B_{\blambda}\sqrt{2T}}$ in Algorithm \ref{algo:lagrangian-surrogate}, we get the following regret bound. With probability $\geq 1 - \delta/2$ over draws of  stochastic gradients  of $\cL$:
\[
    \frac{1}{T}\sum_{t=1}^T\,\cL(\xi^t, \theta^t; \lambda^t) 
    \,\geq\, \max_{\lambda \in \Lambda}\,\frac{1}{T}\sum_{t=1}^T\,\cL(\xi^t, \theta^t; \lambda) \,-\, 2\kappa\,B_\lambda\sqrt{\frac{1 + 16\log(2/\delta)}{T}}.
\]
This again uses the fact that the gradients $\nabla_\lambda \cL( \theta^t, \xi^t; \blambda^t)$ are unbiased, i.e.\ 
$\E\left[\nabla_\lambda \cL( \theta^t, \xi^t; \blambda^t)\right] \,\in\, \partial_\lambda\,\cL( \theta^t, \xi^t; \blambda^t)$. 

Following statement 1 of Lemma \ref{lem:app-fenchel-conjugate}, we set the radius of the space of Lagrange multipliers $\Lambda$ to $\kappa = L$ and apply Theorem \ref{thm:convergence-surrogate-general} to complete the proof, getting with probability $\geq 1 - \delta$ over draws of  stochastic gradients of $\cL$ and $\tL_2$:
% An application of Theorem \ref{thm:convergence-ideal-general} with $\kappa = L$ then gives us:
$$
    \psi\big({\bR}(\bar{\mu})\big)
    \,\leq\, \min_{\theta \in \tilde{\Theta}}\psi\big(\tilde{\bR}(\theta)\big) \,+\,
    % \frac{\tilde{B}_\theta}{T^{1/4}}
    % \,+\,
    % \left(\frac{1}{T^{1/4}}+1\right)\left(B_\Theta\,B_\theta\sqrt{\frac{2}{T}} + L\,B_\lambda\frac{\sqrt{2}}{T^{1/4}}
    %\right),
    2B_\Theta\,B_\theta\sqrt{\frac{1 + 16\log(2/\delta)}{T}} + 2L\,B_\lambda\sqrt{\frac{1 + 16\log(2/\delta)}{T}},
    $$
as desired.
\end{proof}

\section{Proof of Theorem \ref{thm:convergence-ideal-constrained}}
\label{app:thm-convergence-ideal-constrained}
The proof adapts ideas from previous results on constrained optimization and game equilibrium \cite{Narasimhan18, Cotter+19b}. %Our proof applies to  $\mathcal{C} = \R_+^K$. 
We will find it useful to first prove a couple of lemmas. % relate the Lagrangian $\cL$ with the Fenchel conjugate of $\psi$:
%\todohari{Name $\cL$ differently for constrained and unconstrained}
\begin{lem}
\label{lem:app-Lagrange-opt-constrained}
Suppose each $\phi^j$ is strictly convex and monotonically non-decreasing in each argument and $g$ is convex.  
Let $\cL$ be as defined in \eqref{eq:L-min-expand-constrained}. 
Let $\mu^* \in \Delta_\Theta$ be such that 
$\mu^*$ is feasible, i.e.\ $\phi^j(\bR(\mu^*)) \leq 0,\,\forall j \in [J]$, and
$\E_{\theta \sim \mu^*}\left[g(\theta)\right] \leq \E_{\theta \sim \mu}\left[g(\theta)\right]$ for every $\mu \in \Delta_{\Theta}$ that is feasible. 
Further, for $\lambda \in \R_+^{J+K}$, denote $a = [\lambda_1,\ldots,\lambda_J]^\top$ and $b = [\lambda_{J+1}, \ldots, \lambda_{J+K}]^\top$, and for $a\in \R_+^J$, let $\Phi_a(\xi) \,=\, \sum_{j=1}^J a_j\, \phi^j(\xi)$. Then for any $\mu \in \Delta_\Theta$ and $a \in \R_+^J$:
\begin{enumerate}
    \item 
    $\displaystyle \E_{\theta\sim \mu^*}\left[g(\theta)\right]\,=\, \max_{\lambda \in \R_+^{J+K}}\min_{\xi \in \cR,\, \mu \in \Delta_\Theta}\cL(\xi, \mu;\, \lambda).$
    % \item $\nabla\psi(\bR(\mu))\,=\, \amax{\lambda \in \R^K}\,\cL(\xi^*, \mu;\, \lambda)$
    % \item $\nabla\psi^*(\lambda^*)\,=\, \amin{\xi \in \cR}\,\cL_1(\xi;\, \lambda^*)$.
%\end{enumerate}
%\begin{enumerate}
    \item $\displaystyle\nabla\Phi_a(\bR(\mu))\,=\, \argmax{b \in \R_+^K}\min_{\xi \in \cR}\cL(\xi, \mu;\, a, b)$
    \item $\displaystyle \nabla\Phi^*_a(b)
    \,=\, \argmin{\xi \in \cR}\,\cL_1(\xi;\, b), ~\forall b \in \R_+^K$.
\end{enumerate}
\end{lem}
\begin{proof}
%Since there exists $\lambda^* \in \Lambda$ that maximizes $\min_{\xi} \cL(\xi, \mu; \lambda)$ over all $\lambda \in \R^K$, we only need to perform this maximization over $\Lambda$. 
1. We re-state $\cL$ from \eqref{eq:L-min-expand-constrained} for a stochastic classifier $\mu$:
\begin{eqnarray*}
\cL(\mu, \xi; \lambda)
&=&
\E_{\theta \sim \mu}\left[g(\theta)\right]
\,+\,
\sum_{j=1}^J
\lambda_{j}\,
\phi^j(\xi) \,+\,
\sum_{k=1}^K \lambda_{J+k}\, (R_k(\mu) \,-\, \xi_{k}).
\end{eqnarray*}
Since $\cL$ is linear in $\mu(\theta)$, convex in $\xi$ and linear in $\lambda$, strong duality holds, and we have:
\begin{eqnarray*}
\max_{\lambda \in \R_+^{J+K}}\min_{\xi \in \cR,\, \mu \in \Delta_\Theta}\cL(\xi, \mu;\, \lambda)
&=&
\min_{\mu, \xi:\,\xi \geq \bR(\mu),\,\phi^j(\xi) \leq 0, ~\forall j}\, \E_{\theta \sim \mu}\left[g(\theta)\right]\\
&=&
\min_{\mu:\,\phi^j(\bR(\mu)) \leq 0, ~\forall j}\, \E_{\theta \sim \mu}\left[g(\theta)\right]
~~~~~(\text{by monotonicity of $\phi^j$'s})
\\
&=&
% \min_{\theta:\,\phi^j(\bR(\theta)) \leq 0, ~\forall j}\, g(\theta)~~=~~ 
\E_{\theta\sim \mu^*}\left[g(\theta)\right].
\end{eqnarray*}
%where the last step follows from monotonicity of $\phi^j$'s. Since $g$ is convex in $\theta$, we know that for any  $\mu$, $\E_{\theta \sim \mu}\left[g(\theta)\right] \,\geq\, g(\E_{\theta \sim \mu}[\theta])$. Hence the minimization on the right-hand side can be equivalently performed over deterministic models in $\Theta$.

2--3. We note that for fixed $a\in \R_+^J$:
\begin{eqnarray*}
\max_{b \in \R_+^K}\,\min_{\xi \in  [0,1]^K}\,\cL(\xi, \mu; a, b)&=&
\E_{\theta \sim \mu}\left[g(\theta)\right]
\,+\,
\max_{ b \in \R_+^K}
\left\{
\min_{\xi \in \cR}\,
\left\{
\sum_{j=1}^J a_j\, \phi^j(\xi) \,-\, \sum_{k=1}^K b_k\, \xi_{k}
\right\}
\,+\,
\sum_{k=1}^K b_k\,R_k(\mu)
\right\}
\\
&=&
\E_{\theta \sim \mu}\left[g(\theta)\right]
\,+\,
\max_{ b \in \R_+^K}
\left\{
\min_{\xi \in \cR}
\left\{
\Phi_a(\xi) \,-\, \sum_{k=1}^K b_k\, \xi_{k}
\right\}
\,+\,
\sum_{k=1}^Kb_k\,R_k(\mu)
\right\}
\\
&=&
\E_{\theta \sim \mu}\left[g(\theta)\right]
\,+\,
\max_{b \in \R_+^K}
\left\{
-\Phi_a^*(b)
\,+\,
\sum_{k=1}^K b_k\,R_k(\mu)
\right\}\\
&=&
\E_{\theta \sim \mu}\left[g(\theta)\right]
\,+\,
\Phi_{a}^{**}(\bR(\mu))\\
&=&
\E_{\theta \sim \mu}\left[g(\theta)\right]
\,+\,\Phi_a(\bR(\mu))
\end{eqnarray*}
The second and third parts then follow by convexity of $\phi^j$'s and the Fenchel-Young's equality. 
See statements 1 and 3 in  Lemma \ref{lem:app-fenchel-conjugate} for the detailed proof steps. The third part additionally needs strict convexity to ensure that $\Phi^*_a(b)$ is defined for all $b \in \R_+^K$.
% The third part follows from strict convexity of $\phi^j$'s and Fenchel-Young's equality (see Lemma \ref{lem:app-fenchel-conjugate} for the detailed proof steps).
\end{proof}

%Before proving the above theorem, we will find it useful to state a lemma.
\begin{lem}
 Let $\phi: \cR \> \R$  be monotonically non-decreasing in each argument and be $L$-Lipschitz w.r.t.\ $\ell_\infty$ norm  in $[0,1]^K$. 
% Then for any $\xi, \Delta \in \R_+^K$,
% $
% \phi(\xi + \Delta) \,-\, \phi(\xi) \,\leq\, L\max_{k\in [K]} (\Delta_k)_+,
% $
% where $(\Delta_k)_+ = \max\{0, \Delta_k\}$. In particular, 
Then for any $\bar{\mu} \in \Delta_\Theta$
and $\bar{\xi} \in \R_+^K$:
\[
\phi(\bR(\bar{\mu})) \,\leq\,
\phi(\bar{
\xi}) \,+\, L\,\max_{k\in [K]}(R_k(\bar{\mu}) \,-\, \bar{\xi}_k)_+,
\]
\label{lem:helper-Lipchitz}
where $(z)_+ = \max\{0, z\}$.
\end{lem}
\begin{proof}
% For the first part, %let $\1_k$ denote a vector where the $k$th entry is 1 and the other entries are 0. Then 
% by the Lipschitz property of $\phi$, for any $\xi' \in [0,1]^K$
% $
% |\phi(\xi' + \Delta_k) \,-\, \phi(\xi')| \,\leq\, L|\Delta_k|.
% $
We first show that 
for any $\xi, \Delta \in \R_+^K$,
$
\phi(\xi + \Delta) \,-\, \phi(\xi) \,\leq\, L\max_{k\in [K]} (\Delta_k)_+,
$
and the lemma directly follows from this.
Define $\Delta^+_k \,=\, (\Delta_k)_+$ and $\Delta^-_k \,=\, (-\Delta_k)_+$. Then
\begin{eqnarray*}
\phi(\xi + \Delta) \,-\, \phi(\xi)
&=& \phi(\xi + \Delta^+ - \Delta^-) \,-\, \phi(\xi)
\\
&=& \phi(\xi + \Delta^+ - \Delta^-) \,-\, \phi(\xi - \Delta^-) \,+\,
\phi(\xi - \Delta^-)
\,-\, \phi(\xi)
\\
&\leq&
L\max_{k\in [K]}\,|\Delta^+_k|
\,+\,
\phi(\xi - \Delta^-)
\,-\, \phi(\xi)
~~~~~(\text{from the Lipchitz property of $\phi$})
\\
&\leq& L\,\max_{k\in [K]}\,|\Delta^+_k| \,+\, 0 
~~~~~(\text{from monotonicity of $\phi$})
\\
&=&
L\,\max_{k\in [K]}\,(\Delta_k)_+,
\end{eqnarray*}
as desired.
\end{proof}
% $
% Combining these observations, we have
% $
% \phi(\xi' + \1_k\Delta_k) \,-\, \phi(\xi') \,\leq\, L\,(\Delta_k)_+.
% $
% A repeated application of this inequality then gives us:
% \begin{eqnarray*}
% \phi(\xi + \Delta) \,-\, \phi(\xi)
% &=&
% \phi\left(\xi + \sum_{k=1}^K\1_k\Delta_k\right) \,-\, \phi(\xi)\\
% &=&
% \phi\left(\xi + \sum_{k=2}^{K}\1_k\Delta_k\right) \,-\, \phi(\xi) \,+\,
% \phi\left(\xi + \sum_{k=1}^{K}\1_k\Delta_k\right) \,-\, \phi\left(\xi + \sum_{k=2}^{K}\1_k\Delta_k\right)
% \\
% &\leq&
% \phi\left(\xi + \sum_{k=2}^{K}\1_k\Delta_k\right) \,-\, \phi(\xi)
% \,+\, L\,(\Delta_1)_+\\
% &\vdots&\\
% &\leq& 
% L\sum_{k=1}^K (\Delta_k)_+.
% \end{eqnarray*}
% The second part of the lemma directly follows from the above result:
% $$\phi(\bR(\bar{\mu}))
% \,=\, \phi(\bR(\bar{\mu}) \,-\, \bar{\xi} \,+\, \bar{\xi})
% \,\leq\, 
% \phi(\bar{\xi}) \,+\, L\,\max_{k\in [K]}(R_k(\bar{\mu}) \,-\, \bar{\xi}_k)_+.
% \vspace{-20pt}
% $$

\subsection{Proof of Lemma \ref{lem:xi-opt-closed-form-constrained}}
The proof follows directly from statement 3 of Lemma \ref{lem:app-Lagrange-opt-constrained}.

\subsection{General Convergence Result}
We present a  convergence result for a general no-regret strategy for the $\lambda$-player, and then apply it to the case where the player runs OGD with specific step-sizes. In this case,  the iterates  generated by Algorithm \ref{algo:lagrangian-ideal} for \eqref{eq:non-linear-opt} yield an approximate Nash equilibrium, i.e. the $\lambda$-player choosing the fixed strategy $\bar{\lambda} = \frac{1}{T}\sum_{t=1}^T \lambda^{t}$, 
the $\xi$-player choosing the fixed strategy $\bar{\xi} = \frac{1}{T}\sum_{t=1}^T \xi^t$, and the $\theta$-player choosing a uniform distribution $\bar{\mu}$ over $\theta^1, \ldots, \theta^T$, together form an approximate mixed-strategy Nash equilibrium of the zero-sum game in \eqref{eq:custom-opt-as-max-min-1}.
\begin{thm}
    Let $\theta^1, \ldots, \theta^T, \xi^1, \ldots, \xi^T, \lambda^1, \ldots, \lambda^T$ be the iterates generated by Algorithm \ref{algo:lagrangian-ideal} for \eqref{eq:non-linear-opt} when run with a $\rho$-approximate CSO oracle.  Suppose the $\lambda$ iterates  satisfy the following:
    \[
    \frac{1}{T}\sum_{t=1}^T\,\cL(\xi^t, \theta^t; \lambda^t) 
    \,\geq\, \max_{\lambda \in \Lambda}\,\frac{1}{T}\sum_{t=1}^T\,\cL(\xi^t, \theta^t; \lambda) \,-\, \epsilon_\lambda,
    \]
    for some $\epsilon_\lambda>0$.  
    Suppose each $\phi^j$ is strictly convex, monotonically non-decreasing in each argument and $L$-Lispchitz w.r.t.\ $\ell_\infty$ norm  in $[0,1]^K$. Suppose there exists a $\mu' \in \Delta_\Theta$ such that $\phi^j(\bR(\mu')) \leq -\gamma, \, \forall j \in [J]$, for some $\gamma > 0$. Let $\bar{\mu}$ be a stochastic model with a probability mass of $\frac{1}{T}$ on $\theta^t$.
 Let $\mu^* \in \Delta_\Theta$ be such that
 $\mu^*$ is feasible, i.e.\ $\phi^j({\bR}(\mu^*)) \leq 0,\,\forall j \in [J]$, and
 $\E_{\theta \sim \mu^*}\left[g(\theta)\right] \leq \E_{\theta \sim \mu}\left[g(\theta)\right]$ for every $\mu \in \Delta_{\Theta}$ that is feasible. Let $\lambda^* \in \amax{\lambda \in \Lambda}\min_{\xi, \mu}\,\cL(\xi, \mu; \lambda)$.
% Let $\lambda^*$ be a maximizer of $\min_{\xi, \mu}\cL(\xi, \mu; \lambda)$ over $\lambda \in \R_+^K$. %Suppose there exists a $\mu' \in \Delta_\Theta$ such that $\phi^j(\bR(\mu')) \leq -\gamma, \, \forall j \in [J]$, for some $\gamma > 0$. 
    Then setting $\kappa \geq 2\|\lambda^*\|_1$:
    % $$
    % \psi\big(\bR(\bar{\mu})\big)
    % \,\leq\, \min_{\mu \in \Delta_\Theta}\psi\big(\bR(\mu)\big) \,+\,
    % \left(2L/\kappa_1+1\right)(\epsilon_\lambda \,+\, \rho).
    % $$
    $$
    \E_{\theta \sim \bar{\mu}}\left[g\big(\theta)\right]
    \,\leq\, \E_{\theta \sim \mu^*}\left[g\big(\theta)\right] \,+\, \rho \,+\,\epsilon_\lambda
    $$ 
    and 
    $$
    \phi^j(\bR(\bar{\mu})) \,\leq\, 
    2\,(L+1)\,(\rho + \epsilon_\lambda)/\kappa,~~\forall j \in [J].
    $$
\label{thm:convergence-ideal-general-constrained}
\end{thm}
% We are now ready to prove Theorem \ref{thm:convergence-ideal-general}.
\begin{proof}%[Proof of Theorem \ref{thm:convergence-ideal-general-constrained}]
Let 
$\bar{\xi} = \frac{1}{T}\sum_{t=1}^T \xi^t$.
% and $\bar{\lambda} = \frac{1}{T}\sum_{t=1}^T \lambda^{t}$. 
Let $\lambda^*$  be as defined in the theorem statement. %$ \in \amax{\lambda \in \Lambda}\min_{\xi, \mu}\,\cL(\xi, \mu; \lambda)$.
% and $(\xi^*, \mu^*) \in \amin{\xi, \mu}\,\cL(\xi, \mu; \lambda^*).$
%The value that we set for the Lagrange multipliers $\kappa$ ensures that there exists $\lambda^* \in \Lambda$ that maximizes  $\min_{\xi, \mu}\cL(\xi, \mu; \lambda)$ over all $\lambda \in \R^K$. Given this, we can apply Lemma \ref{lem:app-Lagrange-opt-constrained}  to get  $\cL(\xi^*, \mu^*; \lambda^*) \,=\, \psi(\bR(\theta^*))$.

\textbf{Optimality.} The best-response strategies of the $\xi$-player and $\theta$-players give us: 
\begin{equation}
\cL_1(\xi^{t}; \lambda^t) \,=\, \min_{\xi \in \cR}\cL_1(\xi; \lambda^t).
\label{eq:app-xi-1}
\end{equation}
\begin{equation}
% \frac{1}{T}\sum_{t=1}^T\,\cL_2(\theta^t; \lambda^t) \,\leq\,
% \min_{\theta \in \Theta}\,\frac{1}{T}\sum_{t=1}^T\,\cL_2(\theta; \lambda^t)
% \,+\, \epsilon_\theta\\
\cL_2(\theta^{t}; \lambda^t) \,\leq\, \min_{\theta \in \Theta}\cL_2(\theta; \lambda^t) \,+\, \rho.
\label{eq:app-theta-1}
\end{equation}
From \eqref{eq:app-xi-1} and \eqref{eq:app-theta-1}, we further get:
\begin{eqnarray}
\frac{1}{T}\sum_{t=1}^{T}\,
\cL(\xi^{t}, \theta^{t}; \lambda^{t}) 
&=&
\frac{1}{T}\sum_{t=1}^{T}\,
\min_{\xi \in \cR}
\cL(\xi, \theta^{t}; \lambda^{t}) 
\nonumber
\\
&\leq&
\frac{1}{T}\sum_{t=1}^{T}\, 
\min_{\xi \in \cR, \theta \in \Theta}
\cL(\xi, \theta; \lambda^{t}) 
\,+\, \rho
\nonumber
\\
&=&
\frac{1}{T}\sum_{t=1}^{T}\, 
\min_{\xi \in \cR, \mu \in \Delta_\Theta}
\cL(\xi, \mu; \lambda^{t}) 
\,+\, \rho
~~~~~(\text{by linearity of $\cL$ in $\mu$})
\nonumber
\\
&\leq&
\min_{\xi \in \cR, \mu \in \Delta_\Theta}
\frac{1}{T}\sum_{t=1}^{T}\, 
\cL(\xi, \mu; \lambda^{t}) 
\,+\, \rho
\nonumber
\\
&=&
\min_{\xi \in \cR, \mu \in \Delta_\Theta}
\cL\left(\xi, \mu;\, \bar{\lambda}\right)
\,+\, \rho
~~~~~(\text{by linearity of $\cL$ in $\lambda^t$})
\nonumber\\
&\leq&
\max_{\lambda\in \R_+^K}
\min_{\xi \in \cR, \mu \in \Delta_\Theta}
\cL(\xi, \mu; \lambda) 
\,+\, \rho
\nonumber\\
&=&
\E_{\theta\sim \mu^*}\left[g(\theta)\right]
\,+\, \rho,
\label{eq:app-ideal-theta-constrained}
\end{eqnarray}
where the last step follows from statement 1 of Lemma \ref{lem:app-Lagrange-opt-constrained}.

Next, the OGD updates on $\lambda$ satisfy for any $\lambda' \in \Lambda$:
\begin{equation}
\frac{1}{T}\sum_{t=1}^T\,
\cL(\xi^{t}, \theta^{t}; \lambda^t) \,\geq\,
\frac{1}{T}\sum_{t=1}^T\,
\cL(\xi^{t}, \theta^{t}; \lambda')
\,-\, \epsilon_\lambda
\label{eq:app-ideal-lambda-constrained}
\end{equation}
Combining \eqref{eq:app-ideal-theta-constrained} and \eqref{eq:app-ideal-lambda-constrained}, %using the upper-bounding property of the surrogate,  i.e.\ using $\cL(\xi^t, \theta^t; \lambda^t) \leq \tL(\xi^t, \theta^t; \lambda^t), \forall t$,
we have for any $\lambda' \in \Lambda$:
\begin{equation}
\frac{1}{T}\sum_{t=1}^T\,
\cL(\xi^{t}, \theta^{t}; \lambda')
\,\leq\,
\E_{\theta\sim \mu^*}\left[g(\theta)\right]
\,+\, \rho \,+\, \epsilon_\lambda.
\label{eq:app-lambda-star-constrained}
\end{equation}
Setting $\lambda' = 0$ in \eqref{eq:app-lambda-star-constrained} gives us:
\[
\frac{1}{T}\sum_{t=1}^T g(\theta^{t}) \,\leq\, 
\E_{\theta\sim \mu^*}\left[g(\theta)\right]
\,+\, \rho \,+\, \epsilon_\lambda
\]
or
\begin{equation*}
\E_{\theta \sim \bar{\mu}}\left[g(\theta)\right] \,\leq\,
\E_{\theta\sim \mu^*}\left[g(\theta)\right]
\,+\, \rho \,+\, \epsilon_\lambda.
\label{eq:app-mu-bound-constrained}
\end{equation*}
This proves the optimality result.

\textbf{Feasibility.} 
Recall that there are two sets of constraints
$\phi^j(\mu) \leq 0, \forall j \in [J]$ and $R_k(\mu) \leq \xi_k, k \in [K]$. We first look at the first set of constraints. Let $j' \in \text{argmax}_{j\in [J]}\,\phi^j(\bar{\xi})$. If we set $\lambda'_{j'} = \lambda_{j'}^* + \kappa/2$ 
and $\lambda'_{j} = \lambda^*, \,\forall j \ne j', j \in [J+K]$ in \eqref{eq:app-lambda-star-constrained} %, and otherwise set $\lambda' = \lambda^*$ in \eqref{eq:app-lambda-star-constrained} 
(note that $\lambda' \in \Lambda$). This
gives us:
\[
\frac{1}{T}\sum_{t=1}^T\,
\cL(\xi^{t}, \theta^{t}; \lambda^*)
\,+\, \frac{\kappa}{2T}\sum_{t=1}^T\,
\phi^{j'}(\xi^t)
\,\leq\, 
\E_{\theta\sim \mu^*}\left[g(\theta)\right]
\,+\, \rho \,+\, \epsilon_\lambda.
\]
Since $\cL$ is linear in $\bR(\theta)$ and convex in $\xi$, using Jensen's inequality, we have $\frac{1}{T}\sum_{t=1}^T\,
\cL(\xi^{t}, \theta^{t}; \lambda^*) \geq\,
\cL_1(\bar{\xi}; \lambda^{*})
\,+\,
\cL_2(\bar{\mu};
\lambda^*)\,\geq\,
\min_{\xi, \mu} \cL(\xi, \mu; \lambda^{*})
\,=\, 
\E_{\theta\sim \mu^*}\left[g(\theta)\right]$ (by statement 1 of Lemma \ref{lem:app-Lagrange-opt-constrained}). Further by convexity of $\phi^j$, we have:
\[
\E_{\theta\sim \mu^*}\left[g(\theta)\right]
\,+\,
\frac{\kappa}{2}\,\phi^{j'}(\bar{\xi})
\,\leq\, 
\E_{\theta\sim \mu^*}\left[g(\theta)\right]
\,+\, \rho \,+\, \epsilon_\lambda,
\]
which implies:
\begin{equation*}
    \max_{j\in [J]}\,\phi^j(\bar{\xi}) \,\leq\, 2(\rho \,+\, \epsilon_\lambda)/\kappa.
\end{equation*}
Applying Lemma \ref{lem:helper-Lipchitz} to each $\phi^j$, we further get
\begin{equation}
    \max_{j\in [J]}\,\phi^j(\bR(\bar{\mu})) \,\leq\, 
    L\,\max_{k\in [K]}\,(R_{k}(\bar{\mu}) \,-\,
    \bar{\xi}_{k})_+
    \,+\,
    2(\rho \,+\,  \epsilon_\lambda)/\kappa.
    \label{eq:app-phi-bound-constrained}
\end{equation}

For the second set of constraints, let $k' \in \text{argmax}_{k\in [K]}(R_k(\bar{\mu}) \,-\, \bar{\xi}_k)$.
If 
$R_{k'}(\bar{\mu}) \,-\, \bar{\xi}_{k'} \,\leq\, 0$, then 
the feasibility condition is already satisfied. Otherwise,
set $\lambda'_{J+k'} = \lambda_{k'}^* + \kappa/2$ 
and $\lambda'_{j} = \lambda^*, \,\forall j \ne J + k'$
in \eqref{eq:app-lambda-star-constrained} in \eqref{eq:app-lambda-star-constrained}, giving us:
\[
\frac{1}{T}\sum_{t=1}^T\,
\cL(\xi^{t}, \theta^{t}; \lambda^*)
\,+\,
\frac{\kappa}{2T} \sum_{t=1}^T(R_{k'}(\theta^t)
\,-\,
\xi^t_{k'}
)
\,\leq\, 
\E_{\theta\sim \mu^*}\left[g(\theta)\right]
\,+\, \rho \,+\, \epsilon_\lambda.
\]
Following the same steps as above:
\begin{equation}
\max_{k\in [K]}\,(R_{k}(\bar{\mu}) \,-\,
    \bar{\xi}_{k}) \,\leq\, 2(\rho \,+\, \epsilon_\lambda)/\kappa.
    \label{eq:app-xi-bound-constrained}
\end{equation}
Substituting 
\eqref{eq:app-xi-bound-constrained}
back in \eqref{eq:app-phi-bound-constrained}, we have:
\begin{equation*}
\max_{j\in [J]}\,\phi^j(\bR(\bar{\mu}))
\,\leq\,
2\,(L+1)\,(\rho + \epsilon_\lambda)/\kappa,
\end{equation*}
as desired.
\end{proof}

\subsection{Corollary for OGD on $\lambda$}
\begin{lem}
Suppose each $\phi^j$ is monotonically non-decreasing in each argument and $L$-Lipschitz w.r.t.\ the $\ell_\infty$ norm  in $[0,1]^K$. 
Suppose there exists a $\mu' \in \Delta_\Theta$ such that $\phi^j(\bR(\mu')) \leq -\gamma, \, \forall j \in [J]$, for some $\gamma > 0$. Let $B_g \,=\, \max_{\theta \in \Theta}\,g(\theta)$. 
Fix $\mu \in \Delta_\Theta$. Then for any
$$\lambda^* \,\in\, \argmax{\lambda\in \R_+^{J+K}}
\left\{
\min_{\xi \in  \cR, \mu \in \Delta_\Theta}\,
\cL(\xi, \mu;\,\lambda)
\right\},
$$
the following holds:
$\|\lambda^*\|_1 \,\leq\, (L+1)B_g/\gamma$.
\label{lem:app-Lagrange-radius-constrained}
\end{lem}
\begin{proof}
We separate $\lambda^*$ into $a^* = [\lambda^*_1, \ldots, \lambda^*_J]^\top$ and $b^* = [\lambda^*_{J+1}, \ldots, \lambda^*_{J+K}]^\top$. We first bound the norm of $a^*$. Let $\mu^*$ be as defined in Lemma \ref{lem:app-Lagrange-opt-constrained}. We then have:
% Using the fact that $\cL(\xi, \mu; \lambda) = \cL_1(\xi; \lambda) \,+\, \cL_2(\mu; \lambda)$, and that $\cL_2$ is independent of $\xi$, we have
\begin{eqnarray*}
\E_{\theta \sim \mu^*}\left[g(\theta)\right] 
&=& \min_{\xi \in \R_+^K,\, \mu\in \Delta_\Theta}\,
\cL(\xi, \mu; \lambda^*)\\
&=&
\min_{\xi \in \R_+^K,\, \mu\in \Delta_\Theta}\,
\E_{\theta \sim \mu}\left[g(\theta)\right] 
\,+\,
\sum_{j=1}^J
a^*_{j}\,
\phi^j(\xi) 
\,+\,
\sum_{k=1}^K b^*_k\, (R_k(\mu) - \xi_k)
\\
&\leq&
\E_{\theta \sim \mu'}\left[g(\mu)\right]\,+\,
\sum_{j=1}^J a^*_j\, \phi^j(\bR(\mu'))
~~~~(\text{setting $\mu = \mu'$ and $\xi_k = R_k(\mu')$})
\\
&\leq&
B_g \,-\, \gamma\sum_{j=1}^J a^*_j
\,=\, B_g \,-\, \gamma\|a^*\|_1,
\end{eqnarray*}
which gives us:
\[
\|a^*\|_1 \,\leq\, 
(B_g - \E_{\theta \sim \mu^*}\left[g(\theta)\right] ) / \gamma \,\leq\, B_g/\gamma.
\]

%\todohari{While applying Lemma \ref{lem:app-Lagrange-opt-constrained}, you are implicitly exchanging min and max -- make this clear.}
We next bound the norm of $b^*$. Let $\Phi_{a^*}$ be as defined in Lemma \ref{lem:app-Lagrange-opt-constrained}. We then have from statement 2 of Lemma \ref{lem:app-Lagrange-opt-constrained} that $b^* = \nabla\Phi_{a^*}(\bR(\mu^*))$. We further have:
\begin{eqnarray*}
\|b^*\|_1  &\leq&
\max_{\mu\in \Delta_\Theta}\,\Big\|\nabla\Phi_{a^*}(\bR(\mu))\Big\|_1\\
\\
&\leq&
\max_{\xi\in \R_+^K}\,\big\|\nabla\Phi_{a^*}(\xi)\big\|_1\\
\\
&=&
\max_{\xi\in \R_+^K}\,\Big\|\sum_{j=1}^J a^*_j\nabla \phi^j(\xi)\Big\|_1\\
&\leq&
\max_{\xi\in \R_+^K}\,\sum_{j=1}^J |a^*_j|\big\|\nabla \phi^j(\xi)\big\|_1
\\
&\leq&
\sum_{j=1}^J|a^*_j|\max_{\xi\in \R_+^K}\big\|\nabla \phi^j(\xi)\big\|_1\\
&\leq& L\|a^*\|_1 ~\leq~ LB_g/\gamma,
\end{eqnarray*}
where in the last step, we use the Lipschitz property of $\phi^j$.

Thus $\|\lambda^*\|_1 \,=\, \|a^*\|_1 + \|b^*\|_1 \,\leq\, (L+1)B_g/\gamma$, as desired.
\end{proof}

\begin{proof}[Proof of Theorem \ref{thm:convergence-ideal-constrained}]
We apply standard OGD convergence analysis \cite{Zinkevich03} and online-to-batch arguments \cite{Cesa+14} to the $\lambda$-player's gradient updates on $\lambda$. For the sequence $-\cL(\xi^1, \theta^1;\, \cdot), \ldots, -\cL(\xi^T, \theta^T;\, \cdot)$, with $\eta = \frac{\kappa}{B_{\blambda}\sqrt{2T}}$ in Algorithm \ref{algo:lagrangian-ideal}, we get the following regret bound. With probability at least $1-\delta$ over draws of stochastic gradients of $\cL$:
\[
    \frac{1}{T}\sum_{t=1}^T\,\cL(\xi^t, \theta^t; \lambda^t) 
    \,\geq\, \max_{\lambda \in \Lambda}\,\frac{1}{T}\sum_{t=1}^T\,\cL(\xi^t, \theta^t; \lambda) \,-\, 2\kappa\,B_\lambda\sqrt{\frac{1+16\log(1/\delta)}{T}}.
\]
Following Lemma \ref{lem:app-Lagrange-radius-constrained}, we set the radius of the space of Lagrange multipliers $\Lambda$ to $\kappa = 2(L+1)B_g/\gamma$ and apply Theorem \ref{thm:convergence-ideal-general} to complete the proof:
\begin{equation*}
\E_{\theta \sim \bar{\mu}}\left[g(\theta)\right] \,\leq\,
g(\theta^*)
\,+\, \rho \,+\, \frac{4(L+1)B_gB_\lambda}{\gamma}\sqrt{\frac{1+16\log(1/\delta)}{T}}.
\end{equation*}
and
\begin{equation*}
\max_{j\in [J]}\,\phi^j(\bR(\bar{\mu}))
\,\leq\,
\frac{\gamma\rho}{B_g}
\,+\,
4(L+1)B_\lambda\sqrt{\frac{1+16\log(1/\delta)}{T}}.
\end{equation*}
\end{proof}

\section{Proof of Theorem \ref{thm:convergence-surrogate-constrained}}
\label{app:thm-convergence-surrogate-constrained}
%We provide a proof assuming access to exact gradients $\nabla_\theta\tR_k(\theta)$. The proof extends in a straight-forward manner to stochastic gradients using standard online-to-batch conversion arguments \cite{Cesa+14}. Our proof applies to  $\mathcal{C} = \R_+^K$. 
The proof adapts ideas from previous results on constrained optimization and game equilibrium \cite{Cotter+19b, Agarwal+18}. We will find the following lemma useful.
\begin{lem}
\label{lem:app-lagrangian-opt-proxy-constrained}
Let $\tL(\xi, \mu; \lambda) \,=\, \cL_1(\xi; \lambda) + \tL_2(\mu; \lambda)$. Suppose each $\phi^j$ is convex and monotonically non-decreasing in its arguments and $g$ is convex. 
Let $\tilde{\Theta} \,=\, \big\{\theta \in \Theta \,|\, 
   \tilde{\bR}(\theta) \,\in\, \dom\, \phi^j,~\forall j
    \big\}$.
%let $\tilde{\Delta}_\Theta \,=\, \big\{\mu \in \Delta_\Theta \,|\,     \text{each}~\phi^j(\tilde{\bR}(\mu))~\text{is well-defined} \big\}$.  
Let $\tilde{\theta}^* \in \tilde{\Theta}$ be such that $\phi^j(\tilde{\bR}(\tilde{\theta}^*)) \leq 0,\,\forall j \in [J]$ and
$g(\tilde{\theta}^*) \leq g({\theta})$ for all $\theta \in \tilde{\Theta}$ that satisfies the same constraints. %satisfy $\phi^j(\tilde{\bR}(\theta)) \leq 0,\,\forall j \in [J]$. 
Then
\[
\max_{\lambda \in \R_+^K}\min_{\xi \in  \cR, \theta \in {\Theta}}\,
\tL(\xi, \theta;\,\lambda) \,\leq\, g(\tilde{\theta}^*)
\]
\end{lem}
\begin{proof}
%  Note that since $\dom\, \phi^j$ is the non-negative orthant, $\tilde{\Theta}$ contains the set of $\theta \in \Theta$ for which $\tbR(\theta) \geq \0$.  
Since $\tL$ is linear in $\mu(\theta)$, convex in $\xi$ and linear in $\lambda$, strong duality holds and we have:
\begin{eqnarray*}
\max_{\lambda \in \R_+^K}\min_{\xi \in  \cR, \theta \in {\Theta}}\,
\tL(\xi, \mu;\,\lambda) &=& \min_{\theta \in {\Theta},\, \xi~|~ \tbR(\theta) \leq \xi,~\phi^j(\xi) \leq 0,\, \forall j}\, g(\theta)\\
&\leq& \min_{\theta \in \tilde{\Theta},\, \xi~|~ \tbR(\theta) \leq \xi,~\phi^j(\xi) \leq 0,\, \forall j}\, g(\theta)
~~~~~(\text{from $\tilde{\Theta} \subseteq \Theta$})\\
&=& \min_{\theta \in \tilde{\Theta}~|~ \phi^j(\tbR(\theta)) \leq 0,\, \forall j}\, g(\theta),
\end{eqnarray*}
where the last step follows from monotonicity of each $\phi^j$.
% , and in particular from $$\exists \theta \in \Theta, \xi \in \R_+^K ~~\text{s.t.}~~ \tbR(\theta) \leq \xi, ~~\phi^j(\xi) \leq 0, \forall j ~~\iff~~ \exists \theta \in \tilde{\Theta} ~~\text{s.t.}~~ \phi^j(\tbR(\theta)) \leq 0,\, \forall j.$$ 
% Since $g$ is convex in $\theta$, we know that for any 
% $\mu$, $\E_{\theta \sim \mu}\left[g(\theta)\right] \,\geq\,
% g(\E_{\theta \sim \mu}[\theta])$. Hence the minimization on the right-hand side can be equivalently performed over deterministic models in $\tilde{\Theta}$.
\end{proof}

\subsection{General Convergence Result}
We present a convergence result for general no-regret strategies for the $\theta$- and $\lambda$-player that find an approximate coarse-correlated equilibrium, and then specialize it to the case where the players run OGD with specific step-sizes.

\begin{thm}
    Let $\theta^1, \ldots, \theta^T, \xi^1, \ldots, \xi^T, \lambda^1, \ldots, \lambda^T$ be the iterates generated by Algorithm \ref{algo:lagrangian-surrogate} for \eqref{eq:non-linear-opt}. Let $\tilde{\Theta} \,=\, \big\{\theta \in \Theta \,|\, 
   \tilde{\bR}(\theta) \,\in\, \dom\, \phi^j,~\forall j
    \big\}$. 
    Suppose the iterates comprise an approximate coarse-correlated equilibrium, i.e.\ satisfy:
    \begin{equation}
    \frac{1}{T}\sum_{t=1}^T\,\cL_1(\xi^{t}; \lambda^t) \,\leq\, \min_{\xi \in \cR}\frac{1}{T}\sum_{t=1}^T\,\cL_1(\xi; \lambda^t);
    \label{eq:xi-best-response-1}
    \end{equation}
    \begin{equation}
    \frac{1}{T}\sum_{t=1}^T\,\tL_2( \theta^t; \lambda^t) 
    \,\leq\, \min_{\theta \in {\Theta}}\,\frac{1}{T}\sum_{t=1}^T\,\tL_2(\theta; \lambda^t) \,+\, \epsilon_\theta;
    \label{eq:theta-ogd-1}
    \end{equation}
    \begin{equation}
    \frac{1}{T}\sum_{t=1}^T\,\cL(\xi^t, \theta^t; \lambda^t) 
    \,\geq\, \max_{\lambda \in \Lambda}\,\frac{1}{T}\sum_{t=1}^T\,\cL(\xi^t, \theta^t; \lambda) \,-\, \epsilon_\lambda,
    \label{eq:lambda-ogd-1}
    \end{equation}
    for some $\epsilon_\theta>0$ and $\epsilon_\lambda>0$. 
     Suppose each $\phi^j$ is strictly convex, monotonically non-decreasing in each argument and $L$-Lispchitz w.r.t.\ $\ell_\infty$ norm  in $[0,1]^K$. 
    % Let $\lambda^*$ be a maximizer of $\min_{\xi, \mu}\cL(\xi, \mu; \lambda)$ over $\lambda \in \R^K$ and set $\kappa \geq 2\|\lambda^*\|_1$. 
    Let $B_g \,=\, \max_{\theta \in \Theta}\,g(\theta)$. 
    Let $\tilde{\theta}^* \in \tilde{\Theta}$ be such that
    $\phi^j(\tilde{\bR}(\tilde{\theta}^*)) \leq 0,\,\forall j \in [J]$
    and $g(\tilde{\theta}^*) \leq g({\theta})$ for all $\theta \in \tilde{\Theta}$ that satisfies the same constraints. Let $\bar{\mu}$ be a stochastic model with a probability mass of $\frac{1}{T}$ on $\theta^t$. Then:
    %Then setting $\kappa = \frac{L}{\sqrt{\epsilon_\lambda}}$:
    $$
    \E_{\theta \sim \bar{\mu}}\left[g\big(\theta)\right]
    \,\leq\, g(\tilde{\theta}^*) \,+\, \epsilon_\theta + \epsilon_\lambda
    $$ 
    and 
    $$
    \phi^j(\bR(\bar{\mu})) \,\leq\, 
    (L+1)\,(B_g \,+\,
    \epsilon_\theta \,+\, \epsilon_\lambda)/\kappa,~~\forall j \in [J].
    $$
\label{thm:convergence-surrogate-general-constrained}
\end{thm}

\begin{proof}[Proof of Theorem \ref{thm:convergence-surrogate-general-constrained}]
Let $\bar{\xi} = \frac{1}{T}\sum_{t=1}^T \xi^t$.
% and $\bar{\lambda} = \frac{1}{T}\sum_{t=1}^T \lambda^{t}$. 
% Let $(\xi^*, \mu^*, \lambda^*)$ be such that $\lambda^* \in \amax{\lambda \in \Lambda}\,\cL(\xi^*, \mu^*; \lambda)$
% and $(\xi^*, \mu^*)
% \in \amin{\xi, \mu}\,\cL(\xi, \mu; \lambda^*).
% $
%The value that we set for the Lagrange multipliers $\kappa$ ensures that there exists $\lambda^* \in \Lambda$ that maximizes  $\min_{\xi, \mu}\cL(\xi, \mu; \lambda)$ over all $\lambda \in \R^K$. Given this, we can apply Lemma \ref{lem:app-fenchel-conjugate} with $\mu = \mu^*$ and get  $\cL(\xi^*, \mu^*; \lambda^*) \,=\, \psi(\bR(\mu^*))$.

\textbf{Optimality.}
% The best-response strategy of the $\xi$-player gives us: 
% \begin{equation*}
% \cL_1(\xi^{t}; \lambda^t) \,=\, \min_{\xi \in \cR}\cL_1(\xi; \lambda^t)
% \label{eq:app-xi-surrogate}
% \end{equation*}
% \begin{equation}
% % \frac{1}{T}\sum_{t=1}^T\,\cL_2(\theta^t; \lambda^t) \,\leq\,
% % \min_{\theta \in \Theta}\,\frac{1}{T}\sum_{t=1}^T\,\cL_2(\theta; \lambda^t)
% % \,+\, \epsilon_\theta\\
% \cL_2(\theta^{t+1}; \lambda^t) \,\leq\, \min_{\theta \in \Theta}\cL_2(\theta; \lambda^t) \,+\, \rho
% \label{eq:app-theta}
% \end{equation}
% From \eqref{eq:app-xi} and \eqref{eq:app-theta}, we further get:
% and further get:
We have:
\begin{eqnarray}
\frac{1}{T}\sum_{t=1}^{T}\,
\tL(\xi^{t}, \theta^{t}; \lambda^{t})
% &=&
% \frac{1}{T}\sum_{t=1}^{T}\,
% \min_{\xi \in \cR}
% \cL_1(\xi; \lambda^{t}) 
% \,+\,
% \tL_2(\theta^{t}; \lambda^{t}) 
% \nonumber
% \\
&\leq&
\min_{\xi \in \cR}
\frac{1}{T}\sum_{t=1}^{T}\, 
\cL_1(\xi; \lambda^{t})
\,+\,
\frac{1}{T}\sum_{t=1}^{T}\, 
\tL_2(\theta^{t}; \lambda^{t}) 
~~~~~(\text{from \eqref{eq:xi-best-response-1}})
\nonumber
\\
&\leq&
\min_{\xi \in \cR}
\frac{1}{T}\sum_{t=1}^{T}\, 
\cL_1(\xi; \lambda^{t})
\,+\,
\min_{\theta \in \Theta}\,
\frac{1}{T}\sum_{t=1}^{T}\, 
\tL_2(\theta; \lambda^{t}) 
\,+\, \epsilon_\theta
~~~~~(\text{from \eqref{eq:theta-ogd-1}})
\nonumber
\\
&\leq&
\min_{\xi \in \cR, \theta \in \Theta}
\frac{1}{T}\sum_{t=1}^{T}\, 
\tL(\xi, \theta; \lambda^{t})
\,+\, \epsilon_\theta
\nonumber
\\
&=&
\min_{\xi \in \cR,\, \theta \in \Theta}\, 
\tL(\xi, \theta; \bar{\lambda})
\,+\, \epsilon_\theta
\nonumber
\\
&\leq&
\max_{\lambda \in \R_+^K}
\min_{\xi \in \cR,\, \theta \in \Theta}\, 
\tL(\xi, \theta; {\lambda})
\,+\, \epsilon_\theta
\nonumber
\\
&\leq&
% \max_{\lambda \in \R_+^K}
% \min_{\xi \in \cR,\, \theta \in \tilde{\Theta}}\, 
% \tL(\xi, \theta; {\lambda})
% \,+\, \epsilon_\theta
% \nonumber
% \\
% &=&
g(\tilde{\theta}^*) \,+\, \epsilon_\theta.
\label{eq:app-surrogate-theta-constrained}
\end{eqnarray}
where the last step follows from Lemma \ref{lem:app-lagrangian-opt-proxy-constrained}.

We also have from \eqref{eq:lambda-ogd-1}:
\begin{equation}
\frac{1}{T}\sum_{t=1}^T\,
\cL(\xi^{t}, \theta^{t}; \lambda^t) \,\geq\,
\frac{1}{T}\sum_{t=1}^T\,
\cL(\xi^{t}, \theta^{t}; \lambda')
\,-\, \epsilon_\lambda.
\label{eq:app-surrogate-lambda-constrained}
\end{equation}
Combining \eqref{eq:app-surrogate-theta-constrained} and \eqref{eq:app-surrogate-lambda-constrained}, we have for any $\lambda' \in \Lambda$:
\begin{equation}
\frac{1}{T}\sum_{t=1}^T\,
\cL(\xi^{t}, \theta^{t}; \lambda')
\,\leq\,
%\min_{\theta \in {\Theta}}\,
g(\tilde{\theta}^*)
\,+\, \epsilon_\theta \,+\, \epsilon_\lambda.
\label{eq:app-lambda-star-surrogate-constrained}
\end{equation}
Setting $\lambda' = 0$ in \eqref{eq:app-lambda-star-surrogate-constrained} gives us:
\[
\E_{\theta \sim \bar{\mu}}[g(\theta)] \,\leq\, 
g(\tilde{\theta}^*)
\,+\, \epsilon_\theta \,+\, \epsilon_\lambda.
\]
This completes the proof of optimality.

\textbf{Feasibility.}
Recall that there are two sets of constraints
$\phi^j(R_k(\mu)) \leq 0, \forall j \in [J]$ and $R_k(\mu) \leq \xi_k, k \in [K]$ separately. We first look at the first set of constraints. Let $j' \in \text{argmax}_{j\in [J]}\,\phi^j(\bar{\xi})$ and set $\lambda'_{j'} = \kappa$ 
and $\lambda'_{j} = 0, \,\forall j \ne j', j \in [J+K]$ in \eqref{eq:app-lambda-star-surrogate-constrained} %, and otherwise set  $\lambda' = 0$ in \eqref{eq:app-lambda-star-surrogate-constrained} 
(note that $\lambda' \in \Lambda$). This
gives us:
\[
\E_{\theta \sim \bar{\mu}}[g(\theta)]
\,+\, \frac{\kappa}{T}\sum_{t=1}^T\,
\phi^{j'}(\xi^t)
\,\leq\, 
g(\theta^*)
\,+\, \epsilon_\theta \,+\, \epsilon_\lambda.
\]
% Since $\cL$ is convex in $\bR(\theta)$ and $\xi$, using Jensen's inequality, we have $\frac{1}{T}\sum_{t=1}^T\,
% \cL(\xi^{t}, \theta^{t}; \lambda^*) \geq\,
% \cL_1(\bar{\xi}; \lambda^{*})
% \,+\,
% \cL_2(\bar{\mu};
% \lambda^*)\,\geq\,
% \min_{\xi, \mu} \cL(\xi, \mu; \lambda^{*})
% \,=\, 
% g(\theta^*)$. Further 
By convexity of $\phi^j$ and using $g(\tilde{\theta}^*) \leq B_g$, we have:
\[
\phi^{j'}(\bar{\xi})
~\leq~
(B_g 
\,-\,
\E_{\theta \sim \bar{\mu}}[g(\theta)]
\,+\,
    \epsilon_\theta \,+\, \epsilon_\lambda)/\kappa
~\leq~
(B_g
\,+\,
    \epsilon_\theta \,+\, \epsilon_\lambda)/\kappa
\]
which implies:
\begin{equation*}
    \max_{j\in [J]}\,\phi^j(\bar{\xi}) \,\leq\, (B_g \,+\,
    \epsilon_\theta \,+\, \epsilon_\lambda)/\kappa.
\end{equation*}
Applying Lemma \ref{lem:helper-Lipchitz} to each $\phi^j$, we further get
\begin{equation}
    \max_{j\in [J]}\,\phi^j(\bR(\bar{\mu})) \,\leq\, 
    L\,\max_{k\in [K]}\,(R_{k}(\bar{\mu}) \,-\,
    \bar{\xi}_{k})_+
    \,+\,
    (B_g \,+\,
    \epsilon_\theta \,+\, \epsilon_\lambda)/\kappa.
    \label{eq:app-phi-bound-surrogate-constrained}
\end{equation}

For the second set of constraints, let $k' \in \text{argmax}_{k\in [K]}(R_k(\bar{\mu}) \,-\, \bar{\xi}_k)$.
If 
$R_k(\bar{\mu}) \,-\, \bar{\xi}_k \,\leq\, 0$,
then the feasibility result already holds. Otherwise, 
set $\lambda'_{J+k'} = \kappa$ 
and $\lambda'_{j} = 0, \,\forall j \ne J + k'$
in \eqref{eq:app-lambda-star-surrogate-constrained}, and following the same steps as above:
\begin{equation}
\max_{k\in [K]}\,(R_{k}(\bar{\mu}) \,-\,
    \bar{\xi}_{k})_+ \,\leq\, (B_g \,+\,
    \epsilon_\theta \,+\, \epsilon_\lambda)/\kappa.
    \label{eq:app-xi-bound-surrogate-constrained}
\end{equation}
Substituting 
\eqref{eq:app-xi-bound-surrogate-constrained}
back in \eqref{eq:app-phi-bound-surrogate-constrained}, we have:
\begin{equation*}
\max_{j\in [J]}\,\phi^j(\bR(\bar{\mu}))
\,\leq\,
(L+1)\,(B_g \,+\,
    \epsilon_\theta \,+\, \epsilon_\lambda)/\kappa,
\end{equation*}
which completes the feasibility proof.
%Setting $\displaystyle \kappa = \frac{L+1}{\sqrt{\epsilon_\lambda}}$ completes the feasibility proof.
\end{proof}

\subsection{Corollary for OGD on $\theta$ and $\lambda$}
\begin{proof}[Proof of Theorem \ref{thm:convergence-surrogate-constrained}]
% We use standard OGD convergence analysis \cite{Zinkevich03} to the updates of the $\theta$- and $\lambda$-player.
% For the sequence of losses $-\cL(\xi^1, \cdot;\, \lambda^1), \ldots, -\cL(\xi^T, \cdot;\, \lambda^T)$ optimized by the $\theta$-player, setting $\eta = \frac{B_\Theta}{B_{\theta}\sqrt{2T}}$ in Algorithm \ref{algo:lagrangian-surrogate}, we get the following regret bound:
% \[
%     \frac{1}{T}\sum_{t=1}^T\,\cL(\xi^t, \theta^t; \lambda^t) 
%     \,\leq\, \min_{\theta \in \tilde{\Theta}}\,\frac{1}{T}\sum_{t=1}^T\,\cL(\xi^t, \theta; \lambda^t) \,+\, B_\Theta\,B_\theta\sqrt{\frac{2}{T}}.
% \]
We show that  the iterates of Algorithm \ref{algo:lagrangian-surrogate} form an approximate coarse-correlated equilibrium and apply Theorem \ref{thm:convergence-surrogate-general-constrained}. 

The best-response strategy of the $\xi$-player gives us:
\[
\frac{1}{T}\sum_{t=1}^{T}\,
\cL_1(\xi^{t}; \lambda^{t})
~=~
\frac{1}{T}\sum_{t=1}^{T}
\min_{\xi \in \cR}
\, 
\cL_1(\xi; \lambda^{t})
~\leq~
\min_{\xi \in \cR}
\frac{1}{T}\sum_{t=1}^{T}\, 
\cL_1(\xi; \lambda^{t}).
\]

We then derive no-regret guarantees for the updates of the $\theta$- and $\lambda$-player using standard convergence analysis for OGD \cite{Zinkevich03} and standard online-to-batch conversion arguments \cite{Cesa+14}. 
For the sequence of losses $-\cL(\xi^1, \cdot;\, \lambda^1), \ldots, -\cL(\xi^T, \cdot;\, \lambda^T)$ optimized by the $\theta$-player, setting $\eta = \frac{B_\Theta}{B_{\theta}\sqrt{2T}}$ in Algorithm \ref{algo:lagrangian-surrogate}, we get the following regret bound for the $\theta$-player (see Corollary 3 in \cite{Cotter+19} for the complete derivation). With probability $\geq 1 - \delta/2$ over draws of  stochastic gradients of $\tL_1$:
\begin{eqnarray*}
    \frac{1}{T}\sum_{t=1}^T\,\tL(\xi^t, \theta^t; \lambda^t) 
    &\leq& \min_{\theta \in \Theta}\,\frac{1}{T}\sum_{t=1}^T\,\tL(\xi^t, \theta; \lambda^t) \,+\, 2B_\Theta\,B_\theta\sqrt{\frac{1 + 16\log(2/\delta)}{T}}, %\\
    % &\leq& \min_{\theta \in \tilde{\Theta}}\,\frac{1}{T}\sum_{t=1}^T\,\tL(\xi^t, \theta; \lambda^t) \,+\, 2B_\Theta\,B_\theta\sqrt{\frac{1 + 16\log(2/\delta)}{T}},
\end{eqnarray*}
%where we use $\tilde{\Theta} \subseteq \Theta$. 
Similarly, for the sequence of losses $-\cL(\xi^1, \theta^1;\, \cdot), \ldots, -\cL(\xi^T, \theta^T;\, \cdot)$ optimized by the $\lambda$-player, setting $\eta = \frac{\kappa}{B_{\blambda}\sqrt{2T}}$ in Algorithm \ref{algo:lagrangian-surrogate}, we get the following regret bound. With probability $\geq 1 - \delta/2$ over draws of  stochastic gradients of $\cL$:
\[
    \frac{1}{T}\sum_{t=1}^T\,\cL(\xi^t, \theta^t; \lambda^t) 
    \,\geq\, \max_{\lambda \in \Lambda}\,\frac{1}{T}\sum_{t=1}^T\,\cL(\xi^t, \theta^t; \lambda) \,-\, 2\kappa\,B_\lambda\sqrt{\frac{1 + 16\log(2/\delta)}{T}}.
\]

An application of Theorem \ref{thm:convergence-surrogate-general-constrained} with $\kappa  =  (L+1)T^{\omega}$, for $\omega \in (0, 0.5)$ then gives us:
\begin{equation*}
\E_{\theta \sim \bar{\mu}}\left[g(\theta)\right] \,\leq\,
g(\tilde{\theta}^*)
\,+\, 2B_\Theta\,B_\theta\sqrt{\frac{1 + 16\log(2/\delta)}{T}}
\,+\, 2(L+1)B_\lambda\frac{\sqrt{1 + 16\log(2/\delta)}}{T^{1/2-\omega}}
\end{equation*}
and
\begin{equation*}
\max_{j\in [J]}\,\phi^j(\bR(\bar{\mu}))
\,\leq\,
\frac{B_g}{T^{\omega}}
\,+\,
 2B_\Theta\,B_\theta\frac{\sqrt{1 + 16\log(2/\delta)}}{T^{1/2 + \omega}}
 \,+\,
 2(L+1)B_\lambda\sqrt{\frac{1 + 16\log(2/\delta)}{T}},
\end{equation*}
which completes the proof.
\end{proof}
\end{document}